\documentclass{article}

\PassOptionsToPackage{numbers, sort&compress}{natbib}
\usepackage[main, final]{neurips_2026}

\usepackage[utf8]{inputenc}
\usepackage[T1]{fontenc}
\usepackage{xcolor}
\usepackage{hyperref}
\hypersetup{
hypertexnames=false,
  colorlinks=true,
  linkcolor=blue,
  citecolor=teal,
  urlcolor=blue
}
\usepackage{url}
\usepackage{booktabs}
\usepackage{microtype}
\usepackage{graphicx}

\usepackage{amsmath,amssymb,amsfonts,mathtools}
\usepackage{amsthm}

\theoremstyle{plain}
\newtheorem{proposition}{Proposition}

\newtheorem{Lemma}{Lemma}

\theoremstyle{definition}
\newtheorem{definition}{Definition}

\theoremstyle{remark}

\usepackage{bm}
\usepackage{multirow}
\usepackage{array}
\usepackage{siunitx}
\usepackage{wrapfig}
\usepackage{algorithm}
\usepackage{algorithmic}

\newcommand{\R}{\mathbb{R}}
\newcommand{\E}{\mathbb{E}}
\newcommand{\KL}{\mathrm{KL}}

\newcommand{\tr}{\mathrm{tr}}
\newcommand{\diag}{\mathrm{diag}}

\title{The Dynamic-Probabilistic Consistency Gap in Chaotic Surrogate Modeling}

\author{%
  \textbf{Andre Herz}$^{1,2}$ \quad 
  \textbf{Matthijs Pals}$^{3}$ \quad 
  \textbf{Daniel Durstewitz}$^{1,3,4}$ \quad 
  \textbf{Georgia Koppe}$^{1,2,5,6}$ \\
  \vspace{0.4em} \\
  \small $^{1}$Interdisciplinary Center for Scientific Computing, Heidelberg University, Germany\\
  \small $^{2}$Faculty of Mathematics and Computer Science, Heidelberg University, Germany\\
  \small $^{3}$Dept. of Theoretical Neuroscience, Central Institute of Mental Health (CIMH), Mannheim, Germany\\
  \small $^{4}$Faculty of Physics and Astronomy, Heidelberg University, Germany\\
  \small $^{5}$Hector Institute for AI in Psychiatry and Dept. of Psychiatry and Psychotherapy, CIMH, Mannheim, Germany\\
  \small $^{6}$Hertie Institute for AI in Brain Health, University of Tübingen, Germany\\
  \small Correspondence to: \texttt{andre.herz\,@\,iwr.uni-heidelberg.de}
}

\hypersetup{
    pdftitle={The Dynamic-Probabilistic Consistency Gap in Chaotic Surrogate Modeling},
    pdfauthor={Andre Herz, Matthijs Pals, Daniel Durstewitz, Georgia Koppe}
}

\begin{document}
\maketitle
\begin{abstract}
Dynamical systems reconstruction (DSR) aims to learn surrogate models that capture the dynamics underlying time-series data. Reliable scientific inference with these surrogates requires uncertainty estimates consistent with the learned dynamics. We expose a \emph{dynamic-probabilistic consistency} (DPC) gap: optimizing finite-horizon probabilistic objectives can degrade the learned dynamics or decouple predictive uncertainty from the local tangent dynamics it should reflect. We isolate three mechanisms behind this gap: core collapse, noise masking, and blind uncertainty. Specifically, we show that open-loop rollout objectives can penalize covariance growth driven by the dynamics, encouraging optimization shortcuts that weaken local expansion. To mitigate this gap, we propose KAFFEE (Kalman-Aware Framework For Ergodic Emulation), a differentiable extended Kalman filter-based training framework that jointly learns dynamics and noise covariances by coupling filtered innovation likelihoods with covariance propagation through learned local Jacobians, while differentiating through the Kalman updates. On stochastic hyperchaotic Lorenz-96, KAFFEE mitigates the identified failure modes and improves both reconstruction of dynamical invariants and predictive scores relative to open-loop objectives. The DPC gap also appears when probabilistically adapting a DSR foundation model across 13 chaotic systems, where KAFFEE improves predictive likelihood over open-loop adaptation while preserving the pretrained model's long-term dynamics and enabling in-context Bayesian filtering.
\end{abstract}
% !TEX root = main.tex

\section{Introduction}
\label{sec:intro}

\begin{figure}[!t]
 \centering
  \includegraphics[width=\linewidth]{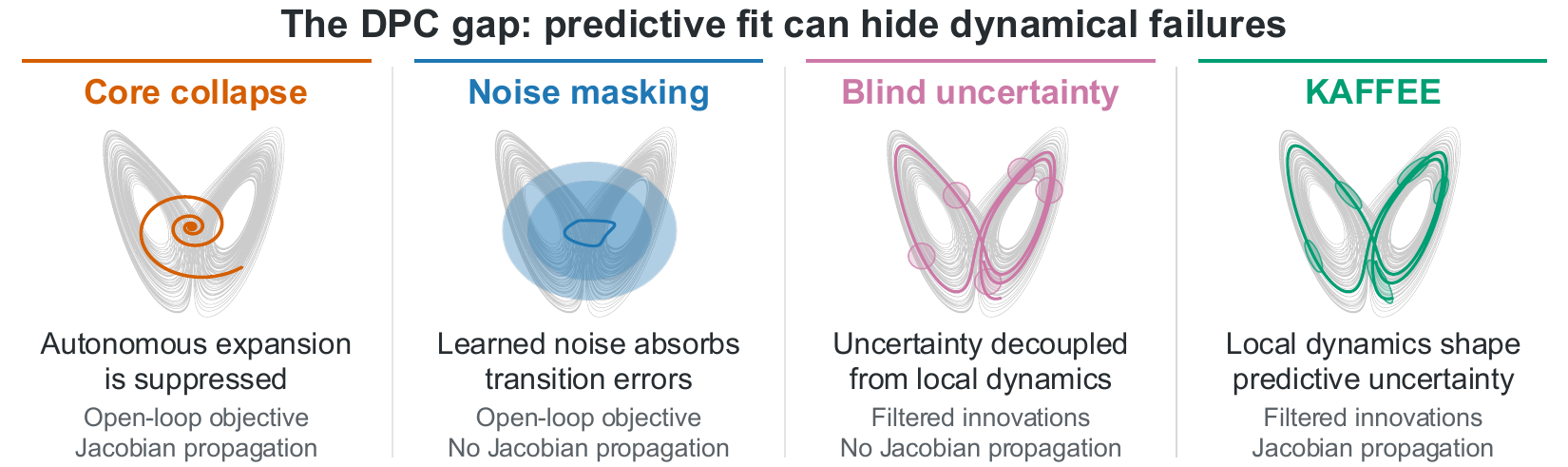}
 \caption{\textbf{The DPC gap in probabilistic DSR.} Schematic comparison of three DPC failure modes and KAFFEE. Finite-horizon probabilistic scores can improve while autonomous dynamics deteriorate or uncertainty remains blind to local dynamics. KAFFEE pairs filtered innovation scoring with Jacobian-based covariance propagation to mitigate the DPC gap. Grey curves illustrate ground-truth dynamics, colored curves learned autonomous dynamics, and shading learned predictive uncertainty.}
  \label{fig:taxonomy}
\end{figure}

Dynamical systems reconstruction (DSR) aims to learn generative neural surrogates of latent dynamics from observed time series \citep{pathak2017using,vyas2020computation,durstewitz2023reconstructing,durstewitz2026position}. Short-horizon predictive accuracy is not sufficient here \citep{park2024when}; rather, the surrogate is asked to act as an autonomous scientific model. It should therefore generate open-loop rollouts that reproduce the long-run dynamical structure of the true system, including attractor geometry, invariant measures, and chaotic expansion rates as quantified by its Lyapunov exponents \citep{goering2024outofdomain,botvinick-greenhouse2025invariant}. 
These properties are central whenever the surrogate is used beyond forecasting, e.g., invariant-measure estimation \citep{jiang2023training, schiff2024dyslim}, bifurcation analysis \citep{luo2024reconstructing}, or perturbation-response studies \citep{falasca2025probing, vinograd2024causal}.

Scientific applications also require probabilistic state estimation and \emph{uncertainty-aware} prediction from noisy observations \citep{gottwald2021supervised,psaros2023uncertainty}. In settings such as data assimilation \citep{reich2015probabilistic}, extreme-event forecasting \citep{price2025probabilistic}, or control under uncertainty \citep{hewing2020cautious}, the model must not only forecast accurately on average, but also represent how uncertainty evolves under the dynamics, including the amplification of small state errors \citep{trevisan2011kalman}. This motivates \emph{probabilistic DSR}: a surrogate should couple autonomous physical fidelity with \emph{dynamically grounded local uncertainty}, tied to the same structure that governs dynamical expansion and contraction.

Achieving this goal creates a structural
optimization tension. Probabilistic objectives balance data fit against predictive uncertainty \citep{mackay1992bayesian,rasmussen2006gaussian}. In dynamical systems, uncertainty is propagated by the
same local Jacobians that govern long-run dynamics
\citep{eckmann1985ergodic}, thereby coupling unstable directions with rapid covariance growth \citep{trevisan2011kalman, carrassi2022data}. This can create several optimization shortcuts: open-loop rollout objectives may contract learned Jacobian products to suppress covariance growth (\emph{core collapse}); flexible noise parametrizations may absorb model error without correcting the learned autonomous dynamics (\emph{noise masking}); and covariance parametrizations decoupled from local Jacobians may remain calibrated on average while being blind to locally unstable episodes (\emph{blind uncertainty}).

We identify and examine these failure modes as instances of the \emph{dynamic-probabilistic consistency} (DPC) gap, a mismatch between improving finite-horizon probabilistic scores and preserving the dynamical-probabilistic properties required of a scientific surrogate. The gap appears when the training objective can improve predictive scores by changing the model in ways that damage autonomous dynamics or the dynamical consistency of its local uncertainty. A model may thus achieve better probabilistic fit while misattributing observed variability to noise rather than deterministic instability inherent to the generating dynamics. For emerging dynamical-system foundation models (DSFMs) \citep{hemmer2025true,lai2026panda,liu2026chaosnexus}, this raises the practical challenge of adding probabilistic state estimation and predictive uncertainty tied to local dynamics, while preserving their zero-shot autonomous behavior.

Our analysis suggests two design principles for probabilistic DSR. First, evaluate likelihood through filtered innovations, scoring each observation before assimilating it to localize uncertainty for the next prediction. Second, local predictive uncertainty should remain dynamically grounded, so that covariance reflects the local Jacobian structure that governs the autonomous dynamics. We introduce KAFFEE (Kalman-Aware Framework For Ergodic Emulation) as a realization of these principles. KAFFEE uses a differentiable Extended Kalman Filter (EKF) and optimizes the negative log-likelihood (NLL) of the EKF innovation sequence \citep{gupta1974computational, gorad2020parameter, sarkka2023bayesian}.

\textbf{Contributions.}
First, we \emph{identify and formalize the DPC gap} and characterize three failure modes.
We isolate open-loop covariance-volume pressure as a concrete mechanism for one failure mode, \emph{core collapse}, in finite-horizon rollout objectives, where likelihood can penalize the Jacobian products that determine chaotic expansion.
Second, we introduce a \emph{three-axis evaluation framework} for probabilistic DSR that separates local probabilistic behavior, stochastic dynamical fidelity, and deterministic dynamical fidelity. On stochastic Lorenz-96 \citep{lorenz1996predictability}, this isolates \emph{core collapse}, \emph{noise masking}, and \emph{blind uncertainty}.
Third, we propose \emph{KAFFEE}, a differentiable EKF-based framework that mitigates the DPC gap.
Fourth, we show that the DPC gap arises when probabilistically adapting the pretrained DSFM DynaMix \citep{hemmer2025true} across 13 held-out chaotic systems, and that KAFFEE enables in-context Bayesian filtering while largely preserving transferred autonomous dynamics.

\section{Preliminaries}
\label{sec:prelim}

\subsection{Probabilistic DSR: the core and the shell}
\label{subsec:prelim_dsr}

We study the reconstruction of a chaotic dynamical system from time series using generative neural surrogates. Let $\bm z_{t+1}=F_\star(\bm z_t)$ denote an unknown, discrete-time autonomous system with latent states $\bm z_t\in\mathcal Z\subset\mathbb R^M$ and observations $\bm x_t\in\mathbb R^N$. We approximate the unknown dynamics $F_\star$ by a neural surrogate $F_{\theta_c}$ and embed it in a probabilistic state-space model 
(SSM). Specifically, we consider a Gaussian SSM with nonlinear transitions and linear observations: \begin{align*}
p_{\theta}(\bm z_{t+1}\mid \bm z_t) = \mathcal N\big(\bm z_{t+1}; F_{\theta_c}(\bm z_t),\bm\Gamma_{\theta_s}\big),
\quad p_{\theta_s}(\bm x_t\mid \bm z_t) = \mathcal N\big(\bm x_t;\bm B\bm z_t,\bm\Lambda_{\theta_s}\big),
\end{align*}
with $\bm\Gamma_{\theta_s} \in \mathbb{R}^{M \times M}$ and $\bm\Lambda_{\theta_s} \in \mathbb{R}^{N \times N}$ denoting the process and observation covariances, and $\bm B \in \mathbb{R}^{N \times M}$ denoting the observation mapping (held fixed in our experiments). The model parameters are partitioned into $\theta=(\theta_c,\theta_s)$.
The \emph{deterministic core} $\theta_c$ controls the autonomous transition mapping. It determines the local Jacobians $\bm J_{\theta_c}(\bm z) := \nabla_{\bm z} F_{\theta_c}(\bm z)$, whose chain products govern local stretching and contraction, yielding the system's dynamical invariants and shaping the attractor geometry \citep{eckmann1985ergodic}. The \emph{stochastic shell} $\theta_s$ parametrizes the uncertainty around this deterministic core, here through $\bm\Gamma_{\theta_s}$ and $\bm\Lambda_{\theta_s}$.

\subsection{Marginal likelihood, filtering, and covariance propagation}
\label{subsec:evidence_filtering}
Under the standard conditional independence and Markov assumptions of an SSM, the observed-data marginal likelihood (evidence) of a sequence $\bm{x}_{1:T}$ is given by 
\begin{equation}
\label{eq:observed_evidence}
p_\theta(\bm x_{1:T}) = \int\left(\prod_{t=1}^T p_{\theta_s}(\bm x_t\mid \bm z_t)\right) p_\theta(\bm z_{1:T}) d\bm z_{1:T},
\end{equation}
where $p_\theta(\bm z_{1:T}) = p_\theta(\bm z_1)\prod_{t=1}^{T-1}p_\theta(\bm z_{t+1}\mid \bm z_t)$ is the surrogate's \emph{open-loop latent trajectory prior} distribution, generated by rolling out the autonomous core and stochastic shell without conditioning on intermediate observations. In the following, we will see how this prior interacts with the likelihood.

Evaluating Eq.~\ref{eq:observed_evidence} requires marginalizing the latent trajectory, typically via recursive filtering. This involves the predictive densities $p_\theta(\bm z_t \mid \bm x_{1:t-1}) = \int p_\theta(\bm z_t \mid \bm z_{t-1}) p_\theta(\bm z_{t-1} \mid \bm x_{1:t-1}) d\bm z_{t-1}$, which are intractable for nonlinear $F_{\theta_c}$. A common approach is approximating the predictive distributions as Gaussian, $p(\bm z_t \mid \bm x_{1:t-1}) \approx \mathcal{N}(\bm z_{t|t-1}, \bm\Sigma_{t|t-1})$, as in the EKF, where the predict step analytically propagates the filtered covariance through the local Jacobians \citep{sarkka2023bayesian}:
\begin{equation}
\label{eq:ekf_predict_cov}
\bm\Sigma_{t|t-1} = \bm J_{\theta_c,t}\bm\Sigma_{t-1|t-1}\bm J_{\theta_c,t}^\top + \bm\Gamma_{\theta_s},\qquad \bm J_{\theta_c,t} = \nabla_{\bm z}F_{\theta_c}(\bm z_{t-1|t-1}).
\end{equation}
The likelihood is then given recursively via $p_\theta(\bm x_{1:T}) = p_\theta(\bm x_1) \prod_{t=2}^T p_\theta(\bm x_t \mid \bm x_{1:t-1})$. This evaluates the \emph{innovations}:
$\bm x_t -\bm{Bz}_{t|t-1}$, which have covariance given by: 
\begin{align}
\label{eq:innovation}
\bm\Omega_t = \bm B\bm\Sigma_{t|t-1}\bm B^\top + \bm\Lambda_{\theta_s}.
\end{align} 
Thus, the same local Jacobians that define the autonomous dynamics also determine how predictive uncertainty is propagated. This coupling implies that uncertainty cannot be treated as an independent calibration layer in probabilistic DSR; rather, predictive calibration must be evaluated together with the dynamics that propagate uncertainty.

\section{Theory and Methods}
\label{sec:theory_methods}

\subsection{Defining the DPC gap}
\label{subsec:dpc_gap}
Let $\mathcal S_H(\theta)$ be a finite-horizon probabilistic training objective (e.g., NLL) and let $\mathcal D(\theta)$ be a certification metric for the quality of the reconstructed dynamics; both defined such that lower values are better.
\begin{definition}[Dynamic-Probabilistic Consistency gap]
A parameter update step $\theta \to \theta'$ exhibits a DPC gap with respect to $(\mathcal S_H,\mathcal D)$ if
\[
\mathcal S_H(\theta') < \mathcal S_H(\theta)
\quad\text{but}\quad
\mathcal D(\theta') > \mathcal D(\theta).
\]
\end{definition}

That is, the optimized probabilistic criterion can favor an update that improves finite-horizon predictive fit while worsening the certification loss, which encodes properties required for probabilistic DSR. In our experiments, we instantiate $\mathcal D$ along three axes: Local probabilistic behavior, shell-on stochastic dynamical fidelity, and shell-off deterministic dynamical fidelity (see Section~\ref{subsec:eval-framework}).

\subsection{How the DPC gap manifests}
\label{subsec:dpc_mechanisms}

How the DPC gap manifests depends on two design choices: on how the training objective scores uncertainty, and on how the model parameterizes that uncertainty. These choices give rise to three failure modes we discuss here and isolate empirically in Section~\ref{sec:results}.

\textbf{Core collapse from open-loop volume pressure.} Learned surrogates are ultimately used as autonomous simulators. Training must address the exposure bias that can arise when one-step or teacher-forced objectives condition on noisy data, whereas autonomous prediction requires purely model-generated states \citep{doya1992bifurcations,levine2022framework,bengio2015scheduled,ranzato2016sequence}.
Finite-window rollout losses are therefore common in system identification and DSR \citep{hess2023generalized, forgione2021continuoustime,beintema2021nonlinear,beintema2023deep}. In probabilistic analogues of this setting, scoring multi-step predictive densities makes uncertainty propagation part of the objective \citep{girard2002multiplestep,doerr2018probabilistic}.

Suppose an $h$-step predictive density is approximated by a Gaussian, $p_\theta(\bm x_{t+h}\mid \bm x_{1:t})=\mathcal N(\bm x_{t+h};\bm\mu_{\theta,t+h},\bm\Omega_{\theta,t+h})$. Its NLL then reads
\[-\log p_\theta(\bm x_{t+h}\mid \bm x_{1:t}) = \frac12\log\det\bm\Omega_{\theta,t+h}+\frac12\|\bm x_{t+h}-\bm\mu_{\theta,t+h}\|^2_{\bm\Omega_{\theta,t+h}^{-1}} + \mathrm{const}.
\]
The Mahalanobis-distance term rewards data fit, while the log-determinant term penalizes predictive volume. The latter term in the Bayesian marginal likelihood is referred to as an ``automatic Occam’s razor'' \citep{rasmussen2006gaussian,mackay1992bayesian}, favouring models that concentrate probability mass while penalizing unnecessarily large predictive volumes. In standard probabilistic forecasting this prevents the model from explaining 
errors by becoming arbitrarily diffuse. 

In chaotic DSR, however, physically correct local expansion also increases predictive volume. This mechanism becomes apparent 
by considering the $h$-step Jacobian product  $\bm\Phi_{\theta,t+h:t} := \bm J_{\theta_c,t+h}\cdots \bm J_{\theta_c,t+1}$. Crucially, the same products govern perturbation growth and their asymptotic singular-value growth rates determine the Lyapunov spectrum $\lambda_1 \geq \dots \geq \lambda_M$, with leading Lyapunov exponent defined by $\lambda_1 = \lim_{h\to\infty} \tfrac1h \log \Vert \bm\Phi_{\theta,t+h:t} \Vert_2$ \citep{eckmann1985ergodic}.

Under local linearization, the open-loop covariance $\bm\Sigma^{\mathrm{open}}_{t+h} \in \mathbb{R}^{M \times M}$ of the surrogate's latent predictive density $p(\bm{z}_{t+h}\mid \bm{x}_{1:t})$ is given by repeatedly applying Eq.~\ref{eq:ekf_predict_cov} without intermediate assimilation:
\[
\bm\Sigma^{\mathrm{open}}_{t+h} \approx \bm\Phi_{\theta,t+h:t}\bm\Sigma_t\bm\Phi_{\theta,t+h:t}^{\top} + \sum_{j=0}^{h-1}\bm\Phi_{\theta,t+h:t+j+1}\bm\Gamma_{\theta_s}\bm\Phi_{\theta,t+h:t+j+1}^{\top}.
\]
The corresponding observation covariance is $\bm\Omega^{\mathrm{open}}_{\theta,t+h} = \bm B\bm\Sigma^{\mathrm{open}}_{t+h}\bm B^\top + \bm\Lambda_{\theta_s}$ (cf. Eq.~\ref{eq:innovation}). Thus, the log-determinant term directly penalizes volumes generated by products of the learned Jacobians. Chaotic systems inherently expand along at least one unstable direction ($\lambda_1 > 0$), causing these volumes to grow rapidly under autonomous propagation \citep{trevisan2011kalman,carrassi2022data}. The optimizer can alleviate this volume pressure by shrinking these products, thereby compromising the autonomous dynamics. Appx.~\ref{app:volume_pressure_proof} formalizes this mechanism by showing that the open-loop log-determinant term increases monotonically along local Jacobian-expansion paths. This isolates the volume term; the Mahalanobis (data-fit) term in the NLL can oppose it.  A similar volume term also appears for other exponential family objectives (given a Laplace approximation), see Appx.~\ref{app:exp_fam}. 

Fig.~\ref{fig:tent_nll_decomposition} makes this tradeoff explicit on the tent map $f_s(z) = s\min\{z,1-z\}$ with additive Gaussian process and observation noise. The open loop objective overestimates the spread of the true distribution. As a result the predictive volume term shifts the open-loop optimum toward weaker expansion. In contrast, filtering localizes uncertainty before each prediction, yielding calibrated predictive distributions and an objective that recovers the true slope. Further details in Appx.~\ref{app:volume_pressure_proof}.

\begin{figure}[!htbp]
  \centering
  \includegraphics[width=\linewidth]{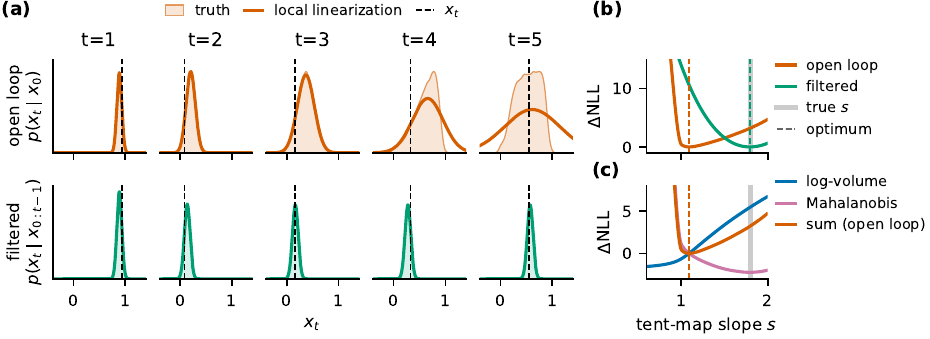}
  \caption{\textbf{Covariance-volume pressure in a noisy tent map.}
\textbf{(a)} Linearized open-loop predictives $p(x_t\mid x_0)$ (top) overestimate the spread of the
true density; filtered predictives $p(x_t\mid x_{0:t-1})$ (bottom) stay calibrated.
\textbf{(b)} $\Delta$NLL profiles vs.\ tent-map slope $s$: the open-loop optimum attenuates the slope, while the filtered optimum recovers it.
\textbf{(c)} Open-loop log-volume and Mahalanobis terms, shifted to zero at the open-loop optimum. The growing log-volume term pulls the
optimum down.}
\label{fig:tent_nll_decomposition}
\end{figure}

Note that the argument here does not imply that open-loop likelihood is intrinsically biased. Under a correctly specified model and with the exact observed-data likelihood, maximum likelihood targets the true data-generating process under standard identifiability assumptions. DPC-type mechanisms arise from practical circumstances in probabilistic DSR: under open-loop propagation, predictive distributions can become strongly non-Gaussian and poorly represented by a locally linearized Gaussian approximation.

\textbf{Noise masking through flexible stochastic shells.}
A further shortcut is possible when predictive uncertainty is parametrized
too flexibly or independently of Jacobian propagation. In this case, the optimizer may inflate $\bm\Gamma_{\theta_s}$, $\bm\Lambda_{\theta_s}$ (or a learned neural variance parametrization) to absorb transition errors without improving the deterministic core \citep{levine2022framework}. This can improve predictive scores despite dynamically incorrect autonomous rollouts. Related effects can also arise in variational state-space models when overly flexible components reduce the penalty on transition-model mismatch, or when independence assumptions are too restrictive \citep{ialongo2019overcoming,zhao2023revisiting}. Note that noise masking and core collapse can occur together, since a noise-inflated shell may conceal deterioration of the autonomous core.

\textbf{Blind uncertainty from missing covariance propagation.}
A third failure mode concerns the local uncertainty itself. A model can preserve reasonable autonomous dynamics and obtain competitive average predictive scores,
while its predictive covariance remains insensitive to state-dependent expansion and
contraction. This happens when the covariance propagation (Eq. \ref{eq:ekf_predict_cov}) is state-invariant. It may also happen when uncertainty is
parametrized separately from the learned Jacobians. Such a model may be
calibrated on average, but its local uncertainty is not dynamically grounded. We treat this as a DPC-gap instance where $\mathcal D$ assesses whether predictive uncertainty evolves in accordance with the true local tangent dynamics (measured in our experiments by $G_{\mathrm{tan}}$, see Section~\ref{subsec:eval-framework}).
Again, a related problem can arise in variational inference, with certain factorized variational families. If uncertainty is not propagated in the chosen factorization (such as in mean-field approximations, previously used for DSR \citep{brenner2022tractable}), learned dynamics can be biased \citep{turner2011two,bayer2021mind}.

\textbf{Relation to noisy-state attenuation.} The open-loop volume mechanism is related to, but distinct from, classical noisy-state attenuation. In noisy-state attenuation, noisy observations are treated directly as true states, turning one-step Gaussian prediction into an errors-in-variables regression and biasing the learned transition coefficient toward zero \citep{fuller1987measurement,staudenmayer2005measurement,mcsharry1999better}. This explains some failures of one-step training on noisy data, as illustrated by the scalar diagnostic in Appx.~\ref{app:attenuation_diagnostic}
and the non-chaotic controls in Appx.~\ref{app:nonchaotic_controls}.

\subsection{KAFFEE: Dynamically grounded local uncertainty}
\label{subsec:kaffee}

The preceding mechanisms suggest two desiderata for probabilistic DSR: likelihood should be evaluated through filtered innovations, and the covariance entering this local objective should remain coupled to the learned transition map. KAFFEE implements both using a differentiable EKF \citep{gupta1974computational,gorad2020parameter,sarkka2023bayesian}.

KAFFEE maintains a filtered Gaussian belief $p_\theta(\bm z_t\mid \bm x_{1:t}) \approx \mathcal N(\bm z_{t|t},\bm\Sigma_{t|t}).$ At each step, the deterministic core predicts the next-step mean $\bm z_{t|t-1} = F_{\theta_c}(\bm z_{t-1|t-1})$, and the covariance is propagated by the same local Jacobians that define the autonomous core (Eq.~\ref{eq:ekf_predict_cov}).
After prediction, observations $\bm x_t$ are assimilated through the Kalman gain $\bm K_t = \bm\Sigma_{t|t-1} \bm B^\top \bm\Omega_t^{-1}$, with the innovation covariance $\bm\Omega_t$ given by Eq. \ref{eq:innovation}. The covariance update can be written as
$\bm\Sigma_{t|t} = \bm\Sigma_{t|t-1} - \bm{K}_t\bm B\bm\Sigma_{t|t-1}$. 
Thus, assimilation contracts uncertainty before the next prediction step (we have $\bm\Sigma_{t|t}\preceq\bm\Sigma_{t|t-1}$; see Appx.~\ref{app:volume_pressure_proof}). This localizes the uncertainty scored by the likelihood before the next prediction. Appx.~\ref{app:ekf_equations} lists the complete EKF equations.

KAFFEE minimizes the resulting innovation NLL with stochastic gradient descent:
\begin{equation}
\label{eq:filtered_likelihood}
\mathcal L_{\mathrm{EKF}} = -\sum_{t=1}^T \log \mathcal{N}(\bm x_t;\bm B \bm z_{t\mid t-1},\bm\Omega_t)=\sum_{t=1}^T \frac12\log\det\bm\Omega_t + \frac12\Vert\bm x_t-\bm B\bm z_{t|t-1}\Vert^2_{\bm\Omega_t^{-1}} + \mathrm{const.}
\end{equation}
In the linear-Gaussian case, the Kalman filter gives the exact NLL; in nonlinear surrogates, the EKF provides a structured inductive bias for dynamically grounded local uncertainty. Appx.~\ref{app:anti_contraction} shows that the Mahalanobis term can outweigh the log-determinant once model covariance is underdispersed relative to the true residual spread.

While the underlying equations are those of a differentiable EKF, KAFFEE defines a dedicated training framework designed to translate deterministic DSR models into probabilistic surrogates with calibrated local uncertainty while reducing DPC-gap shortcuts. KAFFEE initializes the autonomous core $F_{\theta_c}$ from a pretrained deterministic surrogate (or zero-shot foundation model) and optimizes the full probabilistic model via windowed backpropagation through time (BPTT). This lets the model first capture the deterministic dynamical invariants, and then leverage the EKF's covariance propagation as a structural inductive bias for grounding local uncertainty. Furthermore, differentiating through the Kalman update makes KAFFEE a Bayesian-filtering analogue of forcing-based stabilization methods for chaotic recurrent training \citep{hess2023generalized,mikhaeil2022difficulty,sagtekin2025error,herz2026teacher}, similar to \citet{pals2024inferring}; see Section~\ref{sec:related_work}.

\section{Experimental Setup}
\label{sec:exp_setup}

We examine the mechanisms behind the DPC gap and evaluate KAFFEE in two complementary settings: an interpretable mechanistic ablation with piecewise-affine recurrent neural networks (RNNs) \citep{koppe2019identifying,brenner2024almostlinear}, and a scale-up using the pretrained DSFM DynaMix on a range of chaotic systems \citep{hemmer2025true}.

\subsection{Deterministic priors and baselines}
\label{subsec:benchmarks}
 
All probabilistic baselines are initialized from the same deterministic core: an almost-linear RNN (AL-RNN) \citep{brenner2024almostlinear} trained via BPTT with sparse teacher forcing (STF) \citep{mikhaeil2022difficulty}. AL-RNNs partition the state space into polyhedral regions with explicit local linear dynamics ($\bm z_{t+1} = \bm A\bm z_t + \bm W\,\bm\phi^\ast(\bm z_t) + \bm h$), where $\bm\phi^\ast$ denotes ReLU activation, applied only to a subset of latent coordinates. This provides a powerful tangent space approximation of the underlying system with exact, closed-form regionwise Jacobians that KAFFEE exploits directly for covariance propagation. Their piecewise-affine structure also simplifies backpropagation (Appx.~\ref{app:kaffee_scaling_appendix}). Details on AL-RNN pretraining are in Appx.~\ref{app:alrnn_implementation_details}.

To isolate the mechanisms behind the DPC gap, all probabilistic models share the same deterministic core and adaptation budget. As illustrated in Fig.~\ref{fig:taxonomy}, we structure our baselines as a mechanistic taxonomy along two axes: \emph{Objective} (Open-Loop vs.\ Filtered Innovation) and \emph{Parameterization} (Jacobian-tied covariance propagation vs.\ no propagation).
These controls separate the effects of objective and covariance propagation: open-loop training can cause both \emph{core collapse} and \emph{noise masking}, while no-propagation covariance yields \emph{blind uncertainty}. For the ``tied-propagation'' baselines, the predictive covariance is propagated through the learned local Jacobian (Eq.~\ref{eq:ekf_predict_cov}), whereas in ``no-propagation'' this is replaced by a state-invariant recursion, $\bm\Sigma_{t\mid t-1}=\bm\Sigma_{t-1\mid t-1}+\bm\Gamma$. The open-loop objectives score Gaussian rollouts while filtered objectives optimize innovation NLL (Eq.~\ref{eq:filtered_likelihood}). We also include a \emph{Frozen-Core} baseline to isolate the effect of post-hoc calibration. It uses the filtered covariance but keeps $F_{\theta_c}$ fixed.

To demonstrate the pathology at scale, we probabilistically adapt DynaMix, a pretrained DSFM for zero-shot inference \citep{hemmer2025true}. Here, the modular nature of KAFFEE allows us to embed the pretrained DSFM into an SSM framework. In particular, we attach a Gaussian shell to DynaMix  and evaluate whether probabilistic fine-tuning allows for calibrated uncertainty while preserving its zero-shot transferred physics. Additional model and implementation details 
are provided in Appx.~\ref{app:dynamix_practical_details}.

\subsection{Three-Axis Evaluation Framework}
\label{subsec:eval-framework}

To empirically expose the DPC gap, we evaluate surrogates along three axes (details in Appx.~\ref{app:metric_computation}):

\textbf{(1) Local probabilistic behavior} assesses local predictive performance at a physically meaningful horizon ($H=\tfrac14\tau_\lambda$, where $\tau_\lambda = \lambda_1^{-1}$ is the Lyapunov time). We quantify this using ensembles over evenly spaced forecast starts to compute held-out NLL and continuous ranked probability score (CRPS). Because these metrics are time-averaged, they can be dominated by easily predictable regions of the attractor and can obscure overconfidence during highly unstable episodes. To assess whether the covariance propagation of the model agrees with the local expansion of the dynamics, we introduce $G_{\mathrm{tan}}$, a Spearman correlation between process-noise-free log variance growth under the surrogate's covariance propagation and the corresponding growth under the true local Jacobians, pooled across observed coordinates.

\textbf{(2) Stochastic dynamical fidelity} evaluates long shell-on rollouts of the surrogate against the noisy ground truth via state-space divergence ($D_{\mathrm{stsp}}^{\mathrm{stoch}}$), which measures how well the state-space occupancy of generated trajectories agrees with that of the reference dynamics, and unstable-volume growth rate error ($|\Delta h_\lambda^{\mathrm{stoch}}|$), quantified by the absolute reconstruction error in the sum of positive Lyapunov exponents. The sum of these exponents is a fundamental dynamical invariant related to the information production rate of a system \citep{eckmann1985ergodic}.

\textbf{(3) Deterministic dynamical fidelity} isolates the deterministic core (shell removed) and evaluates it against the noise-free true system using state-space divergence ($D_{\mathrm{stsp}}^{\mathrm{core}}$), spectral overlap using Hellinger distance ($D_\mathrm{H}^{\mathrm{core}}$), and deterministic unstable-volume growth rate error ($|\Delta h_\lambda^{\mathrm{core}}|$).

\section{Results}
\label{sec:results}

\subsection{Diagnosing the DPC Gap: Mechanistic ablation on stochastic Lorenz-96}
\label{subsec:alrnn_results}

\begin{figure}[ht]
\centering
\includegraphics[width=\textwidth]{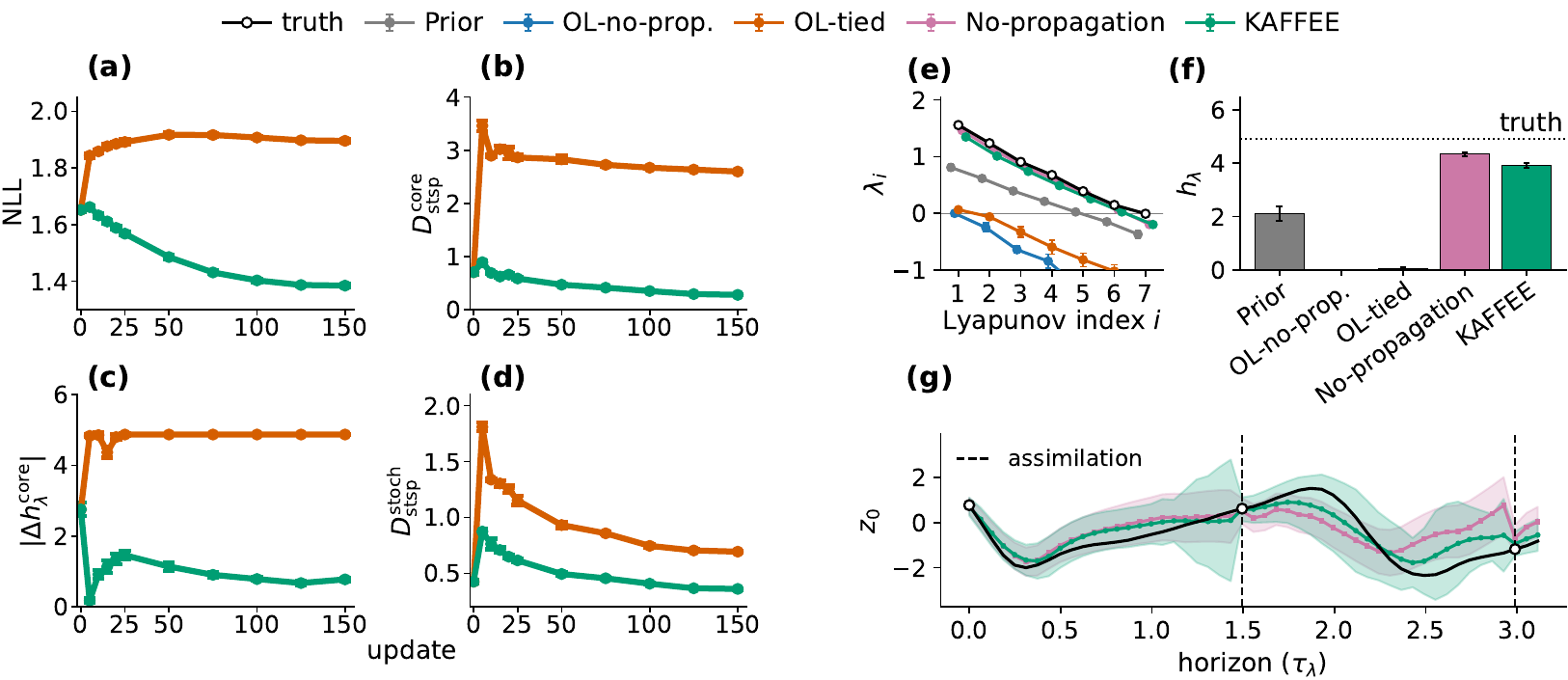}
\caption{\textbf{Stochastic Lorenz-96 (20D) 
results.} \textbf{Left: DPC gap.}
\textbf{(a)} Held-out NLL, \textbf{(b)} deterministic 
state-space divergence ($D_\mathrm{stsp}$), \textbf{(c)} unstable-volume growth error, and \textbf{(d)} stochastic $D_\mathrm{stsp}$.
\textbf{Right: Invariant reconstruction and dynamically grounded local uncertainty.} \textbf{(e)} Median reconstruction of positive deterministic Lyapunov exponents $\lambda_i$.
Only filtered innovations outperform the deterministic prior. \textbf{(f)} Corresponding unstable-volume growth rate $h_\lambda^\mathrm{core}$. \textbf{(g)} Single Lorenz-96 latent; shaded bands 
= predictive $\pm 2\sigma$; assimilation after $1.5\,\tau_\lambda$. KAFFEE grounds local uncertainty through state-dependent Jacobians; no-propagation exhibits blind uncertainty. Errors $=$ SEs over $20$ seeds.
}
\label{fig:l96_results}
\end{figure}

We examine the DPC gap on the Lorenz-96 system \citep{lorenz1996predictability} (dimensionality $N=20$, forcing $F=8$) under substantial noise ($\sigma_{\mathrm{proc}}=\sigma_{\mathrm{obs}}=0.1$ in standardized coordinates), see Appx.~\ref{app:l96_dataset}. Lorenz-96 is hyperchaotic, i.e., it exhibits multiple expanding directions.

Fig.~\ref{fig:l96_results} shows the DPC failure modes in action: under matched covariance propagation and adaptation budgets, minimizing the open-loop NLL objective degrades the autonomous core, while filtered innovation objectives improve it (Panels~a--d). Note that Panel~(a) shows held-out ensemble NLL, not the optimized training loss. Panels~(b) and (c) show autonomous core evolution during training in terms of state-space divergence $D_{\mathrm{stsp}}^{\mathrm{core}}$ and unstable-volume growth rate error ($|\Delta h_\lambda^{\mathrm{core}}|$). Open-loop tied rapidly suppresses the core's expansion, while $D_{\mathrm{stsp}}^{\mathrm{stoch}}$ continues to improve late in adaptation (Panel~d). Thus, shell-on rollouts can look more faithful in terms of probabilistic dynamical fidelity, even when the autonomous core has deteriorated. The remaining Panels~e--g show the reconstruction of the dynamical invariant $h_\lambda$ and the local dynamical grounding of uncertainty. Both filtered objectives (KAFFEE, no-propagation) retain the system's intrinsic hyperchaotic expansion, which open-loop objectives collapse and the deterministic prior struggles to capture (Panels~e--f). KAFFEE combines improved autonomous reconstruction with dynamically grounded local uncertainty. In contrast, the no-propagation baseline re-adds fixed process noise $\bm\Gamma$ at every step, rendering its predictive variance state-invariant and overconfident during unstable episodes (Panel~g).

Table~\ref{tab:main_l96_ablation} maps our empirical results to the mechanistic taxonomy of Fig.~\ref{fig:taxonomy} and the three-axis evaluation framework (Section~\ref{subsec:eval-framework}). The failure modes manifest as predicted. \emph{Core collapse} (open-loop + tied propagation) destroys long-run fidelity. \emph{Noise masking} (open-loop + no propagation) inflates the learned observation shell (median $\mathrm{tr}(\bm\Lambda)/N$ is $0.29$ vs.\ $0.10$ for KAFFEE), while the core also deteriorates. $|\Delta h_\lambda^{\mathrm{core}}|$ saturates for both open-loop objectives, since those cores are not producing expansive dynamics in any direction. \emph{Blind uncertainty} (filtered + no propagation) retains much of the improvement in autonomous reconstruction, yet its state-invariant covariance recursion cannot reflect local physical expansion: its $G_{\mathrm{tan}}$ is zero by construction. KAFFEE is the only method that pairs a large positive $G_{\mathrm{tan}}$ with the best held-out NLL as well as substantially improved deterministic dynamics over the prior. Frozen-core shell calibration shows that calibrating a shell around a fixed prior can produce positive $G_{\mathrm{tan}}$, but it cannot improve the deterministic dynamics of the prior.

\begin{table}[t]
\caption{\textbf{Mechanistic ablation on stochastic Lorenz-96.} All probabilistic baselines share the \emph{same} end-to-end training schedule and are initialized from the same deterministic AL-RNN cores. \textbf{Bold:} median whose $\pm$SE interval overlaps the column's best; \textcolor{red}{red}: signature failure mode (bracketed label). Medians $\pm$ 
SEs across $20$ seeds; see Appx.~\ref{app:metric_computation}.}
\centering
\setlength{\tabcolsep}{2.2pt}
\renewcommand{\arraystretch}{1.0}
\resizebox{\linewidth}{!}{%
\begin{tabular}{@{}l ccc cc ccc@{}}
\toprule
 & \multicolumn{3}{c}{\textbf{Local Probabilistic Behavior}} & \multicolumn{2}{c}{\textbf{Stochastic  Dynamics (Shell+Core)}} & \multicolumn{3}{c}{\textbf{Deterministic Dynamics (Core)}} \\
\cmidrule(lr){2-4} \cmidrule(lr){5-6} \cmidrule(lr){7-9}
\textbf{Model / Objective} & NLL ($\tfrac14\tau_\lambda$) $\downarrow$ & CRPS ($\tfrac14\tau_\lambda$) $\downarrow$ & $G_{\mathrm{tan}}$ $\uparrow$ & $D_{\mathrm{stsp}}^{\mathrm{stoch}}$ $\downarrow$ & $|\Delta h_\lambda^{\mathrm{stoch}}| \downarrow$ & $D_{\mathrm{stsp}}^{\mathrm{core}}$ $\downarrow$ & $|\Delta h_\lambda^{\mathrm{core}}| \downarrow$ & $D_\mathrm{H}^{\mathrm{core}}$ $\downarrow$ \\
\midrule
\multicolumn{9}{@{}l}{\textbf{\emph{Deterministic prior}}}\\
\hspace{1em}AL-RNN & -- & -- & -- & -- & -- & $0.708{\pm}0.057$ & $2.766{\pm}0.250$ & $0.338{\pm}0.020$\\
\midrule
\hspace{1em}Frozen-Core Shell Calibration & $1.428{\pm}0.016$ & {\bfseries\boldmath $0.545{\pm}0.004$} & $0.283{\pm}0.011$ & {\bfseries\boldmath $0.355{\pm}0.005$} & $1.982{\pm}0.058$ & $0.708{\pm}0.057$ & $2.766{\pm}0.250$ & $0.338{\pm}0.020$\\
\multicolumn{9}{@{}l}{\textbf{\emph{Open-Loop objective}}}\\
\hspace{1em}No propagation \textcolor{gray}{[Noise Masking]} & \textcolor{red}{$2.154{\pm}0.013$} & \textcolor{red}{$1.527{\pm}0.017$} & \textcolor{red}{$0.000{\pm}0.000$} & \textcolor{red}{$1.147{\pm}0.010$} & \textcolor{red}{$5.120{\pm}0.000$} & \textcolor{red}{$2.647{\pm}0.012$} & \textcolor{red}{$4.868{\pm}0.014$} & \textcolor{red}{$0.834{\pm}0.007$}\\
\hspace{1em}Tied propagation \textcolor{gray}{[Core Collapse]} & $1.898{\pm}0.008$ & $0.956{\pm}0.019$ & $0.322{\pm}0.012$ & $0.685{\pm}0.013$ & \textcolor{red}{$5.120{\pm}0.000$} & \textcolor{red}{$2.645{\pm}0.016$} & \textcolor{red}{$4.868{\pm}0.000$} & \textcolor{red}{$0.845{\pm}0.003$}\\
\multicolumn{9}{@{}l}{\textbf{\emph{Filtered objective}}}\\
\hspace{1em}No propagation \textcolor{gray}{[Blind Uncertainty]} & $1.437{\pm}0.005$ & $0.569{\pm}0.003$ & \textcolor{red}{$0.000{\pm}0.000$} & $0.492{\pm}0.009$ & {\bfseries\boldmath $0.421{\pm}0.046$} & $0.310{\pm}0.009$ & {\bfseries\boldmath $0.395{\pm}0.075$} & {\bfseries\boldmath $0.128{\pm}0.006$}\\
\hspace{1em}Tied propagation (KAFFEE) & {\bfseries\boldmath $1.388{\pm}0.006$} & {\bfseries\boldmath $0.551{\pm}0.003$} & {\bfseries\boldmath $0.528{\pm}0.003$} & $0.369{\pm}0.005$ & $0.710{\pm}0.046$ & {\bfseries\boldmath $0.291{\pm}0.008$} & $0.811{\pm}0.074$ & {\bfseries\boldmath $0.130{\pm}0.003$}\\
\bottomrule
\end{tabular}
}
\label{tab:main_l96_ablation}
\end{table}

These comparisons also show why no single metric is sufficient alone. The advantage of filtered over open-loop objectives in NLL and CRPS persists across the tested forecast horizons up to one Lyapunov time (Appx.~\ref{app:forecast_skill_spectrum}). KAFFEE has the lowest NLL, frozen-core shell calibration has a slightly lower CRPS, and filtered no-propagation has the smallest $|\Delta h_{\lambda}^{\mathrm{core}}|$ while $G_{\mathrm{tan}}$ stays at zero. Conversely, a positive $G_{\mathrm{tan}}$ can coexist with core collapse, as visible for open-loop tied. Predictive scores, tangent-propagation alignment and autonomous/stochastic dynamical fidelity therefore answer different questions and require separate evaluation.

Appx.~\ref{app:l63} shows that these DPC gap signatures are not Lorenz-96-specific. On stochastic (3D) Lorenz-63, open-loop objectives still collapse the autonomous dynamics, and no-propagation filtering remains locally blind. KAFFEE avoids that collapse while keeping local uncertainty dynamically grounded. However, in the tested setting, the deterministic prior was already faithful, illustrating a limit of joint adaptation: frozen-core calibration preserves the autonomous reconstruction better and attains better local predictive scores. Hence joint adaptation is not uniformly preferable, especially when the prior is already close to the true deterministic dynamics.

Appx.~\ref{app:l96_mechanism_ablations} further tests the role of the Kalman update during training. Stopping gradients through its mean and covariance updates worsens deterministic reconstruction despite identical forward updates at fixed parameters. This supports an important contribution induced by backpropagation through assimilation, consistent with the stabilization mechanism discussed in Section~\ref{sec:related_work} and Appx.~\ref{app:kaffee_gtf_appendix}.

\subsection{In-context probabilistic adaptation of DynaMix}
\label{subsec:dynamix_results}
We next test whether the DPC gap appears when probabilistically adapting a pretrained DSFM. We use DynaMix \citep{hemmer2025true}, which can take in observations of a dynamical system in its context window, and then simulate the system's long-term evolution. Here, we adapt DynaMix to 13 held-out dynamical systems, contaminated by $10\%$ observation noise. All target-system learning is restricted to the 2{,}000 observations in DynaMix's context window. 

We compare three probabilistic adaptation modes: \emph{KAFFEE-DynaMix} jointly updates the DynaMix core and a Gaussian noise shell using the EKF innovation objective. The \emph{open-loop} control uses the same transferred core and shell initialization, but updates the core using a Gaussian open-loop NLL objective. Finally, the \emph{frozen-core} control optimizes only the diagonal shell under the EKF objective while freezing the DynaMix core; its autonomous dynamics are therefore identical to the pretrained core by construction. Details are reported in Appx. Table~\ref{tab:dynamix_results_table}.

Fig.~\ref{fig:dynamix_context_adaptation}a shows the resulting trade-off after 50 in-context optimization steps. Open-loop tied adaptation deteriorates the autonomous core ($D_{\mathrm{stsp}}^\mathrm{core}/D_{\mathrm{stsp}}^{\mathrm{base}}=1.30\pm0.11$); KAFFEE stays near the zero-shot attractor geometries ($1.01\pm0.01$). Over the 13 systems, KAFFEE has lower $D_{\mathrm{stsp}}^\mathrm{core}$ than open-loop no-propagation on $13/13$ systems and than open-loop tied on $12/13$. It also beats both open-loop methods on held-out NLL on $13/13$ systems. Held-out NLL is on par with frozen-core shell calibration, which leaves relative $D_{\mathrm{stsp}}^\mathrm{core}$ at 1 by construction. Fig.~\ref{fig:dynamix_context_adaptation}b illustrates the complementary filtering capability. After adaptation, DynaMix acts as the transition model in an EKF for Bayesian state estimation. Uncertainty is propagated via approximate local Jacobians (Appx.~\ref{app:dynamix_practical_details}), while sparse observations are assimilated online.

\begin{figure}[!htbp]
  \centering
  \includegraphics[width=\linewidth]{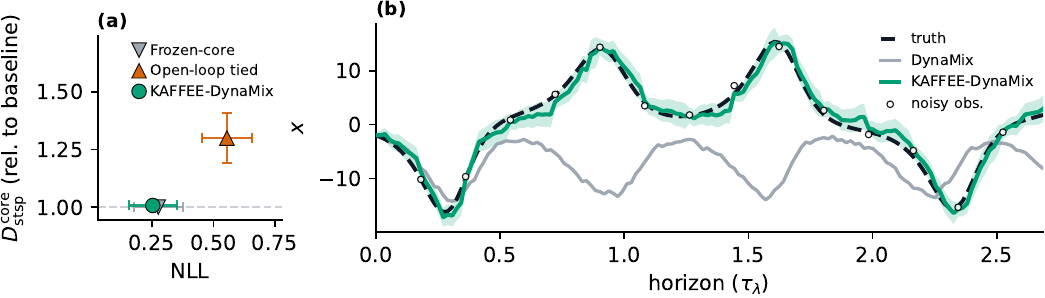}
  \caption{\textbf{DPC gap and online filtering in DynaMix.} \textbf{(a)} KAFFEE stays near the zero-shot attractor while open-loop tied inflates $D_{\mathrm{stsp}}^\mathrm{core}$. Markers are medians of the system-level seed-medians $\pm$ SEs across those $n=13$ systems. \textbf{(b)} KAFFEE turns DynaMix into an online Bayesian filter that assimilates sparse noisy observations; illustrated here on one Lorenz-63 coordinate \citep{lorenz1963deterministic}.}
  \label{fig:dynamix_context_adaptation}
\end{figure}

\section{Related Work}
\label{sec:related_work}
\textbf{Probabilistic DSR.} 
Recovering chaotic attractors from data is a central challenge in DSR. Learned surrogates are asked to act as autonomous generative models that preserve long-run dynamical structure \citep{pathak2017using,vyas2020computation,durstewitz2023reconstructing,goering2024outofdomain}. Probabilistic neural SSMs and SDEs add latent uncertainty, typically via finite-horizon variational or marginal-likelihood objectives \citep{brenner2022tractable,li2020scalable}. Finite-horizon test error in learned dynamical systems need not imply recovery of moments, invariant measures, or Lyapunov structure \citep{park2024when}. We identify objective-level failure modes by which probabilistic objectives can degrade those invariants or mask error through inflated noise. The same issues arise when retrofitting DSFMs \citep{hemmer2025true,liu2026chaosnexus,lai2026panda}. Concurrent work around scoring rule-based adaptation of deterministic simulators/foundation models for partial differential equations independently supports adapting strong deterministic cores rather than training probabilistic models from scratch \citep{diaconu2026probabilistic}. Our results suggest that such proper scores should be complemented by diagnostics of autonomous invariants and locally grounded uncertainty when the scientific target goes beyond finite-horizon forecasting.  

\textbf{Data assimilation and differentiable filtering.}
Data assimilation (DA) is a field overlapping with DSR that combines a dynamical model, noisy observations, and uncertainty to estimate states and improve forecasts \citep{reich2015probabilistic,carrassi2018data}. Chaotic DA in particular emphasizes how filter covariances are shaped by tangent dynamics \citep{trevisan2011kalman,carrassi2022data}. Differentiable EKFs are a standard tool in SSMs \citep{gupta1974computational,kokkala2016sigmapoint,sarkka2023bayesian,gorad2020parameter,haarnoja2016backprop,tanoh2025identifying}, and methods combining machine learning with DA use this machinery to learn or correct dynamics \citep{brajard2020combining,bocquet2021online,adrian2026autodifferentiable}. KAFFEE uses this established filtering machinery to implement a training principle for probabilistic DSR: evaluate a filtered, localized objective while keeping local uncertainty tied to the learned dynamics. For a Gaussian innovation NLL objective and Jacobian-based covariance propagation, this results in the shown differentiable-EKF instantiation. We show how this pairing mitigates the DPC gap and scales to in-context adaptation and online Bayesian filtering with DynaMix.

\textbf{Stabilizing BPTT for chaotic surrogates.} Gradient-based optimization of recurrent surrogates in chaotic regimes is notoriously difficult due to exploding Jacobian product chains \citep{mikhaeil2022difficulty}. Methods like GTF \citep{hess2023generalized} and Error Forcing \citep{sagtekin2025error} stabilize training by damping those products geometrically. However, determining forcing strengths or intervals remains difficult \citep{sip2026double}. In Appx.~\ref{app:kaffee_gtf_appendix} we derive how a Kalman update induces a state-adaptive anisotropic damping of the Jacobian chain for piecewise-affine latent transitions, which can attenuate amplification analogously to forcing methods (as also noted in \citet{pals2024inferring}). In Appx.~\ref{app:l96_mechanism_ablations} we test training through this update against GTF-like fixed scalar forcing strength and against stopping gradients through the Kalman update. On Lorenz-96, KAFFEE achieves better autonomous reconstruction than either control, even when selected fixed gains attain comparable NLL.

\section{Discussion}
\label{sec:discussion}
This work exposes the DPC gap: improving finite-horizon probabilistic training objectives can degrade properties required for DSR, and thus reliable scientific use. In Gaussian open-loop rollout objectives, covariance-volume pressure creates a shortcut toward suppressing physical expansion in the learned dynamics. Flexible observation and process noise parametrizations can also mask transition errors, while missing covariance propagation can leave uncertainty blind to what local dynamics dictate. Our three-axis evaluation separates local probabilistic behavior, stochastic dynamic fidelity, and deterministic dynamic fidelity. Across the AL-RNN and DynaMix experiments, filtered innovation objectives improve the balance between predictive fit and dynamical fidelity, mitigating the DPC gap. KAFFEE does so while also tying covariance propagation to local Jacobians.

In the appendix, we also examine practical extensions. A rank-reduced variant (rr-KAFFEE) achieves competitive performance with dense KAFFEE on Lorenz-96 at filter rank $r=20$ (matching the rank of the true process noise), while its measured training time scales approximately quadratically with latent size at fixed rank \citep{schmidt2023rankreduced} (Appx.~\ref{app:rrkaffee}). Using an expressive class of piecewise-affine RNNs (shPLRNNs \citep{hess2023generalized}), we also demonstrate probabilistic reconstruction on EEG recordings (Appx.~\ref{app:eeg}). 

\textbf{Limitations.} KAFFEE mitigates the observed DPC gap failure modes, but does not guarantee preservation of autonomous dynamics. If the deterministic prior is already close to the truth, joint learning of noise and deterministic transition can worsen that prior (Appx.~\ref{app:l63}). KAFFEE also inherits the EKF's local Gaussian approximation, which may not capture the full complexity of true predictive distributions in highly nonlinear regimes (e.g., for infrequent observations). KAFFEE-DynaMix uses approximate Jacobians, and its filtered no-propagation control remains competitive (Appx.~\ref{app:dynamix_practical_details}). Our certification diagnostic for dynamically grounded local uncertainty $G_\mathrm{tan}$ depends on knowledge of the true local Jacobians; on real-world data such as the studied EEG recordings, the unknown true dynamics limit conclusions about reconstruction and DPC gap mechanisms (Appx.~\ref{app:eeg}). The EKF's dense covariance propagation scales cubically with latent dimension (Appx.~\ref{app:kaffee_scaling_appendix}). Rank reduction lowers its cost in our benchmark, but it requires choosing an appropriate rank (Appx.~\ref{app:rrkaffee}).

\textbf{Future work.} To relax the EKF's local Gaussian assumption in strongly nonlinear or sparsely observed regimes, natural extensions include sigma-point, ensemble, particle, or structured variational approaches \citep{doucet2000rao,wan2000unscented,evensen2003ensemble,zhao2023revisiting,brenner2024integrating,pals2024inferring} (Appx.~\ref{app:empirical_density}). However, their scalability and ability to preserve KAFFEE's mitigation of the DPC gap remain to be established. The connection of KAFFEE to GTF in Appx.~\ref{app:kaffee_gtf_appendix} and ablations in Appx.~\ref{app:l96_mechanism_ablations} 
also suggest using the same stabilization that KAFFEE provides in other training objectives for dynamically faithful probabilistic surrogates.

Taken together, our results show that predictive fit and dynamical reconstruction require separate evaluation, and that a filtered objective tying uncertainty to learned dynamics can reconcile both properties. We hope this work paves the way for surrogate models that handle uncertainty in a principled way without sacrificing dynamical consistency.

Code, data, and trained models for all experiments, together with documentation for reproducing the results, are available at \url{https://github.com/aherz/kaffee-dpc}.
\clearpage
\begin{ack}
This work was funded under the Excellence Strategy of the German Federal and State Governments; by the Federal Ministry of Research, Technology and Space (BMFTR) through the NAILIt project under the neuroAI initiative (grant numbers 01GQ2509A and 01GQ2509B); by the German Research Foundation (DFG) through Research Unit FOR 5159, subproject 11, the Collaborative Research Centre 265 subproject B08, and individual grant DU 354/15-1; and by the Hector II Foundation.
\end{ack}

{\small
\bibliographystyle{abbrvdunsrt}
\bibliography{refs}
}

\clearpage
\appendix
\renewcommand{\figurename}{Appendix Figure}
\setcounter{figure}{0}
\section{Theory Details and Proofs}
\label{app:theory}
This section supplements the theory in the main text. We discuss the proof of volume pressure localization, a noisy-state attenuation
diagnostic, and a condition under which covariance contraction increases the expected (i.e., population) Gaussian NLL. We also discuss and formalize KAFFEE's connection to intervention-based gradient stabilization training methods, in particular generalized teacher forcing (GTF) \citep{hess2023generalized}.

\subsection{Alleviating core collapse via filtering localization}
\label{app:volume_pressure_proof}
In Section~\ref{subsec:dpc_mechanisms}, we discuss how Gaussian objectives applied to open-loop predictive distributions accumulate uncertainty under autonomous propagation.
In contrast, filtered innovation likelihood objectives repeatedly pull forecasts back through data assimilation, such that each likelihood term scores uncertainty propagated from a local post-assimilation belief. Here, we demonstrate this contrast with a diagnostic that isolates the covariance-volume component of the Gaussian objective. As in the main text, we denote filtered estimates by writing $\tau \mid t$ for an estimate at time $\tau$ conditioned on observations up to time $t$. Under local linearization, the open-loop covariance analyzed below is known as the \emph{free-forecast covariance} in Bayesian filtering/data assimilation \citep{carrassi2022data}. In the linear time-varying case, its accumulated process-noise term is known as the \emph{controllability matrix} \citep{jazwinski1970stochastic}.

The following Proposition~\ref{prop:volume_pressure_localization} is a diagnostic result for the covariance-volume term in the Gaussian NLL. This does not mean that the full Gaussian objective always contracts the learned dynamics. The Mahalanobis (data-fit) term can oppose this pressure, especially when the predictive distribution is underdispersed. Instead, the proposition identifies a structural optimization shortcut available to open-loop likelihood training. With open-loop Jacobian covariance propagation, shrinking Jacobian products directly reduces the volume penalty. Removing propagation avoids this penalty, but also makes predictive uncertainty dynamically blind. Filtered innovation likelihood keeps Jacobian propagation, yet scores it after assimilation has reset the belief state, so the long-horizon volume pressure is localized.

\begin{proposition}[Open-loop volume pressure and filtering localization]
\label{prop:volume_pressure_localization}
Fix time step $t$ and horizon $h\ge1$, and assume the observation matrix $\bm B$ and covariances $\bm\Gamma_{\theta_s}\succeq0$, $\bm\Sigma_t\succeq0$, $\bm\Lambda_{\theta_s}\succ0$ to be fixed along each local scaling path $\bm J_{\theta_c,k}(s)=e^s\bm J_{\theta_c,k}$. Along those paths, the observed open-loop covariance can be written as
\[
\bm\Omega^{\mathrm{open}}_{\theta,t+h}(s)
=
\bm\Lambda_{\theta_s}
+
\sum_{\ell=0}^{h} e^{2\ell s}\bm A_\ell,
\qquad
\bm A_\ell\succeq0,
\]
where $\bm A_\ell$ collects the observed-space covariance contributions that
have undergone $\ell$ Jacobian propagations. Hence
\[
\frac{\partial}{\partial s}
\frac12\log\det\bm\Omega^{\mathrm{open}}_{\theta,t+h}(s)
=
\sum_{\ell=1}^{h}
\ell e^{2\ell s}
\tr \left[
\bm\Omega^{\mathrm{open}}_{\theta,t+h}(s)^{-1}
\bm A_\ell
\right]
\ge 0,
\]
with strict inequality whenever some propagated contribution
$\bm A_\ell\neq0$ for $\ell\ge1$. Thus the open-loop log-determinant term
penalizes covariance contributions in proportion to how many Jacobian propagations
they have undergone. If the covariance recursion uses no Jacobian propagation,
this derivative is zero.

For a filtered one-step score, let
\[
\bm\Sigma_{t\mid t-1}(s)
=
e^{2s}
\bm J_{\theta_c,t}\bm\Sigma_{t-1\mid t-1}\bm J_{\theta_c,t}^{\top}
+
\bm\Gamma_{\theta_s},
\qquad
\bm\Omega_t(s)=
\bm B\bm\Sigma_{t\mid t-1}(s)\bm B^\top+\bm\Lambda_{\theta_s}.
\]
Then
\[
\frac{\partial}{\partial s}
\frac12\log\det\bm\Omega_t(s)
=
e^{2s}
\tr \left[
\bm\Omega_t(s)^{-1}
\bm B\bm J_{\theta_c,t}\bm\Sigma_{t-1\mid t-1}
\bm J_{\theta_c,t}^{\top}\bm B^\top
\right].
\]
Moreover, with Kalman gain
$\bm K_t=\bm\Sigma_{t\mid t-1}\bm B^\top\bm\Omega_t^{-1}$, the covariance update satisfies
\[
\bm\Sigma_{t\mid t}
=
\bm\Sigma_{t\mid t-1}
-
\bm\Sigma_{t\mid t-1}\bm B^\top
\bm\Omega_t^{-1}
 \bm B\bm\Sigma_{t\mid t-1}
\preceq
\bm\Sigma_{t\mid t-1}.
\]
Thus filtered innovation scoring keeps Jacobian-based covariance propagation, but
evaluates it locally from a post-assimilation belief instead of on a long
open-loop covariance product.
\end{proposition}

\begin{proof}
We write $\bm J_k:=\bm J_{\theta_c,k}$ and define the $\ell$-step product ending at $t+h$ by
\[
\bm\Phi^{(\ell)}_{t+h}:=\bm J_{t+h}\bm J_{t+h-1}\cdots \bm J_{t+h-\ell+1}, \qquad \bm\Phi^{(0)}_{t+h}:=\bm I.
\]
Along the stated path, $\bm\Phi^{(\ell)}_{t+h}(s)=e^{\ell s}\bm\Phi^{(\ell)}_{t+h}$. The initial covariance undergoes $h$ propagations, whereas the successive process-noise contribution undergoes $0,\ldots,h-1$ propagations. Hence

\[
\bm\Omega^{\mathrm{open}}_{\theta,t+h}(s)
=\bm\Lambda_{\theta_s}+\sum_{\ell=0}^{h}e^{2\ell s}\bm A_\ell,
\]
where
\[
\bm A_\ell = \begin{cases}
			\bm B\bm\Phi^{(h)}_{t+h}\bm\Sigma_t(\bm\Phi^{(h)}_{t+h})^\top\bm B^\top, & \text{for $\ell = h$}\\
            \bm B\bm\Phi^{(\ell)}_{t+h}\bm\Gamma_{\theta_s}(\bm\Phi^{(\ell)}_{t+h})^\top\bm B^\top, & \text{for $0\le\ell<h$}.
		 \end{cases}
\]
Since $\bm\Sigma_t\succeq0$ and $\bm\Gamma_{\theta_s}\succeq0$, each
$\bm A_\ell\succeq0$, and since $\bm\Lambda_{\theta_s}\succ0$, the observed covariance is positive definite with a positive definite inverse. Now we differentiate the log-determinant term. Using $\partial_s\log\det\bm\Omega=\tr(\bm\Omega^{-1}\partial_s\bm\Omega)$ gives
\[
\frac{\partial}{\partial s}\frac12\log\det\bm\Omega^{\mathrm{open}}_{\theta,t+h}(s)
=\sum_{\ell=1}^{h}\ell e^{2\ell s}\tr\!\left[\bm\Omega^{\mathrm{open}}_{\theta,t+h}(s)^{-1}\bm A_\ell\right].
\]
Every trace is nonnegative, and is positive whenever $\bm A_\ell\ne0$, because it equals the trace of the positive-semidefinite matrix $(\bm\Omega_{\theta,t+h}^\mathrm{open})^{-1/2}\bm A_\ell(\bm\Omega_{\theta,t+h}^\mathrm{open})^{-1/2}$:
\[
\tr \left[
(\bm\Omega_{\theta,t+h}^\mathrm{open})^{-1}\bm A_\ell
\right]
=
\tr\left[
(\bm\Omega_{\theta,t+h}^\mathrm{open})^{-1/2}\bm A_\ell(\bm\Omega_{\theta,t+h}^\mathrm{open})^{-1/2}
\right].
\]
With no Jacobian propagation, the covariance does not depend on the scaled Jacobians, so then this derivative is zero. This proves the open-loop part.

We next prove the filtered localization statement. The one-step filtered
predictive covariance along the scaling path is
\[
\bm\Sigma_{t\mid t-1}(s)
=
e^{2s}
\bm J_{\theta_c,t}\bm\Sigma_{t-1\mid t-1}
\bm J_{\theta_c,t}^{\top}
+
\bm\Gamma_{\theta_s},
\]
with corresponding one-step innovation covariance
\[
\bm\Omega_t(s)
=
\bm B\bm\Sigma_{t\mid t-1}(s)\bm B^\top+\bm\Lambda_{\theta_s}.
\]
For the one-step innovation score, differentiating at a fixed $\bm\Sigma_{t-1\mid t-1}$ as above then gives the stated derivative 
\[
\frac{\partial}{\partial s}
\frac12\log\det\bm\Omega_t(s)
=
e^{2s}
\tr \left[
\bm\Omega_t(s)^{-1}
\bm B\bm J_{\theta_c,t}\bm\Sigma_{t-1\mid t-1}
\bm J_{\theta_c,t}^{\top}\bm B^\top
\right].
\]
Finally, write the Kalman gain
\[
\bm K_t
=
\bm\Sigma_{t\mid t-1}\bm B^\top\bm\Omega_t^{-1},
\qquad
\bm\Omega_t
=
\bm B\bm\Sigma_{t\mid t-1}\bm B^\top+\bm\Lambda_{\theta_s}.
\]
The covariance update can be written as
\[
\bm\Sigma_{t\mid t}
=
(\bm I-\bm K_t\bm B)\bm\Sigma_{t\mid t-1}.
\]
Substituting the expression for $\bm K_t$ gives
\[
\bm\Sigma_{t\mid t}
=
\bm\Sigma_{t\mid t-1}
-
\bm\Sigma_{t\mid t-1}\bm B^\top
\bm\Omega_t^{-1}
\bm B\bm\Sigma_{t\mid t-1}.
\]
The subtracted term is positive semidefinite, since
\[
\bm\Sigma_{t\mid t-1}\bm B^\top
\bm\Omega_t^{-1}
\bm B\bm\Sigma_{t\mid t-1}
=
\left(
\bm\Sigma_{t\mid t-1}\bm B^\top\bm\Omega_t^{-1/2}
\right)
\left(
\bm\Sigma_{t\mid t-1}\bm B^\top\bm\Omega_t^{-1/2}
\right)^\top
\succeq0.
\]
Hence
\[
\bm\Sigma_{t\mid t}
\preceq
\bm\Sigma_{t\mid t-1}.
\]
Thus the filtered objective still scores Jacobian-propagated covariance, but
after each observation the covariance is contracted before the next prediction, whereas open-loop continues on the unassimilated product without contraction.
\end{proof}

\textbf{Figure~\ref{fig:tent_nll_decomposition} details.} Proposition~\ref{prop:volume_pressure_localization} isolates the log-volume term of the Gaussian NLL; the Mahalanobis term can oppose this pressure. Fig.~\ref{fig:tent_nll_decomposition} illustrates this tradeoff on the tent map 
\[
f_s(z)=s\min\{z,1-z\},\qquad z\in[0,1],
\]
with slope $s=1.8$, putting it into the chaotic regime \citep{alligood1996chaos}. We generate data by clipping each noisy state $z_{t+1}=f_s(z_t)+\xi_t$ such that $z_{t+1}\in[0,1]$ (to ensure boundedness), with $\xi_t\sim\mathcal N(0,0.04^2)$. Observations have i.i.d Gaussian noise of standard deviation $0.02$. The reference predictive histograms in Panel~(a) are obtained by Monte Carlo simulation for the open-loop case and a bootstrap particle filter with $4\cdot 10^4$ particles for the filtered case. To profile the NLL objectives over $s$, we use $5\cdot 10^4$ trajectories of $25$ transitions, holding both noise scales fixed. The excess profiles in Panel~(b) subtract each objective's own minimum, so their minima can be compared without equating the levels of different NLLs. Panel~(c) shows how the log-volume term shifts the open-loop optimum toward weaker expansion.

\subsection{Open-loop likelihood and volume pressure in exponential family models}
\label{app:exp_fam}

Here, we show that the structural optimization shortcut identified in Proposition \ref{prop:volume_pressure_localization} extends beyond 
Gaussian likelihoods. By applying a Laplace approximation to the posterior in models with Gaussian latent dynamics and arbitrary exponential family observations, we obtain a similar covariance volume penalty under open-loop training.

We now assume the observation distribution is given by a conditional exponential family \citep{murphy2023probabilistic}:
\begin{align*}
p(\bm x_t \mid \bm z_t)
=
h(\bm x_t)
\exp\left(
\bm\eta(\bm z_t)^\top \bm T(\bm x_t)
-
A(\bm\eta(\bm z_t))
\right),
\end{align*}
where $\bm T(\bm x_t) \in \mathbb{R}^N$ is the vector of sufficient statistics,
$h(\bm x_t)$ is the base measure, and $A(\bm\eta(\bm z_t))$ is the log-partition function.

We consider a Generalized Linear Model (GLM) formulation where the natural parameter
$\bm\eta(\bm z_t) \in \mathbb{R}^N$ is mapped from the latents
$\bm z_t \in \mathbb{R}^M$ via emission matrix
$\bm B \in \mathbb{R}^{N\times M}$ and bias
$\bm b \in \mathbb{R}^N$:
\begin{align*}
\bm\eta(\bm z_t)
=
\bm B\bm z_t+\bm b.
\end{align*}
The prior
$p(\bm z_t \mid \bm x_{1:t-1})$
is approximated as Gaussian
$\mathcal{N}(\bm\mu_{t\mid t-1},\bm\Sigma_{t\mid t-1})$
via linearization / the EKF's predict step, as before. We now have to find new expressions for the update step:
\begin{align*}
p(\bm z_t \mid \bm x_{1:t})
\propto
p(\bm x_t \mid \bm z_t)
p(\bm z_t \mid \bm x_{1:t-1}),
\end{align*}
and the innovation likelihood:
\begin{align*}
p(\bm x_t \mid \bm x_{1:t-1})
&=
\int
p(\bm x_t \mid \bm z_t)
p(\bm z_t \mid \bm x_{1:t-1})
\,d\bm z_t,
\\
&=
p(\bm x_t \mid \bm z_t)
 p(\bm z_t \mid \bm x_{1:t-1})\frac{1}{
p(\bm z_t \mid \bm x_{1:t})}.
\end{align*}
Thus concretely we need to find an approximation to
$p(\bm z_t \mid \bm x_{1:t})$.

\textbf{Approximating the posterior with Laplace.} The Laplace approximation gives a Gaussian approximation:
$
p(\bm z_t \mid \bm x_{1:t})
\approx
\mathcal{N}
(\bm\mu^{\mathrm{lp}}_t,
\bm\Sigma^{\mathrm{lp}}_t),
$
where the mean is given by the posterior mode:
$$
\bm\mu^{\mathrm{lp}}_t
=
\arg\max_{\bm z_t}
\log p(\bm z_t \mid \bm x_{1:t}),
$$
and the covariance is the inverse negative Hessian of the log posterior:
$$
\bm\Sigma^{\mathrm{lp}}_t
=
\left(
-\nabla_{\bm z_t}^{2}
\log p(\bm z_t \mid \bm x_{1:t})
\Big|_{\bm z_t=\bm\mu^{\mathrm{lp}}_t}
\right)^{-1}.
$$
The full log posterior is:
\begin{align*}
\log p(\bm z_t \mid \bm x_{1:t})
&=
\log h(\bm x_t)
+
(\bm B\bm z_t+\bm b)^\top
\bm T(\bm x_t)
-
A(\bm B\bm z_t+\bm b)
\\
&\quad
-\frac{1}{2}
\log\det(2\pi\bm\Sigma_{t\mid t-1})
-\frac{1}{2}
(\bm z_t-\bm\mu_{t\mid t-1})^\top
\bm\Sigma_{t\mid t-1}^{-1}
(\bm z_t-\bm\mu_{t\mid t-1})
\\
&\quad
-
\log p(\bm x_t \mid \bm x_{1:t-1}).
\end{align*}
The covariance approximation is:
\begin{align*}
\bm\Sigma^{\mathrm{lp}}_t
&=
\left(
-\nabla_{\bm z_t}^{2}
\log p(\bm z_t \mid \bm x_{1:t})
\Big|_{\bm z_t=\bm\mu^{\mathrm{lp}}_t}
\right)^{-1}\\
&=
\left(
-\nabla_{\bm z_t}^{2}
\log p(\bm x_t \mid \bm z_t)
\Big|_{\bm z_t=\bm\mu^{\mathrm{lp}}_t}
+
\bm\Sigma_{t\mid t-1}^{-1}
\right)^{-1}\\
&=
\left(
\bm B^{\top}\bm H\bm B
+
\bm\Sigma_{t\mid t-1}^{-1}
\right)^{-1},
\end{align*}
with
$
\bm H
=
\nabla_{\bm\eta}^{2}
A(\bm\eta(\bm\mu^{\mathrm{lp}}_t))
$, the curvature of the log partition function.

Note that the log posterior is globally concave for exponential-family observations with a Gaussian prior, and therefore the posterior has a unique mode. Starting from
$
\bm z_t^{(0)}
=
\bm\mu^{\mathrm{lp}}_{t\mid t-1},
$
the mode can be found iteratively:
\begin{align*}
\bm z_t^{(\mathrm{new})}
&=
\bm z_t^{(\mathrm{old})}
-
\left[
\nabla^2_{\bm z_t}
\log p(\bm z_t^{(\mathrm{old})} \mid \bm x_{1:t})
\right]^{-1}
\nabla_{\bm z_t}
\log p(\bm z_t^{(\mathrm{old})} \mid \bm x_{1:t}),
\\
\bm\mu^{\mathrm{lp}}_{t\mid t}
&=
\bm z_t^{(\mathrm{final})},
\end{align*}
where the gradient is:
$
\nabla_{\bm z_t}
\log p(\bm z_t \mid \bm x_{1:t})
=
\bm B^{\top}
\left(
\bm T(\bm x_t)
-
\nabla_{\bm\eta}
A(\bm\eta(\bm z_t))
\right)
-
\bm\Sigma_{t\mid t-1}^{-1}
(\bm z_t-\bm\mu_{t\mid t-1}).
$

\textbf{Innovation likelihood given Laplace approximation to the posterior.} We can now obtain an expression for the innovation log likelihood:
\begin{align*}
\log p(\bm x_t \mid \bm x_{1:t-1})
&=
\log p(\bm x_t \mid \bm z_t)
+
\log p(\bm z_t \mid \bm x_{1:t-1})
-
\log p(\bm z_t \mid \bm x_{1:t}).
\end{align*}
Using the Laplace approximation for the posterior:
\begin{align*}
\log p(\bm x_t \mid \bm x_{1:t-1})
&\approx
\log p(\bm x_t \mid \bm z_t)
-
\frac{1}{2}
\log\det(2\pi\bm\Sigma_{t\mid t-1})
\\
&\quad
-\frac{1}{2}
(\bm z_t-\bm\mu_{t\mid t-1})^\top
\bm\Sigma_{t\mid t-1}^{-1}
(\bm z_t-\bm\mu_{t\mid t-1})
\\
&\quad
-
\left(
-\frac{1}{2}
\log\det(2\pi\bm\Sigma_{t}^{\mathrm{lp}})
-
\frac{1}{2}
(\bm z_t-\bm\mu_{t}^{\mathrm{lp}})^\top
(\bm\Sigma_{t}^{\mathrm{lp}})^{-1}
(\bm z_t-\bm\mu_{t}^{\mathrm{lp}})
\right).
\end{align*}
Evaluating at the Laplace mode
$\bm z_t=\bm\mu_{t}^{\mathrm{lp}}$:
\begin{align*}
\log p(\bm x_t \mid \bm x_{1:t-1})
&\approx
\log p(\bm x_t \mid \bm\mu_{t}^{\mathrm{lp}})
\\
&\quad
-\frac{1}{2}
(\bm\mu_{t}^{\mathrm{lp}}
-
\bm\mu_{t\mid t-1})^\top
\bm\Sigma_{t\mid t-1}^{-1}
(\bm\mu_{t}^{\mathrm{lp}}
-
\bm\mu_{t\mid t-1})
\\
&\quad
-
\frac{1}{2}
\log
\frac{
\det\bm\Sigma_{t\mid t-1}
}{
\det\bm\Sigma_{t}^{\mathrm{lp}}
}.
\end{align*}
Plugging in the definition of the precision,
$
(\bm\Sigma_{t}^{\mathrm{lp}})^{-1}
=
\bm B^{\top}\bm H\bm B
+
\bm\Sigma_{t\mid t-1}^{-1},
$
we simplify the ratio of determinants using the Matrix Determinant Lemma:
$
\det(\bm A+\bm U\bm W\bm V^\top)
=
\det(\bm W^{-1}+\bm V^\top \bm A^{-1}\bm U)
\det(\bm W)
\det(\bm A)
$:
\begin{align*}
\frac{
\det\bm\Sigma_{t\mid t-1}
}{
\det\bm\Sigma_{t}^{\mathrm{lp}}
}
&=
\det(\bm H)
\det(
\bm H^{-1}
+
\bm B
\bm\Sigma_{t\mid t-1}
\bm B^{\top}
).
\end{align*}
Therefore, the final expression becomes:
\begin{align*}
\log p(\bm x_t \mid \bm x_{1:t-1})
&\approx 
\log p(\bm x_t \mid \bm\mu_{t}^{\mathrm{lp}})
\\
&\quad
-\frac{1}{2}
(\bm\mu_{t}^{\mathrm{lp}}
-
\bm\mu_{t\mid t-1})^\top
\bm\Sigma_{t\mid t-1}^{-1}
(\bm\mu_{t}^{\mathrm{lp}}
-
\bm\mu_{t\mid t-1})
\\
&\quad
-\frac{1}{2}
\log\det\bm H
\\
&\quad
-\frac{1}{2}
\log\det
(
\bm H^{-1}
+
\bm B
\bm\Sigma_{t\mid t-1}
\bm B^{\top}
).
\end{align*}

\textbf{Open loop volume pressure for exponential family.}
The open-loop innovation likelihood has the following expression:
\begin{align*}
\log p(\bm x_{t+h} \mid \bm x_{1:t})
&\approx
\log p(\bm x_{t+h} \mid \bm\mu^{\mathrm{lp}}_{t+h})
-\frac{1}{2}
(\bm\mu^{\mathrm{lp}}_{t+h}
-
\bm\mu_{t+h\mid t})^\top
\bm\Sigma_{t+h\mid t}^{-1}
(\bm\mu^{\mathrm{lp}}_{t+h}
-
\bm\mu_{t+h\mid t}) \\
&\quad -\frac{1}{2}
\log\det\bm H
-\frac{1}{2}
\log\det
(
\bm H^{-1}
+
\bm B
\bm\Sigma_{t+h\mid t}
\bm B^{\top}
).
\end{align*}
This indicates that we have the same volume pressure term as in the Gaussian case. Note that the two crucial assumptions made here are that the prior $p(\bm z_{t+h} \mid \bm x_{1:t})$ is approximated as Gaussian (by applying the EKF's predict step $h$ times), and that the local posterior $p(\bm z_{t+h} \mid \bm x_{1:t,t+h})$ is approximated as Gaussian (via the Laplace approximation). The first assumption is also made for the Gaussian open-loop case, and seems more likely to cause errors as nonlinear dynamics over $h$ steps without data corrections can distort the original Gaussian distribution at time $t$ (the true distribution at $t+h$ could be highly multimodal, see Fig.~\ref{fig:tent_nll_decomposition}a), whereas under this first assumption, the log-posterior is guaranteed to remain unimodal.

\textbf{Poisson observations.}
As a concrete example, we consider Poisson objectives as commonly used in neuroscience models \citep{Macke2011empirical,dowling2024exponential, pals2024inferring}.
Assuming
$
x_{t,i}
\sim
\mathrm{Poisson}(\lambda_{t,i})
$
with a log-link function
$
\eta_{t,i}
=
\log\lambda_{t,i}
=
[\bm B\bm z_t+\bm b]_i,
$
the Hessian is a diagonal matrix of the predicted rates evaluated at the mode:
$
\bm H
=
\mathrm{diag}
\left(
\exp(\bm B\bm\mu_{t}^{\mathrm{lp}}+\bm b)
\right)
\equiv
\bm\Lambda_t.
$
The determinant contribution then becomes:
\begin{align*}
-\frac{1}{2}
\log\det\bm\Lambda_t
-\frac{1}{2}
\log\det
(
\bm\Lambda_t^{-1}
+
\bm B
\bm\Sigma_{t\mid t-1}
\bm B^{\top}
).
\end{align*}
Here,
$\bm\Lambda_t^{-1}$
can be interpreted as an effective, state-dependent observation noise covariance.

\subsection{Noisy-state attenuation and innovation identifiability}
\label{app:attenuation_diagnostic}

Here we discuss a scalar diagnostic of classical noisy-state attenuation \citep{fuller1987measurement,staudenmayer2005measurement} and its relationship to the filtered innovation likelihood. Treating noisy observations directly as states yields an attenuated observation-to-observation regression coefficient. In contrast, the exact Kalman innovation likelihood identifies the latent multiplier in the linear-Gaussian setting. We emphasize that the result here is local and does not imply global consistency of EKF training for nonlinear chaotic systems. 

Let the true data-generating process be
\[
z_1\sim\mathcal N(0,p_\star),\qquad z_{t+1}=a_\star z_t+\xi_t,\qquad x_t=z_t+\varepsilon_t,
\]
where
\[
\xi_t \overset{\mathrm{i.i.d.}}{\sim}\mathcal N(0,q_\star), \qquad \varepsilon_t \overset{\mathrm{i.i.d.}}{\sim}\mathcal N(0,r_\star),
\]
with $z_1$, $\{\xi_t\}_{t\ge1}$, and $\{\varepsilon_t\}_{t\ge1}$ mutually independent, $p_\star,q_\star,r_\star>0$, and $a_\star\neq0$. We assume zero mean only for notational convenience, the same calculations hold with affine offsets after centering. We write
\[
P_t^\star:=\operatorname{Var}(z_t).
\]

\textbf{Noisy-state predictive attenuation.}
Consider
\[
q_{a,v}(x_{t+1}\mid x_t) = \mathcal N(x_{t+1};a x_t,v),\qquad v>0.
\]
The corresponding expected NLL at a fixed time $t$ is
\[
\mathcal L_{\mathrm{pred}}(a,v)
=
\frac12\log(2\pi v)
+
\frac{1}{2v}
\E_\star \left[(x_{t+1}-a x_t)^2\right],
\]
where $\E_\star$ denotes expectation under the true data-generating process. For fixed $a$,
the optimal variance is
\[
v^\dagger(a)
=
\E_\star \left[(x_{t+1}-a x_t)^2\right].
\]
Therefore minimizing the profiled (i.e., variance-optimized) NLL over $a$ is equivalent to
minimizing the mean squared prediction error. Hence the optimal multiplier is the classical population least-squares coefficient under errors-in-variables \citep{fuller1987measurement,staudenmayer2005measurement}:
\[
a_{\mathrm{pred}}^\dagger
=
\frac{\E_\star[x_{t+1}x_t]}{\E_\star[x_t^2]}.
\]
Using
\[
x_t=z_t+\varepsilon_t,
\qquad
x_{t+1}=a_\star z_t+\xi_t+\varepsilon_{t+1},
\]
and independence of $z_t,\xi_t,\varepsilon_t,\varepsilon_{t+1}$, we obtain
\[
\E_\star[x_{t+1}x_t]
=
a_\star\operatorname{Var}(z_t)
=
a_\star P_t^\star,
\]
and
\[
\E_\star[x_t^2]
=
\operatorname{Var}(z_t)+\operatorname{Var}(\varepsilon_t)
=
P_t^\star+r_\star.
\]
Thus
\[
a_{\mathrm{pred}}^\dagger
=
a_\star
\frac{P_t^\star}{P_t^\star+r_\star}.
\]
Since $P_t^\star>0$ and $r_\star>0$, the attenuation factor lies strictly in
$(0,1)$. Consequently,
\[
|a_{\mathrm{pred}}^\dagger|<|a_\star|.
\]
In particular, an expanding latent direction is estimated as stable if
\[
|a_\star|
\frac{P_t^\star}{P_t^\star+r_\star}
<1.
\]
If a single coefficient and shared variance are fitted over times $t=1,\ldots,T-1$, the corresponding optimum is
\[
a_{\mathrm{pred},T}^\dagger
=a_\star\frac{\sum_{t=1}^{T-1}P_t^\star}
{\sum_{t=1}^{T-1}(P_t^\star+r_\star)},
\]
again a strict attenuation whenever $r_\star>0$.

Appx.~Fig.~\ref{fig:logistic_attenuation_bifurcation} shows the same attenuation on the logistic map $z_{t+1} = r z_t(1-z_t)$ \citep{alligood1996chaos}: treating noisy observations as states shifts the population optimum $r^\star(\sigma)$ through different dynamical regimes, crossing period-doubling and chaos onsets, so the fitted map can be non-chaotic even when the latent map is chaotic.

\begin{figure}[t]
  \centering
  \includegraphics[width=\linewidth]{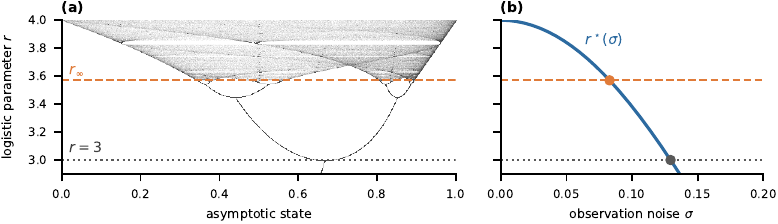}
  \caption{\textbf{Errors-in-variables attenuation on the logistic map.}
  \textbf{(a)} Bifurcation diagram of the noise-free logistic map as a function of the parameter $r$. \textbf{(b)} Population noisy-state optimum $r^\star(\sigma)$ vs. observation-noise standard deviation. Increasing $\sigma$ attenuates $r^\star$ through the onset of chaos ($r_\infty \approx 3.6$), the period-2 bifurcation ($r=3$), and into the stable fixed-point regime ($r<3$).}
  \label{fig:logistic_attenuation_bifurcation}
\end{figure}

\textbf{Innovation identifiability in the correctly specified family.}
Now consider the same linear-Gaussian process
\[
z_1\sim\mathcal N(0,p),\qquad
z_{t+1}=az_t+\xi_t,\qquad
x_t=z_t+\varepsilon_t,
\]
with
\[
\xi_t\overset{\mathrm{i.i.d.}}{\sim}\mathcal N(0,q),
\qquad
\varepsilon_t\overset{\mathrm{i.i.d.}}{\sim}\mathcal N(0,r),
\]
independent of $z_1$ and of each other with $p,q,r>0$. Let $\theta=(a,q,r,p)$, and let $\mathbb P_\theta^{1:T}$ denote the joint distribution of $x_{1:T}$, with density $p_\theta(x_{1:T})$. Since the model is linear-Gaussian, the Kalman innovation factorization is the exact likelihood:
\[
p_\theta(x_{1:T})
=
p_\theta(x_1)
\prod_{t=2}^{T}
p_\theta(x_t\mid x_{1:t-1})
=
p_\theta(x_1)
\prod_{t=2}^{T}
\mathcal N(x_t;\hat x^\theta_{t\mid t-1},\Omega^\theta_t),
\]
where $\hat x^\theta_{t\mid t-1}$ are the predicted observation means. The expected innovation NLL is therefore
\[
\mathcal L_{\mathrm{KF}}(\theta)
:=
\E_{\theta_\star} \left[-\log p_\theta(x_{1:T})\right],
\]
Equivalently,
\[
\mathcal L_{\mathrm{KF}}(\theta)
=
H(\mathbb P_{\theta_\star}^{1:T})
+
\KL \left(
\mathbb P_{\theta_\star}^{1:T}
\,\Vert\,
\mathbb P_{\theta}^{1:T}
\right).
\]
Since the KL term is nonnegative, the minimum value is attained at $\theta_\star$. Therefore every population minimizer must have the same observation law
\[
\mathbb P_{\theta}^{1:T}=\mathbb P_{\theta_\star}^{1:T},
\]
so it remains only to show identifiability.

Because the distributions are zero-mean Gaussian, equality of
 $\mathbb P_{\theta}^{1:T} $ and  $\mathbb P_{\theta_\star}^{1:T} $ follows from the equality of their covariance matrices. For $T\ge3$, the covariance matrix entries are:
\[
\sigma_{11}:=\operatorname{Cov}_\theta(x_1,x_1)=p+r,
\]
\[
\sigma_{21}:=\operatorname{Cov}_\theta(x_2,x_1)=ap,
\]
\[
\sigma_{31}:=\operatorname{Cov}_\theta(x_3,x_1)=a^2p,
\]
and
\[
\sigma_{22}:=\operatorname{Cov}_\theta(x_2,x_2)=a^2p+q+r.
\]
Since $a_\star\neq0$ and $p_\star>0$,
\[
\sigma_{21}^\star=a_\star p_\star\neq0.
\]
Hence equality of the observation laws implies
\[
a=\frac{\sigma_{31}}{\sigma_{21}}=\frac{\sigma_{31}^\star}{\sigma_{21}^\star}=a_\star.
\]
Then
\[
p=\frac{\sigma_{21}}{a}=\frac{\sigma_{21}^\star}{a_\star}=p_\star,
\]
and
\[
r=\sigma_{11}-p=\sigma_{11}^\star-p_\star=r_\star.
\]
Finally,
\[
q=\sigma_{22}-a^2p-r=\sigma_{22}^\star-a_\star^2p_\star-r_\star=q_\star.
\]
Therefore equality of the observation distributions implies
\[
(a,q,r,p)=(a_\star,q_\star,r_\star,p_\star), 
\]
and the expected Kalman innovation NLL has the unique population-risk minimizer
\[
(a^\dagger_{\mathrm{KF}},q^\dagger_{\mathrm{KF}},
r^\dagger_{\mathrm{KF}},p^\dagger_{\mathrm{KF}})
=
(a_\star,q_\star,r_\star,p_\star).
\]

\subsection{When covariance contraction increases the expected NLL}
\label{app:anti_contraction}

The log-determinant term of the NLL favors smaller predictive covariances, but the Mahalanobis (data-fit) term can oppose this tendency. The following local calculation shows that, when the covariance is already too small relative to the prediction errors, further contraction increases the expected score. 

To isolate a single contraction direction, fix a reference pair $\theta=(\theta_c,\theta_s)$. Consider Jacobians of the form
\[
\bm J_{\theta_c,t}(s)=s\,\bar{\bm J}_{\theta_c,t}
\]
along a chosen covariance direction, and write
\[
\bm A_{\theta,t}
:=
\bm B\bar{\bm J}_{\theta_c,t}
\bm \Sigma_{t-1\mid t-1}
\bar{\bm J}_{\theta_c,t}^{\top}\bm B^\top,
\qquad
\bm C_{\theta_s}
:=
\bm B\bm\Gamma_{\theta_s}\bm B^\top+\bm\Lambda_{\theta_s}.
\]
Assume $\bm C_{\theta_s}\succ0$. Define
\[
\bm\Omega_{\theta,t}(s):=s^2\bm A_{\theta,t}+\bm C_{\theta_s},
\qquad
\bm\Omega_t^\star:=\operatorname{Cov}(\bm x_t\mid \bm x_{1:t-1}),
\]
and
\[
\mathcal L(s)
:=
\frac12\sum_{t=1}^T
\left[
\log\det \bm\Omega_{\theta,t}(s)
+
\tr \left(
\bm\Omega_{\theta,t}(s)^{-1}\bm\Omega_t^\star
\right)
\right].
\]

\begin{Lemma}[Anti-contraction along a covariance-rescaling path]	
\label{lem:anti_contraction}
Consider a local covariance-scaling path at fixed filtered covariances $\{\bm\Sigma_{t-1\mid t-1}\}_{t=1}^T$ and fixed innovation means. We characterize the partial derivative of the centered-covariance contribution to the expected innovation NLL with respect to a local Jacobian-rescaling direction (ignoring the recursion through all earlier EKF covariances).
If there exists $s_0>0$ such that
\[
\bm\Omega_{\theta,t}(s_0)\preceq \bm\Omega_t^\star
\qquad \forall t,
\]
with
\[
\bm\Omega_{\theta,t_\star}(s_0)\prec \bm\Omega_{t_\star}^\star
\]
for at least one $t_\star$ satisfying $\bm A_{\theta,t_\star}\neq0$, then
\[
\frac{\partial}{\partial s}\mathcal L(s)<0
\qquad
\text{for all }0<s\le s_0.
\]
Thus no local minimizer of $\mathcal L$ lies in $(0,s_0]$ of this contraction path.
\end{Lemma}

\begin{proof}[Proof of Lemma~\ref{lem:anti_contraction}]
Since $\bm\Sigma_{t-1\mid t-1}\succeq 0$, we have $\bm A_{\theta,t}\succeq 0$. The assumption $\bm C_{\theta_s}\succ0$ implies
$\bm\Omega_{\theta,t}(s)\succ0$ for all $s\ge0$, so the inverse square roots below are well-defined.

Write $\theta=(\theta_c,\theta_s)$ and $\bm \Omega_t^\star:=\operatorname{Cov}(\bm x_t\mid \bm x_{1:t-1})$. For each $t$,
\[
\bm \Omega_{\theta,t}(s)=s^2\bm A_{\theta,t}+\bm C_{\theta_s},
\qquad
\frac{\partial}{\partial s}\bm \Omega_{\theta,t}(s)=2s\bm A_{\theta,t}.
\]
Differentiating $\mathcal L$ gives
\[
\frac{\partial}{\partial s}\mathcal L(s)
=
\frac12\sum_{t=1}^T
\left(
\tr \big(\bm \Omega_{\theta,t}(s)^{-1}\bm \Omega_{\theta,t}'(s)\big)
-
\tr \big(\bm \Omega_{\theta,t}(s)^{-1}\bm \Omega_{\theta,t}'(s)\bm \Omega_{\theta,t}(s)^{-1}\bm \Omega_t^\star\big)
\right).
\]
Substituting $\bm \Omega_{\theta,t}'(s)=2s\bm A_{\theta,t}$ yields
\[
\frac{\partial}{\partial s}\mathcal L(s)
=
s\sum_{t=1}^T
\left(
\tr \big(\bm \Omega_{\theta,t}(s)^{-1}\bm A_{\theta,t}\big)
-
\tr \big(\bm \Omega_{\theta,t}(s)^{-1}\bm A_{\theta,t}\bm \Omega_{\theta,t}(s)^{-1}\bm \Omega_t^\star\big)
\right).
\]

Define
\[
\bm M_t(s):=\bm \Omega_{\theta,t}(s)^{-1/2}\bm A_{\theta,t}\bm \Omega_{\theta,t}(s)^{-1/2}\succeq 0,
\qquad
\bm Z_t(s):=\bm \Omega_{\theta,t}(s)^{-1/2}\bm \Omega_t^\star\bm \Omega_{\theta,t}(s)^{-1/2}.
\]
Using cyclicity of the trace,
\[
\tr \big(\bm \Omega_{\theta,t}^{-1}\bm A_{\theta,t}\big)=\tr \big(\bm M_t\big),
\qquad
\tr \big(\bm \Omega_{\theta,t}^{-1}\bm A_{\theta,t}\bm \Omega_{\theta,t}^{-1}\bm \Omega_t^\star\big)=\tr \big(\bm M_t\bm Z_t\big),
\]
so
\[
\frac{\partial}{\partial s}\mathcal L(s)
=
s\sum_{t=1}^T \tr \big(\bm M_t(s)(\bm I-\bm Z_t(s))\big).
\]

Assume now that $0<s\le s_0$. Since $\bm \Omega_{\theta,t}(s)=s^2\bm A_{\theta,t}+\bm C_{\theta_s}$ is monotone increasing in $s$,
\[
\bm \Omega_{\theta,t}(s)\preceq \bm \Omega_{\theta,t}(s_0)\preceq \bm \Omega_t^\star
\qquad \forall t.
\]
Therefore
\[
\bm Z_t(s)=\bm \Omega_{\theta,t}(s)^{-1/2}\bm \Omega_t^\star\bm \Omega_{\theta,t}(s)^{-1/2}\succeq \bm I,
\]
so $\bm I-\bm Z_t(s)\preceq 0$.
Because $\bm M_t(s)\succeq 0$, each summand is non-positive. Indeed, writing $\bm I-\bm Z_t=-\bm N_t$ with $\bm N_t\succeq 0$,
\[
\tr \big(\bm M_t(\bm I-\bm Z_t)\big)=-\tr \big(\bm M_t\bm N_t\big)\le 0,
\]
since $\tr(\bm A\bm C)\ge 0$ for positive semidefinite $\bm A,\bm C$.

For index $t_\star$, the assumption $\bm \Omega_{\theta,t_\star}(s_0)\prec \bm \Omega_{t_\star}^\star$ implies
\[
\bm \Omega_{\theta,t_\star}(s)\preceq \bm \Omega_{\theta,t_\star}(s_0)\prec \bm \Omega_{t_\star}^\star,
\]
hence $\bm Z_{t_\star}(s)\succ \bm I$.
Since $\bm A_{\theta,t_\star}\neq 0$, we also have $\bm M_{t_\star}(s)\neq 0$.
Therefore
\[
\tr \big(\bm M_{t_\star}(s)(\bm I-\bm Z_{t_\star}(s))\big)<0.
\]
Summing over $t$ yields
\[
\frac{\partial}{\partial s}\mathcal L(s)<0
\qquad \text{for all } 0<s\le s_0.
\]
Hence no local minimizer can lie in $(0,s_0]$.
\end{proof}

For completeness, the conclusion also holds for biased predictions.
With conditioning histories and predictions fixed, let $\bm b_t:=\mathbb E[\bm x_t\mid\bm x_{1:t-1}]-\bm B\bm z_{t\mid t-1}$.
The full expected Gaussian innovation NLL is
\[
\mathcal L_{\mathrm{full}}(s)
=\mathcal L(s)
+\frac12\sum_{t=1}^T
\bm b_t^\top\bm\Omega_{\theta,t}(s)^{-1}\bm b_t
+\mathrm{const}.
\]
Increasing $s$ increases or leaves unchanged the predictive variance in every direction, so the same fixed prediction bias receives a smaller or unchanged
Mahalanobis penalty $\bm b_t^\top\bm\Omega_{\theta,t}(s)^{-1}\bm b_t$. The additional bias term is therefore
non-increasing in $s$. Thus $\mathcal L_{\mathrm{full}}'(s)\le\mathcal L'(s)<0$
for $0<s\le s_0$, i.e., further covariance contraction increases
the expected NLL even when the fixed predictions are biased.
\subsection{Kalman updates as adaptive anisotropic forcing}
\label{app:kaffee_gtf_appendix}

Isolating the optimization mechanism by which filtering alters chaotic Jacobian product chains connects the KAFFEE objective to work on exploding BPTT gradients and generalized teacher forcing (GTF) \citep{mikhaeil2022difficulty,hess2023generalized}. Recent work in this direction extended GTF to project the network orthogonally toward the zero-error manifold \citep{sagtekin2025error}, and interpreted identity teacher forcing as an intervention-based generalized-Bayes objective \citep{herz2026teacher}. While GTF and error forcing modify the state recursion using a
forcing-strength hyperparameter, KAFFEE instead determines the correction from the predicted covariance and observation-noise
model through a Kalman update, without introducing a separate forcing-strength
hyperparameter.

Consider GTF, where model predictions are combined with a data-inferred state using a scalar forcing strength. The Kalman update has a related form:
\[
\bm z_{t\mid t}
= (\bm I-\bm K_t\bm B)F_{\theta_c}(\bm z_{t-1\mid t-1})
+\bm K_t\bm x_t.
\]
Here, the correction is determined by the predicted covariance and the observation-noise model. It can act differently in different latent directions, and incoming data can correct unobserved coordinates through their covariance with observed coordinates. This parallels the perspective in \citet{pals2024inferring}, where uncertainty-adaptive interpolation between model-predicted and data-inferred states is related to GTF. For the piecewise-affine transition models used in our experiments
(AL-RNN and shPLRNN), this connection has a particularly simple local expression.

\begin{proposition}[Filtering factor for piecewise-affine transitions]
\label{prop:kaffee_gtf_factorization}
Consider the EKF recursion (Appx.~\ref{app:ekf_equations}) with a piecewise-affine transition map,
linear observations, and state-independent noise covariances with
$\bm\Lambda\succ0$. Fix the
parameters, observations, and starting covariance $\bm\Sigma_{r\mid r}$ at step $r$.
Suppose the filtered trajectory remains away from the ReLU-activation boundaries over the interval $r, \ldots ,t$, such that sufficiently small perturbations of $\bm z_{r\mid r}$ preserve its
sequence of affine regions $\{\bm c_k\}_{r:t}$ (Appx.~\ref{app:alrnn_implementation_details}). Then
\[
\left.
\frac{\partial\bm z_{t\mid t}}{\partial\bm z_{r\mid r}}
\right|_{\bm\Sigma_{r\mid r}} = \bm T_t\bm T_{t-1}\cdots\bm T_{r+1},
\qquad
\bm T_k=(\bm I-\bm K_k\bm B)\bm J_k,
\]
where $\bm J_k=\nabla_{\bm z} F_{\theta_c}(\bm z_{k-1\mid k-1})$.
\end{proposition}

\begin{proof}
Within a fixed sequence of affine regions, each $\bm J_k$ is locally
constant with respect to the state. Starting from the same covariance,
the prediction
\[
\bm\Sigma_{k\mid k-1}
=\bm J_k\bm\Sigma_{k-1\mid k-1}\bm J_k^\top+\bm\Gamma
\]
and the covariance update therefore give the same covariance sequence
under the assumed sufficiently small perturbations of the starting mean.
Consequently, the Kalman gains are also locally unchanged. Differentiating
the update of the mean then gives
\[
\left.
\frac{\partial\bm z_{k\mid k}}{\partial\bm z_{r\mid r}}
\right|_{\bm\Sigma_{r\mid r}}
=
(\bm I-\bm K_k\bm B)\bm J_k
\left.
\frac{\partial\bm z_{k-1\mid k-1}}{\partial\bm z_{r\mid r}}
\right|_{\bm\Sigma_{r\mid r}}.
\]

Applying this identity iteratively then proves the stated ordered product.
\end{proof}

For $\bm B=\bm I$ and a fixed scalar gain $\bm K_t=\alpha\bm I$,
the update then becomes
$(1-\alpha)F_{\theta_c}(\bm z_{t-1\mid t-1})+\alpha\bm x_t$,
with derivative $(1-\alpha)\bm J_t$, recovering the GTF damping factor of \citet{hess2023generalized}. KAFFEE instead obtains the gain from the propagated covariance and learned noise model,
without introducing a chosen separate scalar forcing strength. For example, in the scalar case (and with $B=1$), we have $K_t=\Sigma_{t\mid t-1}/
(\Sigma_{t\mid t-1}+\Lambda)$; so at fixed observation noise, a less certain
prediction receives a larger observational correction. The general matrix gain extends this to anisotropic uncertainty. Although it is locally constant within the
fixed activation sequence considered in the proof, it can change across
time and across sequences through the covariance recursion and the
regions visited by the filtered trajectory. Thus, with the ``full'' gain, the factor $\bm I-\bm K_t\bm B$ can reduce amplification of state perturbations, but it need not contract every direction. Chaotic dynamics exhibit direction-dependent perturbation
growth \citep{alligood1996chaos,trevisan2011kalman}, which motivates adaptive, anisotropic damping rather than a uniform scalar factor. Appx.~\ref{app:l96_mechanism_ablations} also  examines the empirical benefit of this approach on Lorenz-96.

The combined matrix $\bm T_t=(\bm I-\bm K_t\bm B)\bm J_t$ can have a largest singular value below one even when $\bm J_t$ is expansive. If $\|\bm T_k\|_2\leq\rho<1$ throughout the interval, the proposition bounds the mean-to-mean sensitivity by $\rho^{t-r}$. However, while our proposition explains a connection to forcing methods, it does not guarantee bounded training gradients or preservation of autonomous dynamics. Moreover, the derivative here is with respect to the starting state at fixed parameters, while parameter changes still affect Jacobians, covariances, and the Kalman gains, and KAFFEE differentiates through these paths.

\paragraph{Beyond piecewise-affine transitions.}
Keeping the same observation and noise assumptions, for a sufficiently smooth transition model (for example, twice continuously
differentiable), the one-step mean derivative at
fixed incoming covariance additionally contains gain-variation terms:
\[
\left.
\frac{\partial\bm z_{t\mid t}}{\partial\bm z_{t-1\mid t-1}}
\right|_{\bm\Sigma_{t-1\mid t-1}}
=(\bm I-\bm K_t\bm B)\bm J_t
+\sum_{j=1}^{N}
(\bm x_t-\bm B\bm z_{t\mid t-1})_j
\frac{\partial\bm k_{t,j}}{\partial\bm z_{t-1\mid t-1}},
\]
where $\bm k_{t,j}$ denotes column $j$ of $\bm K_t$, and the derivatives
on the right also hold the incoming covariance fixed. This follows from directly differentiating the mean update. Over multiple steps, the recurrent state comprises both the mean and covariance: an earlier mean perturbation may change a later covariance and hence a later mean, and the complete derivative includes these paths. The piecewise-affine case above avoids this complication because the covariance sequence does not change locally with the starting mean.

The ablations in Appx.~\ref{app:l96_mechanism_ablations}
test the contribution of differentiation through assimilation. On Lorenz-96, stopping gradients through the Kalman update
worsens autonomous reconstruction despite identical forward updates at fixed parameters. Selected fixed scalar gains achieve comparable NLL but yield worse autonomous reconstruction than KAFFEE. These results support the practical value of learning through the Kalman update.

\subsection{Particle-based inference and variational objectives}
\label{app:empirical_density}

Our main-text volume-pressure argument is targeted at locally
linearized Gaussian predictive densities. Sample/Particle-based inference
and variational objectives
\citep{naesseth2018variational,le2018autoencoding,maddison2017filtering,pals2024inferring}
can represent non-Gaussian state uncertainty without projecting the full distribution onto a single Gaussian. However, in high-dimensional chaotic settings, approximation error and computational cost incurred by particle degeneracy can still limit their applicability. Also, while this relaxes a specific approximation constraint, it does not by itself ensure that autonomous dynamics are preserved.  Whether these alternatives mitigate the DPC gap at a practical computational cost thus remains
to be established.

\subsection{Computational cost and scaling}
\label{app:kaffee_scaling_appendix}

At a high level, KAFFEE inherits the scaling profile of EKF-style Gaussian filtering: it propagates one mean/covariance pair rather than a particle cloud or a large forecast ensemble. In our (dense) implementation, the 
learned shell covariances are diagonal, that is, $\bm \Gamma$ and $\bm \Lambda$ are parameterized by only $O(M+N)$ log-standard-deviation parameters rather than full $M\times M$ and $N\times N$ matrices. However, after those diagonal shells initialize the filter, the predictive and filtered state covariances produced by Jacobian-based propagation are transported as full dense matrices in latent space. Consequently, the dominant forward cost is the dense EKF recursion, not the shell parameterization itself: with latent dimension $M$ and observation dimension $N$, the forward step is dominated by the Jacobian propagation and Kalman update, with cost on the order of $O(M^3 + M^2N + N^3)$ per step when full latent covariance is retained. Note that for diagonal observation covariances, we can completely avoid the $O(N^3)$ inversion of the innovation covariance, via the Woodbury inversion lemma (at an additional $O(M^3)$ cost), which is practical for settings with high-D observations \citep{rasmussen2006gaussian}. For a batched training window, the cost scales linearly in the batch size and loss horizon.

For smooth nonlinear transitions, the gain-variation terms discussed above matter mainly for the exact backward path rather than the forward EKF pass. In our implementation the differentiable EKF is optimized end-to-end, and automatic differentiation captures the dependence of these terms by differentiating through the Kalman gain and the Jacobian-based covariance propagation. The exact backward pass inherits second-order sensitivities of the transition map associated with the gain-variation, thus effectively requiring Hessian-vector product information (even when no dense Hessian tensor is formed explicitly). However, for the piecewise-linear (or piecewise-affine) latent transition models we use in our experiments (AL-RNN, shPLRNN), those second-order state derivatives vanish everywhere in state space, except on the boundaries of the affine subregions (whose configuration is expressed by a binary switching code depending on which ReLU units are active, see Appx.~\ref{app:alrnn_implementation_details} and Proposition~\ref{prop:kaffee_gtf_factorization}). Thus, the main cost in that case is propagating and differentiating dense latent covariances.

Dense KAFFEE remains attractive in moderately high dimensions, but much larger latent states would likely require structured approximations and approximate Jacobian/backward passes, such as diagonal-, block-, or low-rank covariance \citep{chang2022diagonal,chang2023lowrank, schmidt2023rankreduced}. Chaotic DA further motivates these reductions: in ``perfect surrogate model'' settings, Kalman error covariances concentrate on the unstable-neutral subspace \citep{trevisan2011kalman,carrassi2022data}. With additive model/process noise however, covariance in stable directions can be filled by the noise and remain bounded, which favors retaining extra directions rather than a purely unstable-subspace truncation \citep{grudzien2018chaotic}. We instantiate one such approximation for KAFFEE in Appx.~\ref{app:rrkaffee}.

\section{KAFFEE algorithm and filtering equations}
\label{app:kaffee_algorithm}

\subsection{EKF Predict-update equations}
\label{app:ekf_equations}

For completeness, we provide the full recursive equations for the EKF utilized in the KAFFEE architecture. Let $F_{\theta_c}$ denote the autonomous transition map parameterized by the deterministic core, and let $\bm{\Gamma}_{\theta_s}$ and $\bm{\Lambda}_{\theta_s}$ denote the learned process and observation noise covariances parameterizing the stochastic shell. We assume a linear observation model $\bm{B}$.

Given a filtered state belief $\mathcal{N}(\bm{z}_{t-1\mid t-1}, \bm{\Sigma}_{t-1\mid t-1})$ at time $t-1$, the forward pass of KAFFEE performs the EKF's predict-update cycle at time $t$ \citep{sarkka2023bayesian}:

\textbf{1. Predict step (autonomous expansion).}
We first obtain the predictive density:
\begin{align*}
    p(\bm z _t\mid \bm x_{1:t-1})=\int p(\bm z_t \mid \bm z_{t-1})p(\bm z _{t-1}\mid \bm x_{1:t-1}) d\bm z_{t-1}.
\end{align*}
We apply the dynamics $p(\bm z_t \mid \bm z_{t-1})$ and marginalize the previous latent $\bm z_{t-1}$. While the integral is intractable for nonlinear dynamics, the EKF proceeds by linearizing around $\bm{z}_{t-1 \mid t-1}$. This results in a Gaussian approximation to $p(\bm z _t\mid \bm x_{1:t-1})$, where
the deterministic core projects the state mean forward, while the local Jacobian $\bm{J}_{\theta_c, t} = \nabla_{\bm{z}} F_{\theta_c}(\bm{z}_{t-1 \mid t-1})$ propagates the covariance:
\begin{align}
\bm{z}_{t\mid t-1} &= F_{\theta_c}(\bm{z}_{t-1\mid t-1}) \label{eq:kf_predict_mean}, \\
\bm{\Sigma}_{t\mid t-1} &= \bm{J}_{\theta_c, t} \, \bm{\Sigma}_{t-1\mid t-1} \, \bm{J}_{\theta_c, t}^\top + \bm{\Gamma}_{\theta_s} \label{eq:kf_predict_cov}.
\end{align}

\textbf{2. Update step (data assimilation).}
Next, we correct the state and covariance using the observations \citep{reich2015probabilistic,carrassi2018data}. This step (typically) contracts the predictive variance (Proposition~\ref{prop:volume_pressure_localization}).

Using $p(\bm x_t\mid\bm z_t, \bm x_{1:t-1})=p(\bm x_t\mid \bm z_t)$ and $p(\bm z_t\mid \bm x_{1:t-1})$ obtained in the last step, we have the following expression for the joint predictive distribution $p(\bm x_t,\bm z_t \mid \bm x_{1:t-1})$:
\begin{align*}
\begin{bmatrix}
\bm z_t \\
\bm x_t
\end{bmatrix}
\sim \mathcal{N}
\left(
\begin{bmatrix}
\bm z_{t|t-1} \\
\bm B\bm z_{t|t-1}
\end{bmatrix},
\begin{bmatrix}
\bm \Sigma_{t|t-1} & \bm \Sigma_{t|t-1} \bm B^\top \\
\bm B \bm \Sigma_{t|t-1} & \bm B \bm \Sigma_{t|t-1} \bm B^\top + \bm{\Lambda}_{\theta_s}
\end{bmatrix}
\right),
\end{align*}
where we can read off the innovation covariance
\begin{align}
\bm{\Omega}_t &= \bm{B} \bm{\Sigma}_{t\mid t-1} \bm{B}^\top + \bm{\Lambda}_{\theta_s}. \label{eq:kf_innovation_cov} 
\end{align}
Given this joint distribution, we can apply standard Gaussian conditioning identities to obtain the filtered distribution $p(\bm z_t \mid \bm x_{1:t})$, with mean and covariance given by:
\begin{align}
\bm{z}_{t\mid t} &= \bm{z}_{t\mid t-1} + \bm{K}_t (\bm{x}_t - \bm{B} \bm{z}_{t\mid t-1}), \label{eq:kf_update_mean} \\
\bm{\Sigma}_{t\mid t} &= (\bm{I} - \bm{K}_t \bm{B}) \bm{\Sigma}_{t\mid t-1}. \label{eq:kf_update_cov}
\end{align}
Where $\bm{K}_t$, is the Kalman Gain, which determines the optimal blending weight between the prediction and the incoming noisy observation $\bm{x}_t$:
\begin{align}
\bm{K}_t &= \bm{\Sigma}_{t\mid t-1} \bm{B}^\top \bm{\Omega}_t^{-1} .\label{eq:kf_gain}
\end{align}

Equivalently, for numerical stability, one may use the symmetric Joseph form for the covariance:
\begin{equation}
\label{eq:joseph}
\bm{\Sigma}_{t\mid t}=(\bm I-\bm K_t\bm B)\bm\Sigma_{t\mid t-1}(\bm I-\bm K_t\bm B)^\top+\bm K_t\bm\Lambda_{\theta_s}\bm K_t^\top.
\end{equation}

During training, the innovation $(\bm{x}_t - \bm{B} \bm{z}_{t\mid t-1})$ and their covariance $\bm{\Omega}_t$ are passed directly to the innovation NLL objective (Eq. \ref{eq:filtered_likelihood}) to compute the (localized) loss before proceeding to $t+1$.

Interleaving prediction and assimilation prevents the innovation objective from scoring a single long open-loop covariance product. Under standard linear-Gaussian regularity conditions, this recursion approaches a stable equilibrium \citep{reif1999stochastic}. In the nonlinear setting, we use the same local mechanism as a stabilizing inductive bias, see also Proposition ~\ref{prop:volume_pressure_localization}.

\subsection{KAFFEE optimization details}
\label{app:kaffee_opt_details}
We now proceed to describe additional training details of KAFFEE. The differentiable EKF described in Appendix~\ref{app:ekf_equations} is optimized end-to-end. Algorithm~\ref{alg:kaffee_opt} then states the generic KAFFEE optimization loop: each step samples one or more truncated windows, propagates the predict--propagate--update recursion through every step, accumulates the innovation NLL, and backpropagates jointly through the core parameters $\theta$ and the shell parameters $(\bm \Gamma,\bm \Lambda)$.

\textbf{AL-RNN Jacobian.} The closed-form piecewise-linear Jacobian of the AL-RNN is  (cf. Eq.~\ref{eq:alrnn}),
\[
\frac{\partial F_{\bm\theta}}{\partial \bm z}(\bm z)
=
\bm A+\bm W\bm D(\bm c(\bm z)),
\qquad
c_j(\bm z)=\mathbf{1}\{z_{M-P+j}>0\}.
\]
This Jacobian is used for covariance propagation, $G_{\mathrm{tan}}$, and Lyapunov spectra. For DynaMix experts we use the analogous per-expert Jacobian, with mixture weights detached as described in Appx.~\ref{app:dynamix_practical_details}.

\textbf{Shell initialization and optimization details.}
KAFFEE jointly optimizes the AL-RNN core and diagonal shell from the AL-RNN checkpoint with the initialization embedding $\bm E$ fixed after deterministic training. The diagonal noise shell is initialized from the same-seed AL-RNN's teacher-forcing residual scale (in standardized observation coordinates). Writing $\hat{\bm r}_0\in\mathbb{R}_+^N$ for the coordinate-wise standard deviation of the teacher-forcing residuals on the noisy training observations, we use a fixed $90/10$ variance split ($\alpha = 0.9$ in Algorithm~\ref{alg:kaffee_opt})
\[
\bm \Lambda_0 = \operatorname{diag} \big((\sqrt{0.9}\,\hat{\bm r}_0)^2\big),
\qquad
\bm \Gamma_0 = \bar q_0^2\bm I_M,
\qquad
\bar q_0:=\frac{1}{N}\sum_{j=1}^N \sqrt{0.1}\,\hat r_{0,j}.
\]
At each gradient step we sample 16 contiguous windows of total length 40 transitions (i.e., 41 observations) from the standardized noisy training trajectory, propagate the EKF through the full window, and include only the final 32 innovation terms in the loss after an 8-step burn-in. Optimization uses Adam ($\beta_1 = 0.9$, $\beta_2 = 0.999$) \citep{kingma2014adam} with learning rate $5\times10^{-3}$ decayed to $10^{-5}$ via cosine annealing, gradient clipping $2.0$, and $150$ joint iterations. Among tracked optimization iterations, we select the model with lowest $D_{\mathrm{stsp}}^{\mathrm{stoch}}$ against a $10^4$-step noisy validation rollout. 

\textbf{Windowed truncation and covariance burn-in.}
\label{app:windowed_truncation_covburnin}
During training, KAFFEE is optimized using BPTT over contiguous truncated sequences. Consequently, at the start of each window, the filter mean is anchored to the observed state and the predictive covariance is reinitialized from the ``static'' shell parameters. If those first few innovation terms enter the loss immediately, the optimizer can partly improve the objective by exploiting transient covariance that exists only because the window was truncated there. The goal is to couple predictive uncertainty to the learned transition map through Jacobian-propagated covariance, so the loss-contributing portion of each window should reflect covariance propagation that has already had a short opportunity to align with the local dynamics.

For that reason, KAFFEE runs use a short covariance burn-in. We sample windows of total length $L_{\mathrm{burn}}+L+1$ observations with burn-in phase $L_{\mathrm{burn}}$ and loss horizon $L$, propagate the filter recursion across all steps, but exclude the first $L_{\mathrm{burn}}$ innovation terms from the loss. Writing the per-step innovation contribution as
\[
\ell_t = \frac12\Big[ \log\det \bm \Omega_t + \bm v_t^\top \bm \Omega_t^{-1}\bm v_t + N\log(2\pi)\Big],
\]
the loss on one sampled window becomes
\[
\mathcal L_{\mathrm{EKF}}^{\mathrm{score}}=\frac{1}{L}\sum_{t=L_{\mathrm{burn}}+1}^{L_{\mathrm{burn}}+L} \ell_t.
\]
The covariance burn-in phase alleviates the optimizer's incentive to fit a covariance transient. To match schedules, the open-loop NLL control mirrors the same sampled window length. In practice, $L_{\mathrm{burn}}$ should be long enough that the predictive covariance is no longer dominated by its initialization. In our experiments, we set $L_{\mathrm{burn}} = 8$ throughout. 

\begin{algorithm}[!htpb]
\caption{KAFFEE optimization protocol}
\label{alg:kaffee_opt}
\begin{algorithmic}[1]
\STATE \textbf{Input:} observations $\bm x_{1:T}$, transition map $F_\theta$, observation matrix $\bm B$, deterministic prior model $\theta_c^{(0)}$, covariance burn-in length $L_{\mathrm{burn}}$, scored length $L$, variance split $\alpha\in(0,1)$ (here, $\alpha = 0.9$ throughout), AL-RNN embedding matrix $\bm{E}$.
\STATE \textbf{Initialization.} Compute one-step teacher-forcing residual scales $\hat{\bm r}_0\in\mathbb R_+^N$ on noisy training trajectory under $\theta_c^{(0)}$. Set
\[
\bm \Lambda_0
=
\operatorname{diag} \big(\alpha\,\hat{\bm r}_0^{\,2}\big),
\qquad
\bm \Gamma_0
=
\bar q_0^2\bm I_M,
\qquad
\bar q_0
=
\frac{1}{N}
\sum_{j=1}^N
\sqrt{1-\alpha}\,\hat r_{0,j}.
\]
Set $\theta_c\leftarrow\theta_c^{(0)}$, $\theta_s \leftarrow(\bm\Gamma_0,\bm\Lambda_0)$, and initialize $\theta \leftarrow (\theta_c,\theta_s)$.
\FOR{each optimization step}
    \STATE Sample a mini-batch $\mathcal B$ of windows
    $\bm x_{t_0:t_0+L_{\mathrm{burn}}+L}$.
    \STATE Form the current diagonal shell covariances $\bm\Gamma$ and $\bm\Lambda$.
    \STATE Set $\mathcal L_{\mathrm{batch}}\leftarrow 0$.
    \FOR{each window in $\mathcal B$}
    \STATE \textbf{Initialize filtered belief:} 
    In case of using the AL-RNN as deterministic prior, set $\bar{\bm z}_{t_0}=\bm E\bm x_{t_0}$ using the fixed embedding inherited from deterministic training (Appx.~\ref{app:alrnn_implementation_details}). Initialize the filter through the observation matrix and initialize the covariance from the current learned process covariance:
    \[
    \bm z_{t_0\mid t_0}
    =\bar{\bm z}_{t_0}+\bm B^\top(\bm x_{t_0}-\bm B\bar{\bm z}_{t_0}),\qquad\bm\Sigma_{t_0\mid t_0}=\bm\Gamma.
    \]
       \FOR{$t=t_0+1,\ldots,t_0+L_{\mathrm{burn}}+L$}
            \STATE \textbf{Predict:}
            \[
            \bm z_{t\mid t-1}=F_\theta(\bm z_{t-1\mid t-1}),
            \qquad
            \bm J_t
            =
            \frac{\partial F_\theta}{\partial \bm z}(\bm z_{t-1\mid t-1}) .
            \]
            \STATE \textbf{Propagate covariance:}
            \[
            \bm\Sigma_{t\mid t-1}
            =
            \bm J_t\bm\Sigma_{t-1\mid t-1}\bm J_t^\top
            +
            \bm\Gamma .
            \]
            \STATE \textbf{Form innovation:}
            \[
            \bm v_t
            =
            \bm x_t-\bm B\bm z_{t\mid t-1},
            \qquad
            \bm\Omega_t
            =
            \bm B\bm\Sigma_{t\mid t-1}\bm B^\top+\bm\Lambda .
            \]
            \IF{$t>t_0+L_{\mathrm{burn}}$}
                \STATE Accumulate innovation NLL
                \[
                \mathcal L_{\mathrm{batch}}
                \leftarrow
                \mathcal L_{\mathrm{batch}}
                +
                \frac{1}{2|\mathcal B|L}
                \left[
                \log\det\bm\Omega_t
                +
                \bm v_t^\top\bm\Omega_t^{-1}\bm v_t
                +
                N\log(2\pi)
                \right].
                \]
            \ENDIF
            \STATE \textbf{Update:}
            \[
            \bm K_t
            =
            \bm\Sigma_{t\mid t-1}\bm B^\top\bm\Omega_t^{-1},
            \]
            \[
            \bm z_{t\mid t}
            =
            \bm z_{t\mid t-1}+\bm K_t\bm v_t,
            \qquad
            \bm\Sigma_{t\mid t}
            =
            (\bm I-\bm K_t\bm B)\bm\Sigma_{t\mid t-1}(\bm I-\bm K_t\bm B)^\top+\bm K_t\bm\Lambda\bm K_t^\top.
            \]
        \ENDFOR
    \ENDFOR
    \STATE Backpropagate $\mathcal L_{\mathrm{batch}}$ through the predict--propagate--update computation and update $\theta$.
\ENDFOR
\end{algorithmic}
\end{algorithm}

\section{Experimental Details}
\label{app:experimental_details}

This section consolidates model, dataset, metric, and optimization protocol details required to reproduce Tables~\ref{tab:main_l96_ablation} and~\ref{tab:dynamix_results_table}.

\subsection{AL-RNN architecture and deterministic training}
\label{app:alrnn_implementation_details}

All baselines in Table~\ref{tab:main_l96_ablation} share the same AL-RNN core and deterministic pretraining protocol. AL-RNNs were introduced for interpretable DSR through sparse use of nonlinear units, enabling parsimonious piecewise linear representations and compact symbolic descriptions of the learned dynamics \citep{brenner2024almostlinear}. They have also been used for analyzing the role of nonlinearity in sequence modeling \citep{brenner2026uncovering}, continual DSR \citep{hashemi2026continual}, and learning control policies from reconstructed dynamics \citep{emonds2026learning}. 

The AL-RNN's deterministic transition for $\bm z_t\in\mathbb R^M$ is
\begin{equation}
\label{eq:alrnn}
\bm z_{t+1}
=
F_{\bm\theta}(\bm z_t)
=
\bm A\bm z_t + \bm W\,\bm\phi^\ast(\bm z_t) + \bm h,
\end{equation}
where $\bm A \in \mathbb{R}^{M\times M}$, $\bm W \in \mathbb{R}^{M\times M}$ are diagonal and full matrices, respectively. $\bm\phi^\ast$ applies a ReLU gating to the last $P$ coordinates only, and the first $M-P$ coordinates are linear. Equivalently, writing the binary switching code $\bm c_t\in\{0,1\}^P$ and diagonal gate matrix
\[
\bm D(\bm c_t)
=
\mathrm{diag}\big(\underbrace{1,\dots,1}_{M-P},c_{t,1},\dots,c_{t,P}\big),
\]
the transition can be written as
\begin{equation}
\label{app:alrnn-transition}
\bm z_{t+1}=
\big(\bm A+\bm W\bm D(\bm c_t)\big)\bm z_t + \bm h.
\end{equation}
We use latent dimension $M=100$, $N=20$, and $P=50$ ReLU-gates. Observations are always the first $N$ latent coordinates, so the observation operator is fixed to
\begin{equation}
\label{app:observation_operator}
\bm B = [\bm I_N\ \bm 0] \in \mathbb R^{N\times M},
\qquad
\bm x_t = \bm B\bm z_t.
\end{equation}
To initialize a rollout from the first observation $\bm x_1\in\mathbb R^N$, a learned linear embedding $\bm E\in\mathbb R^{M\times N}$ is applied to overwrite the observed coordinates \citep{brenner2024almostlinear},
\[
\bm z_1 \leftarrow \bm E\bm x_1,
\qquad
\bm z_1 \leftarrow \bm z_1 + \bm B^{\top}(\bm x_1-\bm B\bm z_1).
\]

\textbf{Deterministic prior pretraining.}
Each deterministic AL-RNN core is pretrained on the standardized noisy Lorenz-96 training observations described in Appx.~\ref{app:l96_dataset}. Training uses BPTT with sparse identity teacher forcing and forcing interval $\tau=20$ \citep{mikhaeil2022difficulty,brenner2024almostlinear}. Because the observation map is the identity projection onto the first $N$ coordinates, this is identity teacher forcing (ITF): at forcing times $\mathcal T_\tau=\{t:t\equiv 0\ (\mathrm{mod}\ \tau),\ t>0\}$ we apply
\[
\tilde{\bm z}_t
=
\bm z_t + \bm B^{\top}(\bm x_t-\bm B\bm z_t),
\]
which overwrites the observed coordinates and leaves the remaining coordinates untouched. Between forcing times the model is free-running. Training minimizes next-step MSE on the observed coordinates along these interleaved rollouts,
\[
\mathcal L_{\mathrm{ITF}}
=
\frac{1}{L_{\mathrm{BPTT}}N}
\sum_{t=1}^{L_{\mathrm{BPTT}}}
\lVert \bm B\bm z_{t+1}-\bm x_{t+1}\rVert_2^2,
\qquad
L_{\mathrm{BPTT}}=50,
\]
using mini-batches of $16$ contiguous windows, $50$ batches per epoch, and $2\cdot 10^3$ epochs on the $T_{\mathrm{train}}=8 \cdot 10^4$ time steps long training trajectory. Optimization uses RAdam with an exponential learning-rate decay from $10^{-3}$ to $10^{-5}$ \citep{liu2020variance}. Initialization follows \citep{brenner2024almostlinear}: $\bm h=\bm 0$, the entries of $\bm W$ are i.i.d. samples drawn from $\mathcal N(0,0.01)$, the diagonal of $\bm A$ is taken from a normalized random positive-definite matrix, and the embedding $\bm E$ is initialized entry-wise from $\mathcal U(-1/\sqrt{N},1/\sqrt{N})$. Table~\ref{tab:main_l96_ablation} reports $20$ independent AL-RNNs trained with this same protocol. Every probabilistic method in this table inherits the corresponding deterministic core and initializer $\bm E$.

\subsection{Probabilistic baselines}
\label{app:alrnn_probabilistic_baselines}

For Lorenz-96, the matched controls use the same AL-RNN initializations and the same diagonal noise parametrization and initialization as KAFFEE. For shell training via BPTT, KAFFEE and all controls use Adam with learning rate $5\times10^{-3}$ cosine-annealed to $10^{-5}$, gradient clipping $2.0$, and $150$ joint iterations (Appx.~\ref{app:kaffee_opt_details}). Each sampled window contains 40 transitions (and $41$ observations): 8 transitions are propagated for covariance burn-in and the final 32 steps are scored. The EKF covariance update uses the Joseph form (Eq.~\ref{eq:joseph}).

\textbf{Open-loop objectives.}
For a training sequence with origin $t_0$, the two open-loop baselines in Table~\ref{tab:main_l96_ablation} propagate
$\bm z_{t\mid t_0}=F_{\theta_c}(\bm z_{t-1\mid t_0})$ and
$\bm\Sigma_{t\mid t_0}$ without intermediate assimilation. They then score the Gaussian open-loop NLL $-\log\mathcal N \left(\bm x_t;\bm B\bm z_{t\mid t_0},
\bm B\bm\Sigma_{t\mid t_0}\bm B^\top+\bm\Lambda_{\theta_s}\right).$ The two open-loop methods differ only in their covariance recursion: \emph{tied propagation} (core collapse) uses the Jacobian recursion (Eq.~\ref{eq:ekf_predict_cov}); \emph{no
propagation} (noise masking)
replaces this by the state-invariant recursion
$\bm\Sigma_{t+1\mid t_0}=\bm\Sigma_{t\mid t_0}+\bm\Gamma_{\theta_s}$,
so the only path through which the core enters the loss is the predicted mean.

\textbf{Filtered no covariance propagation.}
This control instantiates the filtered \emph{no-propagation} (blind uncertainty) control. The objective and optimization schedule are identical to KAFFEE. The only difference is the covariance recursion: $\bm\Sigma_{t\mid t-1}=\bm\Sigma_{t-1\mid t-1}+\bm\Gamma_{\theta_s}$ replaces the Jacobian-propagated recursion of Eq.~\ref{eq:ekf_predict_cov}, so the core no longer enters the predictive covariance. The mean is still predicted by the nonlinear core transition and corrected using the Kalman update. We can see this as running the regular EKF but using an identity approximation for the local Jacobians.

\textbf{Frozen-core shell calibration.}
This row in Table~\ref{tab:main_l96_ablation} quantifies how much of the improvement is attainable by post-hoc noise calibration alone. We freeze the AL-RNN prior and optimize only the diagonal stochastic shell $(\bm \Gamma_{\theta_s},\bm \Lambda_{\theta_s})$ under the EKF innovation NLL. By construction the shell-off core metrics are identical to the deterministic AL-RNN prior. This control therefore shows how much filtered probabilistic and stochastic dynamical performance is achievable without ever updating the deterministic core.

\subsection{Lorenz-96 dataset}
\label{app:l96_dataset}
For Table~\ref{tab:main_l96_ablation} we use the $N=20$-dimensional Lorenz-96 system \citep{lorenz1996predictability}:
\[
\frac{\mathrm d z_i}{\mathrm d t}
=
(z_{i+1}-z_{i-2})z_{i-1}-z_i+F,
\qquad i=1,\dots,20,
\]
with forcing $F=8$ and discretization interval $\Delta t=0.04$. 
We denote the corresponding discrete-time latent state by $\bm z_t\in\mathbb R^{20}$. Let $\bm\mu_{\mathrm{train}}\in\mathbb R^{20}$ and $\bm s_{\mathrm{train}}\in\mathbb R_+^{20}$ denote the coordinate-wise mean and standard deviation of the noisy training trajectory in raw coordinates. The models are trained on trajectories generated in raw coordinates and subsequently standardized,
\[
\tilde{\bm z}_t
=
\operatorname{diag}(\bm s_{\mathrm{train}})^{-1}(\bm z_t^{\mathrm{raw}}-\bm\mu_{\mathrm{train}}),
\qquad
\tilde{\bm x}_t
=
\operatorname{diag}(\bm s_{\mathrm{train}})^{-1}(\bm x_t^{\mathrm{raw}}-\bm\mu_{\mathrm{train}}).
\]

The process noise $\sigma_{\mathrm{proc}}$ is defined on a standardized scale. Concretely, we first fix reference coordinate scales $\bm s_{\mathrm{ref}}\in\mathbb R_+^{20}$ from a long clean run and then simulate the raw stochastic dynamics as a Runge--Kutta 4-th order drift step plus Euler--Maruyama diffusion,
\begin{equation*}
\bm z_{t+1}^{\mathrm{raw}}
=
\mathrm{RK4}_{\Delta t}(\bm z_t^{\mathrm{raw}})
+
\sigma_{\mathrm{proc}}\,\sqrt{\Delta t}\,\operatorname{diag}(\bm s_{\mathrm{ref}})\bm\varepsilon_t, \qquad
\bm\varepsilon_t\sim\mathcal N(\bm 0,\bm I_{20}).
\end{equation*}
Thus $\sigma_{\mathrm{proc}}$ specifies the diffusion magnitude relative to a fixed standardized scale, while the numerical integration itself is carried out in raw coordinates. After standardization by $\bm s_{\mathrm{train}}$, the resulting per-step diffusion in standardized coordinates is therefore approximately $\sigma_{\mathrm{proc}}\sqrt{\Delta t}$ per coordinate.

Observation noise is added separately, after latent simulation, in raw coordinates:
\begin{equation*}
\bm x_t^{\mathrm{raw}}= \bm z_t^{\mathrm{raw}}+
\sigma_{\mathrm{obs}}\,\operatorname{diag}(\bm s_{\mathrm{latent}})\bm\eta_t,\qquad
\bm\eta_t \sim\mathcal N(\bm 0,\bm I_{20}).
\end{equation*}
Because $\bm s_{\mathrm{train}}$ is the coordinate-wise standard deviation of these noisy observations, it includes the observation-noise variance. After z-scoring, the observation-noise standard deviation is therefore slightly below the generation coefficient $\sigma_{\mathrm{obs}}=0.1$ (in our generated data, $0.0995$). $s_{\mathrm{latent}}$ is the coordinate-wise standard deviation of the simulated latent trajectory (before adding observation noise). Table~\ref{tab:main_l96_ablation} quotes these generation coefficients $\sigma_{\mathrm{obs}}=\sigma_{\mathrm{proc}}=0.1$. All models (AL-RNN, KAFFEE, probabilistic baselines) train on these standardized noisy observations.

\subsection{Metric computation and aggregation}
\label{app:metric_computation}
For the predictive metrics in Table~\ref{tab:main_l96_ablation}, the forecast horizon used is a quarter of a Lyapunov time, i.e. $\tfrac14\tau_{\lambda} = 0.25\lambda_{1,\mathrm{ref}}^{core\, -1}$:
\[
H
=
\left\lfloor
\frac{0.25}{\lambda^\mathrm{core}_{1,\mathrm{ref}}\,\Delta t}
\right\rfloor
= 4
\text{ steps,}
\]
where $\lambda_{1,\mathrm{ref}}^\mathrm{core}$ is the leading Lyapunov exponent of the ground truth deterministic dynamics. To evaluate the local predictive metrics, each seed/ensemble cohort first averages its predictive metrics over evenly spaced forecast starts, and the table medians and robust SE values are then formed across those independent runs. Throughout, we use $K = 200$ ensemble members and an observation-noise floor of $10^{-3}$ for the standard deviation inside the Gaussian likelihood.

\textbf{NLL.}
For each forecast start, all ensemble members share the latent state initialized from the noisy observation at that origin. We autoregressively propagate this ensemble forward $H$ steps without assimilating intermediate observations. At every $h=1,\dots,H$ we compute the ensemble mean $\bm\mu_h$ and standard deviation $\bm\sigma_h$ over the ensemble members and evaluate the per-step Gaussian NLL against the noisy reference $\bm y_h$ by
\[
\operatorname{NLL}(\bm y_h;\bm\mu_h,\bm\sigma_h)
=
\frac{1}{N}
\sum_{j=1}^{N}
\frac12\left[
\log(2\pi\sigma_{h,j}^2)
+
\frac{(y_{h,j}-\mu_{h,j})^2}{\sigma_{h,j}^2}
\right].
\]
The reported NLL is the path mean over horizon $H$, averaged over $100$ forecast starts.

\textbf{CRPS.}
The CRPS is evaluated at the endpoint of the bounded horizon $H$, using $10$ forecast starts. For the reference observation $\bm y$ at the endpoint and the corresponding ensemble $\{\bm x^{(k)}\}_{k=1}^{K}$, we compute the marginal ensemble CRPS
\[
\operatorname{CRPS}(\bm y,\{\bm x^{(k)}\})
=
\frac{1}{N}
\sum_{j=1}^{N}
\left[
\frac{1}{K}\sum_{k=1}^{K}\lvert x_j^{(k)}-y_j\rvert
-
\frac{1}{2K^2}\sum_{k=1}^{K}\sum_{\ell=1}^{K}\lvert x_j^{(k)}-x_j^{(\ell)}\rvert
\right].
\]
The reported entry is the average of this endpoint CRPS over the $10$ starts.

\textbf{Dynamical grounding correlation $G_{\mathrm{tan}}$.}
To measure whether a model's local uncertainty is dynamically grounded, this metric correlates the coordinate-wise one-step predictive variance growth against the local tangent expansion of the true system. We use $G_{\mathrm{tan}}$ as a controlled probabilistic DSR certification metric: it requires access to the true local Jacobians and is therefore not directly available on real-world datasets. In such settings, analogous diagnostics would need to rely on proxy tangent estimates, withheld-observation spread growth, or innovation residual checks rather than oracle Jacobians.

Let $\bm\Sigma^m_{t\mid t}$ be the filtered/posterior covariance of method/model $m$, and let
\[
\bm \Omega^m_{t\mid t}:=\bm B\bm\Sigma^m_{t\mid t}\bm B^\top
\]
be its observed-space posterior covariance. Let $\bm J^m_t\in\mathbb{R}^{M\times M}$ be the Jacobian used for that model's covariance propagation, evaluated at the filtered mean, and let $\bm J^\star_t\in\mathbb{R}^{N\times N}$ be the reference one-step Jacobian at the corresponding true state. The filter's predict step still adds process noise $\bm\Gamma$ (Eq.~\ref{eq:ekf_predict_cov}), but for $G_{\mathrm{tan}}$ we only use the noise-free pushforward of the posterior. For the model this (observation-space) pushforward is $\bm B\bm J^m_t\bm\Sigma^m_{t\mid t}\bm J^{m\top}_t\bm B^\top$, and for the reference, the observed posterior is pushed forward by the reference Jacobians, $\bm J^\star_t\bm\Omega^m_{t\mid t}\bm J_t^{\star\top}$. We define the model's predictive log covariance growth of observed coordinate $j$ at time $t$ as
\[
\Delta u^m_{j,t}
:=
\log\left[\bm B\bm J^m_t\bm\Sigma^m_{t\mid t}\bm J^{m\top}_t\bm B^\top\right]_{jj}-\log\left[\bm \Omega^m_{t\mid t}\right]_{jj},
\]
and the reference tangent log-growth as
\[
\Delta u^{\star,m}_{j,t}:=\log\left[\bm J^\star_t\bm \Omega^m_{t\mid t}\bm J_t^{\star\top}\right]_{jj}-\log\left[\bm \Omega^m_{t\mid t}\right]_{jj}.
\]
The first quantity is the variance growth assigned by the model; the second one is the variance growth that the true local tangent dynamics would assign to the same posterior. We then report the rank correlation of these two quantities:
\begin{equation}
\label{eq:gtan}
G_{\mathrm{tan}}(m) :=
\operatorname{Spearman}_{(j,t)}
\left(\Delta u^m_{j,t},\Delta u^{\star,m}_{j,t}\right),
\end{equation}
pooled over observed coordinates and time. Since we subtract the posterior log variance baseline $\log(\bm \Omega^m_{t\mid t})_{jj}$ on both sides, $G_{\mathrm{tan}}$ focuses on local variance growth, not on absolute magnitude. If either growth is constant (in our implementation, if either pooled series has standard deviation below $10^{-5}$), the correlation is defined to be $0$. This corresponds to the ``no-propagation'' case: there, the covariance propagation map is the identity (not the core's transition Jacobian), so $\Delta u^m_{j,t}=0$ at every coordinate and time. 

Unless otherwise stated, $G_{\mathrm{tan}}$ denotes the one-step propagated variance growth correlation. It is computed after a $200$-step warm-up and then pooled over the following
$4\cdot 10^3$ prediction/update steps.

For the forecast-skill spectrum in Appx.~\ref{app:forecast_skill_spectrum}, we
also use an $H$-step version, denoted $G_{\mathrm{tan}}^{(H)}$. Here the filter assimilates only every $H$ steps. From the filtered covariance at assimilation, the model covariance is pushed forward by its covariance propagation map (i.e., Jacobian for tied propagation, identity for no-propagation models) at each forecast state, and the reference covariance by the true tangent map at the true state, for every lead time $h=1,\ldots,H$. We do not add process noise $\bm\Gamma$ between the assimilations (this is for computing $G_{\mathrm{tan}}^{(H)}$; the filter continues to add process noise when propagating the predictive covariance). The subtracted baseline is the filtered variance from the last assimilation, and the correlation pools these lead times over observed coordinates and assimilation cycles.

\textbf{Stochastic dynamical fidelity.}
For shell-on stochastic evaluation, we initialize the model rollout on a noisy state from held-out data, then score a total length of $T=L=5 \cdot 10^4$ time steps after discarding the first $10^3$ steps to account for a potential transient phase. The metrics compare those rollouts with stored reference trajectories of the same length. For the three-dimensional systems of the DynaMix experiments (and the additional Lorenz-63 ablation, Appx.~\ref{app:l63}), the state-space divergence is the histogram KL divergence
\[
D_{\mathrm{stsp}}^{\mathrm{stoch}}
=
\sum_b p_{\mathrm{ref}}(b)
\log\frac{p_{\mathrm{ref}}(b)}{p_{\mathrm{gen}}(b)},
\]
using $30$ bins per coordinate. For Lorenz-96, the exponential scaling of bins with data dimension renders this histogram approach infeasible, so density estimation is executed instead by Gaussian mixture models \citep{koppe2019identifying,hess2023generalized,brenner2024almostlinear}. The reference and generated distributions are approximated by isotropic Gaussian mixtures centered on the respective trajectories in standardized observation coordinates, denoted by $\bm x^\mathrm{ref}$ and $\bm x^{\mathrm{gen}}$:
\begin{equation}
p_{\mathrm{ref}}(\bm x)\approx\frac{1}{T}\sum_{t=1}^{T}\mathcal{N}(\bm x\mid\bm x_t^\mathrm{ref},\bm\Sigma),
\qquad
p_{\mathrm{gen}}(\bm x)\approx\frac{1}{L}\sum_{\ell=1}^{L}\mathcal{N}(\bm x\mid \bm x_\ell^{\mathrm{gen}},\bm\Sigma),
\end{equation}
and the KL divergence between them is estimated via:
\begin{equation}
\tilde{D}_{\mathrm{stsp}} \big(p_{\mathrm{ref}}(\bm x),\,p_{\mathrm{gen}}(\bm x)\big)
\approx
\frac{1}{n}\sum_{i=1}^{n}\log
\frac{\tfrac{1}{T}\sum_{t=1}^{T}\mathcal{N}(\bm x_i\mid\bm x_t^{\mathrm{ref}},\bm\Sigma)}
     {\tfrac{1}{L}\sum_{\ell=1}^{L}\mathcal{N}(\bm x_i\mid\bm x_\ell^{\mathrm{gen}},\bm\Sigma)},
\label{eq:kl_mc}
\end{equation}
using $n$ Monte Carlo samples $\bm x_i\sim p_{\mathrm{ref}}$ \citep{hershey2007approximating}. The covariance $\bm\Sigma=\sigma^2\bm I$ is a scaled identity, and for compatibility with \citet{hess2023generalized,brenner2024almostlinear} we set $\sigma^2=1$. In our implementation $n=10^3$ MC draws.

The unstable-volume growth rate of the stochastic dynamics and its error are given by the sum of positive Lyapunov exponents $\lambda^\mathrm{stoch}_i$:
\[
h^\mathrm{stoch}_\lambda:=\sum_i \max(\lambda_i^\mathrm{stoch},0),
\qquad
|\Delta h^\mathrm{stoch}_\lambda|:=|\hat h^\mathrm{stoch}_\lambda-h^\mathrm{stoch}_{\lambda,\mathrm{ref}}|,
\]
using QR re-orthonormalization on the (shell-on) one-step Jacobian recursion after a $10^3$-step burn-in and $3\cdot10^4$ scored transitions to estimate Lyapunov spectra \citep{benettin1980lyapunov,geist1990comparison}. Observation noise is excluded from this tangent map, so we only evaluate Jacobians at process-noise-perturbed states. The reference spectra are obtained from the process-noisy true maps along the same latent reference.

\textbf{Deterministic dynamical fidelity.}
To evaluate the autonomous core we initialize the deterministic rollout from a clean (noise-free) state, and apply the same rules (rollout lengths, transient burn-in) as for the stochastic dynamical fidelity, using the noise-free ground truth. The autonomous deterministic versions of the state space divergence $D_\mathrm{stsp}^\mathrm{core}$ and unstable-volume growth rate $h_{\lambda}^{\mathrm{core}}$ are evaluated analogously, by stripping the shell, i.e. generating rollouts only using the autonomous transition map.

To evaluate spectral overlap, we use the mean Hellinger distance between the per-coordinate power spectra, $D_H^{\mathrm{core}}$, after each signal is centered, standardized, transformed with a real Fourier transform, smoothed by a Gaussian kernel of width $20$ (chosen according to prior literature \citep{hess2023generalized,brenner2024almostlinear}), and renormalized to unit mass. Writing the smoothed normalized spectra as $p_j(\omega)$ and $q_j(\omega)$, we report
\[
D_H^{\mathrm{core}}
=
\frac{1}{N}\sum_{j=1}^{N}
\sqrt{1-\sum_{\omega}\sqrt{p_j(\omega)q_j(\omega)}}.
\]

\textbf{Aggregation.}
We first compute every metric per seed or per disjoint ensemble cohort and then report the median together with the robust standard error of the median,
\[
\operatorname{SE}_{\mathrm{median}}
\approx
0.929\,\frac{\operatorname{IQR}}{\sqrt{n}}.
\]
For the Lorenz-96 experiments (Table~\ref{tab:main_l96_ablation}) we use $20$ independent seeds. For DynaMix (Table~\ref{tab:dynamix_results_table}), we use $20$ seeds as well on each of the 13 systems except the Dadras system. Here, since open-loop tied diverged for $3$ seeds, we dropped those seeds for every method. 

\subsection{DynaMix architecture and probabilistic adaptation protocol}
\label{app:dynamix_practical_details}

This section records the complete numerical results and additional details of the DynaMix benchmark experiment in Section~\ref{subsec:dynamix_results}. We do not re-state the full DynaMix architecture; for that we refer to \citet{hemmer2025true}. The purpose here is to state the state-space interpretation of DynaMix used in the benchmark, the shell parameterization, the Jacobian construction, the KAFFEE protocol, and the definitions of the applied baselines.

\textbf{State-space model form of DynaMix.} DynaMix is a DSR foundation model, in the sense that its zero-shot dynamics are not supposed to (only) constitute accurate forecasts, but also capture the long-term structure underlying the (chaotic) data-generating system \citep{hemmer2025true,durstewitz2026position}. While DynaMix is aligned with deterministic DSR, our present work emphasizes the need for uncertainty-aware, probabilistic DSR (foundation) models. To pursue this with DynaMix, we apply KAFFEE to embed it into a Bayesian filter/state-space model framework. In particular, we treat DynaMix as a pretrained latent transition model with fixed context.

Let $\bm z_t\in\R^M$ denote the latent state, $\bm x_t\in\R^N$ the observed state, and $\mathrm{ctx}$ the context supplied to the pretrained model. This context consists of a reference sequence of observations from the target system, from which DynaMix is asked to zero-shot the underlying dynamics. The transition is written as
\[
\bm z_{t+1}=F_{\mathrm{DynaMix}}(\bm z_t,\mathrm{ctx}, \bm\varepsilon_t^{\mathrm{gate}}), \qquad \bm \varepsilon_t^\mathrm{gate} \sim \mathcal{N}(0,\bm \Sigma_\mathrm{gate}),
\]
with DynaMix's mixture-of-experts form
\[
F_{\mathrm{DynaMix}}(\bm z_t,\mathrm{ctx})
=
\sum_{k=1}^{E} w_{t,k} f_k(\bm z_t), \qquad \bm w_t=g_{\psi}(\mathrm{ctx},\bm z_t,\bm\varepsilon_t^{\mathrm{gate}}),
\]
where $f_k:\mathbb R^M\to\mathbb R^M$ denotes the $k$-th (AL-RNN) expert's transition map, and $w_{t,k}$ is its state- and context-dependent mixture weight. The gating network $g_{\psi}$ combines a convolutional neural network encoding of the context trajectory with the current latent state via a state attention mechanism to produce these expert-weights for the $E$ experts on a simplex $\bm w_t\in\Delta^{E-1}$. This state attention mechanism contains an internal diagonal Gaussian perturbation (``exploration noise'') $\bm \varepsilon^{\mathrm{gate}}_t$.

\textbf{Gate stochasticity.}
The pretrained DynaMix model carries a small internal standardized per-coordinate noise in its attention computation: before context-attention weights are formed, the current latent state is projected to observation space and perturbed by a learnable diagonal Gaussian noise term. In the model used here,
$\bm{\Sigma}_\mathrm{gate} \approx \mathrm{diag}(0.003, 0.003, 0.008)$ in standardized coordinates. We retain this mechanism during adaptation rather than modifying the pretrained foundation model, see Eq.~\ref{eq:dynamix_ssm}. For the EKF approximation, each gate evaluation uses one realization of this internal noise, we do not explicitly marginalize its contribution to the predictive covariance. We interpret the learned diagonal process covariance $\bm\Gamma$ as an effective process noise term that absorbs residual error from this approximation. We refer to $F_{\mathrm{DynaMix}}$ with the external Gaussian shell removed as the shell-off core transition; this transition retains the internal gate noise.

\textbf{DynaMix specifications.} The model specifications were directly adopted from \cite{hemmer2025true}. Specifically, the model has $M=30$ latent dimensions, $N=3$ observed dimensions, $E=10$ AL-RNN experts with $P=2$ ReLU-gated coordinates, respectively. Each expert is therefore piecewise affine,
\[
f_k(\bm z)=\bm A_k\bm z+\bm W_k\bm\phi_k^\ast(\bm z)+\bm h_k,
\]
with AL-RNN parametrization as described in Appx.~\ref{app:alrnn_implementation_details}.

Throughout the DynaMix experiments, the first $N$ latent coordinates are identified with the observable state. In other words, we apply the same observation operator $\bm B$ that was used for the single AL-RNN experiments on Lorenz-96 (Eq.~\ref{app:observation_operator}).

\textbf{Shell parameterization.} All probabilistic DynaMix baselines embed the model into a state-space model framework using an external Gaussian shell around the DynaMix core. We use the same diagonal parameterization for every DynaMix variant:
\begin{equation}
\label{eq:dynamix_ssm}
\bm z_{t+1}=F_{\mathrm{DynaMix}}(\bm z_t,\mathrm{ctx},\bm \varepsilon_t^{\mathrm{gate}})+\bm \Gamma^{1/2}\bm \xi_t,
\qquad
\bm x_t=\bm B\bm z_t+\bm \Lambda^{1/2}\bm \eta_t,
\end{equation}
where $\bm \xi_t \sim \mathcal{N}(\bm 0, \bm I_M)$ and $\bm \eta_t \sim \mathcal{N}(\bm 0, \bm I_N)$,
and
\[
\bm \Gamma=\diag(q_1^2,\ldots,q_M^2),
\qquad
\bm \Lambda=\diag(r_1^2,\ldots,r_N^2).
\]
These initial scales are chosen analogously to the Lorenz-96 AL-RNN experiments, via teacher-forced residuals and variance split $\alpha = 0.9$ (see Appx.~\ref{app:kaffee_opt_details}).

\textbf{In-context KAFFEE protocol.} We use the released model \texttt{dynamix-3d-alrnn-v1.0}, distributed under CC BY 4.0, published alongside \cite{hemmer2025true} in \texttt{\url{https://github.com/DurstewitzLab/DynaMix-python}}. This model was pretrained on a benchmark of three-dimensional chaotic attractors \citep{gilpin2021chaos}. In our experiments, we adapt it to previously unseen systems. Table~\ref{tab:dynamix_systems} reports the 13 held-out chaotic ODEs we evaluated and Table~\ref{tab:dynamix_results_table} shows the corresponding aggregated results.

\begin{table}[t]
\caption{\textbf{Held-out DynaMix benchmark systems.} The table lists the $13$ $3$-dimensional ODE systems we used for the DynaMix in-context KAFFEE adaptation, together with the 
system parameters, numerical integration step size $\Delta t$, transient discarding (burn-in) length, and the reference 
leading Lyapunov exponent $\lambda_1^{\mathrm{ref}}$. The exponents listed here are computed using the ground-truth Jacobians along $10^4$ steps of reference rollouts using QR re-orthonormalization \citep{geist1990comparison,benettin1980lyapunov}. For a complete description, visualization and sources of these systems we refer to \citet{gilpin2021chaos}.}
\centering
\small
\setlength{\tabcolsep}{5pt}
\renewcommand{\arraystretch}{1.0}
\resizebox{\linewidth}{!}{%
\begin{tabular}{@{}llclc@{}}
\toprule
\textbf{System} & \textbf{Parameters} & $\Delta t$ & \textbf{Burn-in} & $\lambda_1^{\mathrm{ref}}$ \\
\midrule
Aizawa & $a{=}0.95,\ b{=}0.7,\ c{=}0.6,\ d{=}3.5,\ e{=}0.25,\ f{=}0.1$ & 0.02 &  1000 & 0.088 \\
Burke--Shaw & $n{=}10,\ e{=}13$ & 0.005 & 1000 & 2.314 \\
Chen & $a{=}35,\ b{=}3,\ c{=}28$ & 0.005  & 2000 & 2.042 \\
Chua & $\alpha{=}15.6,\ \beta{=}28,\ m_0{=}{-}1.143,\ m_1{=}{-}0.714$ & 0.02 & 500 & 0.407 \\
Dadras & $p{=}3,\ q{=}2.7,\ r{=}1.7,\ s{=}2,\ e{=}9$ & 0.02  & 500 & 0.367 \\
Genesio--Tesi & $a{=}0.44,\ b{=}1.1,\ c{=}1$ & 0.02  & 500 & 0.108 \\
Hadley & $a{=}0.2,\ b{=}4,\ f{=}9,\ g{=}1$ & 0.02 & 500 & 0.215 \\
Halvorsen & $a{=}1.4,\ b{=}4$ & 0.01 & 1000 & 0.685 \\
Lorenz--63 & $\sigma{=}10,\ \rho{=}28,\ \beta{=}8/3$ & 0.02 & 500 & 0.902 \\
Nose--Hoover & $a{=}1.5$ & 0.02 & 500 & 0.005 \\
Qi--Chen & $a{=}38,\ b{=}8/3,\ c{=}80$ & 0.005 & 2000 & 3.677 \\
Rabinovich--Fabrikant & $\gamma{=}0.87,\ \alpha{=}1.1$ & 0.01 & 1000 & 0.215 \\
Sprott B & --- & 0.05 & 500 & 0.207 \\
\bottomrule
\end{tabular}
}
\label{tab:dynamix_systems}
\end{table}

For each system, the first $2\cdot 10^3$ noisy observations form the fixed context. Our probabilistic KAFFEE-adaptation only trains on windows drawn inside this context. The scored window length $W_{\mathrm{scored}}$ is chosen per system according to its Lyapunov time:
\[
 W_{\mathrm{scored}} = \mathrm{round}\left(\frac{\tau_\lambda}{\Delta t}\right).
\]
We clip $W_{\mathrm{scored}}$ to a range between $24$ and $128$ steps. The covariance burn-in stays at $L_{\mathrm{burn}}=8$ and the shell initialization at $\alpha=0.9$ for all systems (see Appx.~\ref{app:kaffee_opt_details}). The DynaMix batch loss also averages the innovation NLL over the $N$ observation coordinates.

\textbf{Jacobian construction.} To apply KAFFEE on the shell-off DynaMix core, we require a local transition Jacobian. Here, we approximate those Jacobians while freezing the mixture weights for the experts, which yields an analytical Jacobian again. Concretely, 
we evaluate the gating network at the current latent state under one realization of the gate noise $\bm \varepsilon^{\mathrm{gate}}_t$, and the resulting mixture weights $\bm w_t \in \Delta^{E-1}$ are detached (i.e., treated as locally constant in $\bm z$). The local Jacobian is assembled as the convex combination of the experts' closed-form AL-RNN Jacobians,
\[
\bm J_t
\approx
\sum_{k=1}^{E} w_{t,k}\, \bm J_k(\bm z_{t\mid t})
\in\R^{M\times M},
\qquad
\bm J_k(\bm z) = \nabla_{\bm z} f_k(\bm z) =\bm A_k + \bm W_k\,\bm D_k(\bm c_t),
\]
using the AL-RNN expert transition maps $f_k$, where $\bm A_k$ is diagonal, $\bm W_k$ is full, and $\bm D_k(\bm c_t)=\mathrm{diag}(1,\ldots,1,c_{t,1},\ldots,c_{t,P})$ is the binary switching matrix evaluated at the current latent state (i.e., $c_{t,j}=\mathbf{1}\{\bm z_{t,M-P+j}>0\}$), see Eq.~\ref{app:alrnn-transition}. The mixture weights are held fixed only for covariance propagation, while the gate remains trainable through the predicted mean. The Jacobian approximation thus omits the term
\[
\sum_{k=1}^E f_k(\bm z)\nabla_{\bm z} w_{t,k}(\bm z)^\top.
\]
We use this Jacobian approximation for KAFFEE-DynaMix, the open-loop tied-propagation control, and frozen-core shell calibration. The no-propagation controls instead use the identity to propagate the covariance.

\textbf{KAFFEE-DynaMix training.} The DynaMix ablations shown in Table~\ref{tab:dynamix_results_table} use the same general procedure as the AL-RNN study (see Algorithm~\ref{alg:kaffee_opt}): starting from a teacher-forced-residual shell initialization ($\alpha=0.9$), KAFFEE jointly optimizes the DynaMix core and the process/observation noise under the innovation NLL. Optimization uses Adam ($\beta_1 = 0.9$, $\beta_2 = 0.999$) \citep{kingma2014adam}, learning rate $5\times10^{-5}$, batch size $16$, gradient clipping $1.0$, and proceeds for $50$ iterations on the $2000$-step context. The context observations are held fixed for all windows of a given target system.

\subsubsection{DynaMix evaluation protocol}
\label{app:dynamix_practical_eval}

While the training objectives differ, we evaluate all probabilistic DynaMix baselines using the same protocol with 256 ensemble members. Forecasts start from an encoding of the noisy observation at each origin. As in Appx.~\ref{app:metric_computation}, NLL averages diagonal Gaussian scores over leads $1,\ldots,H$, forecast starts, and observed coordinates. Here, the noisy targets are standardized using the context's mean and standard deviation, and $H=\lfloor{0.25\tau_\lambda/\Delta t\rfloor}$ is clipped to $[24,128]$ steps.

In this experiment we also report the ten-step forecast root-mean-square error
\begin{equation}
\label{eq:rmse_10}
\mathrm{RMSE}_{10} = \mathbb{E}_{t}\left[\,\sqrt{\tfrac{1}{N}\big\|\bar{\bm x}_{t+10\mid t}-\bm x^\star_{t+10}\big\|_2^2}\,\right],
\end{equation}
where $\bar{\bm x}_{t+10\mid t}$ is the ensemble-mean ten-step forecast initialized by encoding the noisy observation at the held-out origin, and $\bm x^\star_{t+10}$ is the corresponding clean reference state; the expectation is taken over different starts on held-out windows. We report $\mathrm{RMSE}_{10}$ in observation coordinates standardized using the noisy context's mean and standard deviation. Shell-off dynamical fidelity metrics ($D_{\mathrm{stsp}}^{\mathrm{core}}$, $D_H$) are computed from rollouts of the adapted core with the external stochastic shell removed, since our evaluation question here is whether the transferred core dynamics can be retained under probabilistic adaptation. Since absolute $D_{\mathrm{stsp}}$ is not comparable across the 13 systems, we compute both the adapted core's and the base model's divergence from the same clean reference trajectory, and form $D_{\mathrm{stsp}}^{\mathrm{core}}/D_{\mathrm{stsp}}^{\mathrm{base}}$ for each system and seed before aggregation.

\subsubsection{DynaMix baseline definitions}
\label{app:dynamix_practical_baselines}

In the main text we report a compact comparison between KAFFEE-DynaMix, frozen-core
shell calibration (``frozen-core''), and the open-loop tied-propagation control (``open-loop''). Table~\ref{tab:dynamix_results_table} expands this comparison to the mechanistic two-axis taxonomy of Fig.~\ref{fig:taxonomy}: \emph{Objective} $\in\{\text{open-loop},\,\text{filtered}\}$ $\times$ \emph{Covariance propagation} $\in\{\text{tied (Jacobian)},\,\text{no propagation (identity)}\}$. Thus, as for the Lorenz-96 ablation (Table~\ref{tab:main_l96_ablation}), all four probabilistic baselines share one common training schedule and the only varying factor is the per-step objective and the covariance recursion. Open-loop adaptation inflates core occupancy relative to the DynaMix base model. Among methods that update the core, the two filtered objectives have similar aggregate held-out NLL, $\mathrm{RMSE}_{10}$, and occupancy; frozen-core remains competitive on NLL while leaving core dynamics untouched by construction.

\textbf{DynaMix base.} The pretrained base model with no adaptation (and no shell). This is the core transition model for the probabilistic adaptation methods.

\textbf{Frozen-core shell calibration.} The diagonal shell is optimized under the EKF innovation NLL (Eq.~\ref{eq:filtered_likelihood}), while the pretrained DynaMix weights are held fixed. This baseline exists to isolate the contribution of post-hoc noise calibration from any learned adaptation of the transition map. The shell-off rollouts for $D_{\mathrm{stsp}}^\mathrm{core}$ and $D_H^\mathrm{core}$ for this control hence match the baseline by construction.

\textbf{Open-loop, tied propagation.} This is the direct DynaMix analogue of the Lorenz-96 ``Open-loop, tied propagation'' control: DynaMix core and shell are updated jointly under a Gaussian open-loop NLL, where at each scored step the predictive mean and covariance are propagated by the EKF predict step (Eq.~\ref{eq:ekf_predict_cov}).

\textbf{Open-loop, no propagation.} Identical to the open-loop tied baseline, except that the predictive covariance recursion is replaced by $\bm\Sigma_{t\mid t_0}=\bm\Sigma_{t-1\mid t_0}+\bm\Gamma_{\theta_s}$, where $t_0$ denotes the fixed conditioning origin (see also the Lorenz-96 analogue, Appx.~\ref{app:alrnn_probabilistic_baselines}).

\textbf{Filtered, tied propagation (KAFFEE-DynaMix).} Shell and core are updated jointly under the EKF innovation NLL (Eq.~\ref{eq:filtered_likelihood}) with predictive covariance propagated using the Jacobian approximation above, according to Algorithm~\ref{alg:kaffee_opt}. This is the direct analogue of the Lorenz-96 KAFFEE method.

\textbf{Filtered, no propagation.} Identical to KAFFEE-DynaMix except that the predict-step covariance recursion is replaced by the state-invariant version; the EKF update is still applied at every step. 

\begin{table}[t]
\caption{\textbf{In-context probabilistic adaptation of DynaMix on 13 held-out ODE systems.} All probabilistic methods share the same in-context adaptation schedule and diagonal noise parametrization, and are initialized from the same pretrained DynaMix core. NLL is the diagonal ensemble-Gaussian score averaged over leads $1,\ldots,H$, where $H$ corresponds to $0.25\tau_\lambda$ clipped to $[24,128]$ steps. \textbf{Bold}: median whose $\pm$ SE interval overlaps the column's best among methods that update the core; \textcolor{red}{red}: open-loop's degradation of autonomous dynamics. Medians $\pm$ robust SEs over $n=13$ systems (per-system seed-medians); see Appx.~\ref{app:dynamix_practical_eval}--\ref{app:dynamix_practical_baselines}.}
\centering
\setlength{\tabcolsep}{2.5pt}
\renewcommand{\arraystretch}{1.0}
\resizebox{\linewidth}{!}{%
\begin{tabular}{@{}l cccc@{}}
\toprule
 & \multicolumn{2}{c}{\textbf{Local Probabilistic Behavior}} & \multicolumn{2}{c}{\textbf{Shell-off Dynamics (Core)}} \\
\cmidrule(lr){2-3} \cmidrule(lr){4-5}
\textbf{Model / Objective} & NLL ($0.25\,\tau_\lambda$) $\downarrow$ & $\mathrm{RMSE}_{10}\downarrow$ & $D_{\mathrm{stsp}}^{\mathrm{core}}/D_{\mathrm{stsp}}^{\mathrm{base}}\downarrow$ & $D_H^{\mathrm{core}}\downarrow$\\
\midrule
\multicolumn{5}{@{}l}{\textbf{\emph{Zero-shot transferred core}}}\\
\hspace{1em}DynaMix base (no adaptation) & -- & -- & $1.000{\pm}0.000$ & $0.399{\pm}0.066$\\
\midrule
\hspace{1em}Frozen-Core Shell Calibration & $0.276{\pm}0.100$ & $0.178{\pm}0.005$ & $1.000{\pm}0.000$ & $0.399{\pm}0.066$ \\
\multicolumn{5}{@{}l}{\textbf{\emph{Open-Loop objective}}}\\
\hspace{1em}No propagation & $0.700{\pm}0.145$ & $0.242{\pm}0.015$ & \textcolor{red}{$1.427{\pm}0.228$} & \textcolor{red}{$0.471{\pm}0.033$} \\
\hspace{1em}Tied propagation & $0.555{\pm}0.100$ & $0.225{\pm}0.012$ & \textcolor{red}{$1.300{\pm}0.108$} & {\bfseries\boldmath $0.398{\pm}0.032$} \\
\multicolumn{5}{@{}l}{\textbf{\emph{Filtered objective}}}\\
\hspace{1em}No propagation & {\bfseries\boldmath $0.255{\pm}0.097$} & {\bfseries\boldmath $0.168{\pm}0.005$} & {\bfseries\boldmath $1.000{\pm}0.014$} & {\bfseries\boldmath $0.349{\pm}0.035$} \\
\hspace{1em}Tied propagation (KAFFEE-DynaMix) & {\bfseries\boldmath $0.254{\pm}0.098$} & {\bfseries\boldmath $0.168{\pm}0.004$} & {\bfseries\boldmath $1.007{\pm}0.013$} & {\bfseries\boldmath $0.382{\pm}0.038$}\\
\bottomrule
\end{tabular}
}
\label{tab:dynamix_results_table}
\end{table}

Table~\ref{tab:dynamix_results_table} shows that open-loop adaptation inflates core occupancy and is worse as a (local) forecast. Open-loop tied has a shell-off spectral distance similar to the zero-shot core, while no propagation increases that distance. Note that the scored training window is one Lyapunov time only when that length falls inside $[24,128]$ steps; the evaluation horizon instead corresponds to a clipped quarter Lyapunov time. Thus, in contrast to the Lorenz-96 experiment, the volume penalty here may distort the attractor without harming its spectral frequencies. The two filtered objectives have similar aggregate NLL, $\mathrm{RMSE}_{10}$, and state-space divergence, and broadly keep the zero-shot core's shell-off quality. Red highlighting in the table marks an open-loop core metric that is worse than the baseline/frozen-core autonomous dynamics.

\subsection{Compute resources}
\label{app:compute_resources}

The following numbers are based on runs using NVIDIA RTX PRO 6000 Blackwell GPUs and are reported as median $\pm$ robust SE. On Lorenz-96 Table~\ref{tab:main_l96_ablation}, one $150$-step shell run (batch size $16$, window $40$), including the per-iteration $10^4$-step validation $D_{\mathrm{stsp}}^{\mathrm{stoch}}$ used for model selection, takes $121\pm 2\,\mathrm{s}$ for KAFFEE and $114$--$119\,\mathrm{s}$ for the four matched controls. Peak allocated CUDA memory for KAFFEE is $0.22\,\mathrm{GiB}$. The training without selection is $9.4\pm 0.2\,\mathrm{s}$ for KAFFEE ($n=20$); one selection score takes $0.682\pm 0.001\,\mathrm{s}$. Sequentially, the table grid is $5$ methods $\times$ $20$ seeds $\approx 3.3$ single-GPU hours (shared AL-RNN core pretraining not included).
For DynaMix, each recorded time is the $50$-step adaptation only (batch size $16$), not the subsequent held-out scoring or $5\cdot10^4$-step core evaluation. Aggregating as in Table~\ref{tab:dynamix_results_table} (median over $n=13$ systems of per-system seed-medians), time is $47\pm 3\,\mathrm{s}$ for KAFFEE-DynaMix, $44\pm 2\,\mathrm{s}$ for open-loop tied, $36\pm 2\,\mathrm{s}$ for filtered no-propagation, $32\pm 2\,\mathrm{s}$ for open-loop no-propagation, and $31\pm 2\,\mathrm{s}$ for frozen-core shell calibration. Peak allocated CUDA memory for KAFFEE-DynaMix is $0.25\pm 0.01\,\mathrm{GiB}$. Sequentially, that adaptation grid is $5$ methods $\times$ $13$ systems $\times$ $20$ seeds $\approx 14$ hours when using a single GPU.

\section{Additional Experiments}
\label{app:appendix_additional_experiments}

This section contains six supporting experiments. First, the stochastic Lorenz-96 ablation is repeated on the 3-dimensional stochastic Lorenz-63 system \citep{lorenz1963deterministic} (Appx.~\ref{app:l63}). Second, we verify that the Lorenz-96 ablation is not tied to the particular forecasting horizon used in Table~\ref{tab:main_l96_ablation} (Appx.~\ref{app:forecast_skill_spectrum}). Third, we isolate whether adaptive/anisotropic Kalman gains are offering deterministic-core reconstruction benefits that a constant forcing strength (or a stop-gradient through the Kalman update) does not (Appx.~\ref{app:l96_mechanism_ablations}). Fourth, we use two non-chaotic systems to isolate the classical noisy-state attenuation mechanism discussed in Appx.~\ref{app:attenuation_diagnostic} (Appx.~\ref{app:nonchaotic_controls}). Fifth, Appx.~\ref{app:rrkaffee} reports performance and empirical scaling exponents for time/memory of dense vs.\ rank-reduced KAFFEE (rr-KAFFEE) on Lorenz-96. Sixth, we evaluate probabilistic adaptation on an EEG recording to assess real-world data feasibility (Appx.~\ref{app:eeg}).

\subsection{Ablation on stochastic Lorenz-63}
\label{app:l63}
Table~\ref{tab:main_l96_ablation} showed a hyperchaotic system (multiple unstable directions), whose deterministic prior was miscalibrated. The Lorenz-63 system has a single positive (i.e., unstable) Lyapunov exponent ($\lambda_1 = 0.905$), and an already faithful deterministic prior (see  Table~\ref{tab:main_l63_ablation}). This ablation therefore tests not only whether the DPC gap signatures persist on an additional system, but also whether they appear when the deterministic prior does not need to be retuned. For this, we keep the taxonomy of Fig.~\ref{fig:taxonomy} and once more use the three-axis evaluation of Section~\ref{subsec:eval-framework}.

\textbf{Experimental details.}
Data generation follows Appx.~\ref{app:l96_dataset}, using the 3-dimensional Lorenz-63 system \citep{lorenz1963deterministic}:
\begin{align*}
\frac{\mathrm d z_1}{\mathrm d t}
&=
\sigma(z_2-z_1),
\\
\frac{\mathrm d z_2}{\mathrm d t}
&=
z_1(\rho-z_3)-z_2,
\\
\frac{\mathrm d z_3}{\mathrm d t}
&=
z_1 z_2-\beta z_3,
\end{align*}
in the standard chaotic regime $\sigma=10$, $\rho=28$, $\beta=8/3$, discretization interval $\Delta t=0.01$, and $\sigma_{\mathrm{obs}}=\sigma_{\mathrm{proc}}=0.1$. As in Table~\ref{tab:main_l96_ablation}, we use $20$ AL-RNNs as deterministic priors; here with $M=30$, $P=10$, and teacher-forcing interval $\tau=16$. Deterministic pretraining otherwise follows Appx.~\ref{app:alrnn_implementation_details}. Shell training uses the same matched baselines and settings as Appx.~\ref{app:alrnn_probabilistic_baselines}, except that the learning rate is cosine-annealed from $10^{-3}$ rather than $5\times10^{-3}$. Deterministic pretraining uses sequences of length $200$. Shell training uses the same length with $8$ covariance burn-in steps. This roughly matches the Lorenz-96 setting in Lyapunov time: $192$ scored steps equals approximately $1.7\tau_\lambda$ (for Lorenz-96 we scored $32$ steps, approximately $ 2\tau_\lambda$). Local NLL and CRPS use the same $\tfrac14\tau_\lambda$ horizon as Appx.~\ref{app:metric_computation}, hence $H=27$ here.  

\begin{figure}[!htbp]
\centering
\includegraphics[width=0.98\textwidth]{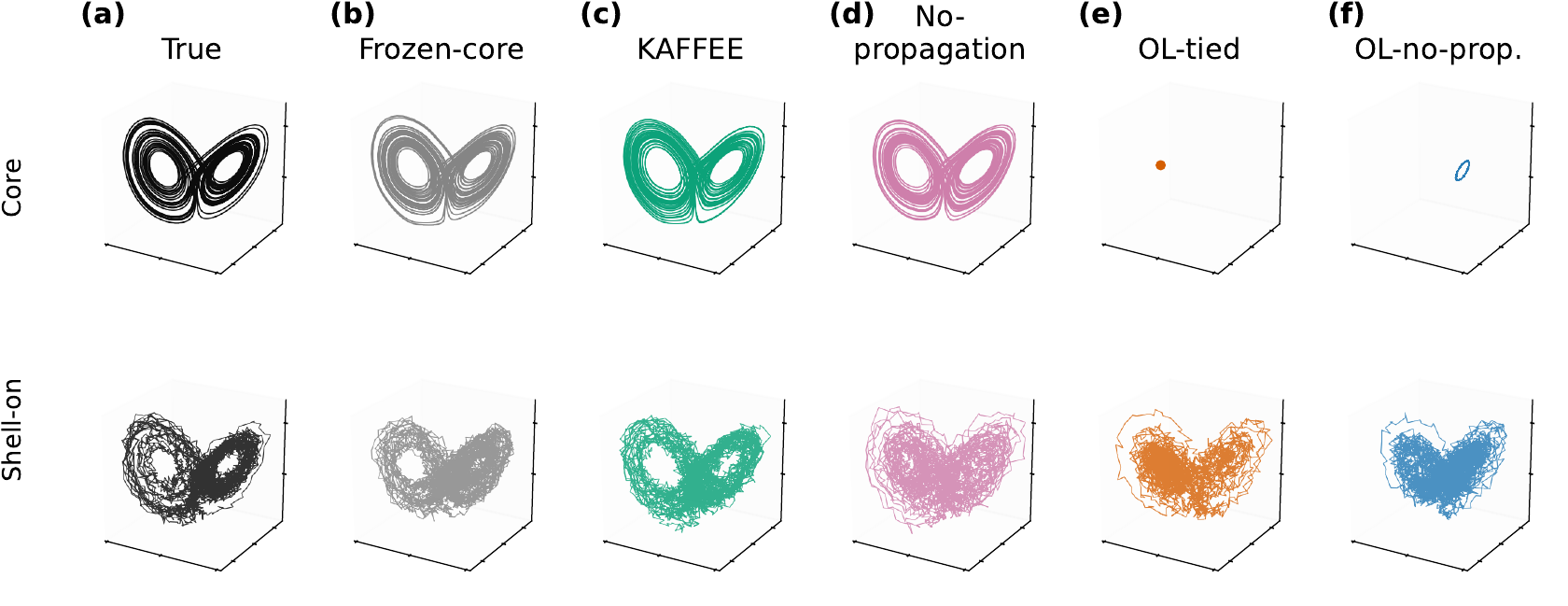}
\caption{\textbf{Shell-off and shell-on rollouts for Lorenz-63.} Shown is one exemplary core seed. \textbf{Top:} Trained autonomous cores under the different objectives/parametrizations. \textbf{(a--d)} Frozen-core, KAFFEE and filtered no-propagation leave the faithful deterministic prior intact and stay close to the true attractor. \textbf{(e--f)} The open-loop objectives both collapse the core. Open-loop tied collapses to a fixed point; open-loop no-propagation collapses to a limit cycle. \textbf{Bottom:} the shell-on rollouts (with process and observation noise) of the same models. \textbf{(a--d)} The filtered objectives produce good overlap with the noisy ground truth, frozen-core noise calibration is closest in agreement. \textbf{(e--f)} Open-loop objectives use the DPC gap's failure modes to ``mask'' their collapsed core; producing a Lorenz-like attractor with their shell-on rollouts.}
\label{fig:l63_attractors}
\end{figure}

\begin{table}[t]
\caption{\textbf{Mechanistic ablation on stochastic Lorenz-63.} All probabilistic baselines share the same training schedule and are initialized from the same AL-RNN cores. Open-loop objectives collapse the attractor; no-propagation remains locally blind to the dynamics-induced uncertainty. \textbf{Bold}: median whose $\pm$SE interval overlaps the column's best; \textcolor{red}{red}: signature failure mode (bracketed label). Medians $\pm$ SEs across $20$ seeds.}
\centering
\setlength{\tabcolsep}{2.2pt}
\renewcommand{\arraystretch}{1.0}
\resizebox{\linewidth}{!}{%
\begin{tabular}{@{}l ccc cc ccc@{}}
\toprule
 & \multicolumn{3}{c}{\textbf{Local Probabilistic Behavior}} & \multicolumn{2}{c}{\textbf{Stochastic  Dynamics (Shell+Core)}} & \multicolumn{3}{c}{\textbf{Deterministic Dynamics (Core)}} \\
\cmidrule(lr){2-4} \cmidrule(lr){5-6} \cmidrule(lr){7-9}
\textbf{Model / Objective} & NLL ($\tfrac14\tau_\lambda$) $\downarrow$ & CRPS ($\tfrac14\tau_\lambda$) $\downarrow$ & $G_{\mathrm{tan}}$ $\uparrow$ & $D_{\mathrm{stsp}}^{\mathrm{stoch}}$ $\downarrow$ & $|\Delta h_\lambda^{\mathrm{stoch}}| \downarrow$ & $D_{\mathrm{stsp}}^{\mathrm{core}}$ $\downarrow$ & $|\Delta h_\lambda^{\mathrm{core}}| \downarrow$ & $D_\mathrm{H}^{\mathrm{core}}$ $\downarrow$ \\
\midrule
\multicolumn{9}{@{}l}{\textbf{\emph{Deterministic prior}}}\\
\hspace{1em}AL-RNN (no shell) & -- & -- & -- & -- & -- & {\bfseries\boldmath $0.362{\pm}0.033$} & {\bfseries\boldmath $0.040{\pm}0.008$} & {\bfseries\boldmath $0.069{\pm}0.005$}\\
\midrule
\hspace{1em}Frozen-Core Shell Calibration & {\bfseries\boldmath $1.771{\pm}0.003$} & {\bfseries\boldmath $0.818{\pm}0.011$} & $0.740{\pm}0.020$ & $0.302{\pm}0.014$ & $0.117{\pm}0.007$ & {\bfseries\boldmath $0.362{\pm}0.033$} & {\bfseries\boldmath $0.040{\pm}0.008$} & {\bfseries\boldmath $0.069{\pm}0.005$}\\
\multicolumn{9}{@{}l}{\textbf{\emph{Open-Loop objective}}}\\
\hspace{1em}No propagation \textcolor{gray}{[Core Collapse]} & $2.227{\pm}0.066$ & $1.794{\pm}0.065$ & \textcolor{red}{$0.000{\pm}0.000$} & $2.968{\pm}0.616$ & $0.380{\pm}0.099$ & \textcolor{red}{$15.080{\pm}1.183$} & \textcolor{red}{$0.916{\pm}0.054$} & \textcolor{red}{$0.683{\pm}0.080$}\\
\hspace{1em}Tied propagation \textcolor{gray}{[Core Collapse]} & {\bfseries\boldmath $2.026{\pm}0.361$} & $1.367{\pm}0.360$ & $0.683{\pm}0.025$ & $1.972{\pm}0.617$ & $0.455{\pm}0.132$ & \textcolor{red}{$15.365{\pm}0.277$} & \textcolor{red}{$0.916{\pm}0.002$} & \textcolor{red}{$0.630{\pm}0.060$}\\
\multicolumn{9}{@{}l}{\textbf{\emph{Filtered objective}}}\\
\hspace{1em}No propagation \textcolor{gray}{[Blind Uncertainty]} & $1.779{\pm}0.002$ & {\bfseries\boldmath $0.828{\pm}0.009$} & \textcolor{red}{$0.000{\pm}0.000$} & $0.282{\pm}0.015$ & {\bfseries\boldmath $0.073{\pm}0.023$} & $0.449{\pm}0.030$ & {\bfseries\boldmath $0.042{\pm}0.005$} & $0.090{\pm}0.004$\\
\hspace{1em}Tied propagation (KAFFEE) & $1.934{\pm}0.025$ & $0.944{\pm}0.027$ & {\bfseries\boldmath $0.776{\pm}0.013$} & {\bfseries\boldmath $0.246{\pm}0.008$} & {\bfseries\boldmath $0.056{\pm}0.010$} & $1.351{\pm}0.134$ & $0.074{\pm}0.012$ & $0.148{\pm}0.017$\\
\bottomrule
\end{tabular}
}
\label{tab:main_l63_ablation}
\end{table}

The DPC signatures appear here as well. Open-loop tied collapses the autonomous dynamics: median $D_{\mathrm{stsp}}^{\mathrm{core}}=15.37$ and $h_\lambda^{\mathrm{core}}\approx 0$, see e.g. Appx.~Fig.~\ref{fig:l63_attractors}(e). With no covariance propagation, open-loop does the same ($15/20$ cores collapse the unstable dynamics) and remains locally blind. However, in this experiment, the open-loop no-propagation baseline did not inflate $\Lambda$, but collapsed the core through the Mahalanobis term driving the updates (hence we call both failures ``core collapse'' here). Appx.~Fig.~\ref{fig:l63_attractors} shows that the noisy shell-on rollouts can indeed mask a collapsed core.

For the filtered innovation objectives, no-propagation retains the deterministic dynamics and exhibits competitive NLL/CRPS, but its state-invariant covariance recursion cannot follow the unstable direction. KAFFEE achieves the highest alignment of local dynamics with uncertainty, $G_{\mathrm{tan}}\approx 0.776$. On this system, the AL-RNN prior is already a strong deterministic core (median $D_{\mathrm{stsp}}^{\mathrm{core}}=0.362$, $|\Delta h_\lambda^{\mathrm{core}}|=0.040$), so frozen-core shell calibration already gives strong local probabilistic behavior and grounded uncertainty. This shows that post-hoc noise calibration can be preferential whenever the core must not change or is already a strong approximation to the deterministic dynamics.

\subsection{Forecast-skill spectrum across horizons}
\label{app:forecast_skill_spectrum}

Table~\ref{tab:main_l96_ablation} reports local predictive scores at $H=\tfrac14\tau_\lambda$. To check that the conclusions do not depend on this single choice, Appx.~Fig.~\ref{fig:l96_forecast_skill_spectrum} evaluates the same
$20$ Lorenz-96 seeds across $H\in\{0.13,0.27,0.60,1.00\}\tau_\lambda$. The NLL is the autoregressive ensemble NLL, CRPS is scored at the endpoint $t+H$, and $G_{\mathrm{tan}}^{(H)}$ measures whether predictive variance growth over the assimilation cycle follows local tangent expansion.

\begin{figure}[!htbp]
  \centering
  \includegraphics[width=\linewidth]{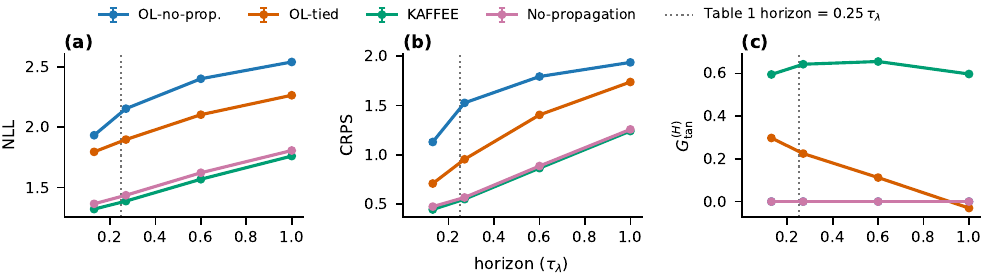}
  \caption{\textbf{Lorenz-96 forecast-skill spectrum.}
   \textbf{(a, b)} Held-out NLL and CRPS worsen for all objectives as the horizon grows. Filtered objectives remain below open-loop ones, and filtered no-propagation stays close to KAFFEE.
  \textbf{(c)} Among the jointly adapted models, only KAFFEE keeps a large positive $G_{\mathrm{tan}}^{(H)}$. Both no-propagation methods stay at zero, and open-loop tied propagation falls toward zero for larger horizons. Markers are medians across $20$ seeds; error bars are robust SEs.}
\label{fig:l96_forecast_skill_spectrum}
\end{figure}

The spectrum over horizons supports the same conclusion as Table~\ref{tab:main_l96_ablation}. 
NLL and CRPS worsen for every objective as the horizon grows, as expected for chaotic dynamics. Filtered objectives remain below the open-loop ones at every horizon. However, KAFFEE is still close to filtered no-propagation in NLL and CRPS. The local grounding metric separates these cases. KAFFEE maintains a large positive $G_{\mathrm{tan}}^{(H)}$ at every horizon $H$, and both no-propagation methods stay at zero. Open-loop tied propagation falls toward zero as the horizon grows. The blind-uncertainty failure is thus not a horizon-specific artifact.

\subsection{Fixed-gain forcing and stop-gradient through the Kalman update}
\label{app:l96_mechanism_ablations}

The Kalman update in KAFFEE both assimilates observations in forward mode and can attenuate amplification in the backward Jacobian product \citep{mikhaeil2022difficulty}, see Appx.~\ref{app:kaffee_gtf_appendix}. Table~\ref{tab:l96_mechanism_ablations} isolates whether KAFFEE's benefit in reconstructing the deterministic core comes from learning through that update in BPTT, rather than from forward assimilation or a fixed forcing strength, using the Lorenz-96 data and setting of Table~\ref{tab:main_l96_ablation}. We implement two different controls: ``Update-Stopgradient'' (Update-SG) keeps the forward EKF identical to KAFFEE at fixed parameters but stops gradients through the Kalman update. Denoting the stop-gradient function by
$\operatorname{sg}(\cdot)$, we propagate
\[
\widetilde{\bm z}_t
=\operatorname{sg}(\bm z_{t\mid t})
+\bm z_{t\mid t-1}-\operatorname{sg}(\bm z_{t\mid t-1}),
\qquad
\widetilde{\bm\Sigma}_t
=\operatorname{sg}(\bm\Sigma_{t\mid t})
+\bm\Sigma_{t\mid t-1}-\operatorname{sg}(\bm\Sigma_{t\mid t-1}).
\]
So the filter still assimilates in the forward pass (and the loss is still the filtered innovation NLL), but backpropagation skips this update and gradients flow only through the filter's predictions. GTF-$\alpha$ replaces the learned Kalman gain by a constant/scalar forcing strength $\bm K_\alpha=\alpha\bm B^\top$, followed by the covariance update. This mimics the fixed forcing strength applied in GTF \citep{hess2023generalized}. We refer to these controls as GTF-$\alpha$ and evaluate $\bm K_\alpha$ for a range of $\alpha$ values. The fixed gains are used during training only. At evaluation, all fitted models use ordinary Kalman assimilation to compute the filtered states and covariances entering $G_{\mathrm{tan}}$; NLL uses the common forecast-ensemble evaluation of Appx.~\ref{app:metric_computation}.

\begin{table}[!htbp]
\caption{\textbf{KAFFEE mechanism ablations on Lorenz-96.} Deterministic dynamics reconstruction performance. Medians $\pm$ robust SE over $20$ seeds. Settings match Table~\ref{tab:main_l96_ablation}. \textbf{Bold}: median whose $\pm$SE interval overlaps the column's best.}
\centering
\setlength{\tabcolsep}{3.5pt}
\renewcommand{\arraystretch}{1.0}
\begin{tabular}{@{}lcccc@{}}
\toprule
\textbf{Method} & NLL $\downarrow$ & $G_{\mathrm{tan}}$ $\uparrow$ & $D_{\mathrm{stsp}}^{\mathrm{core}}$ $\downarrow$ & $|\Delta h_\lambda^{\mathrm{core}}|$ $\downarrow$ \\
\midrule
KAFFEE (Table~\ref{tab:main_l96_ablation}) & {\bfseries\boldmath $1.388{\pm}0.006$} & {\bfseries\boldmath $0.528{\pm}0.003$} & {\bfseries\boldmath $0.291{\pm}0.008$} & {\bfseries\boldmath $0.811{\pm}0.074$} \\
Update-SG & $1.662{\pm}0.017$ & $0.283{\pm}0.029$ & $2.664{\pm}0.337$ & $4.132{\pm}0.320$ \\
\midrule
\multicolumn{5}{@{}l}{\textbf{Fixed gain (GTF-$\alpha$)}}\\
\hspace{1em}$\alpha=0$ & $1.893{\pm}0.008$ & $0.314{\pm}0.015$ & $2.621{\pm}0.011$ & $4.868{\pm}0.000$ \\
\hspace{1em}$\alpha=0.1$ & $1.431{\pm}0.004$ & $0.402{\pm}0.006$ & $1.284{\pm}0.414$ & $4.137{\pm}0.139$ \\
\hspace{1em}$\alpha=0.25$ & {\bfseries\boldmath $1.394{\pm}0.004$} & $0.484{\pm}0.006$ & $0.533{\pm}0.040$ & $2.830{\pm}0.086$ \\
\hspace{1em}$\alpha=0.5$ & $1.405{\pm}0.005$ & {\bfseries\boldmath $0.529{\pm}0.003$} & $0.529{\pm}0.017$ & $2.968{\pm}0.053$ \\
\hspace{1em}$\alpha=1$ & $2.078{\pm}0.013$ & $0.492{\pm}0.006$ & $0.821{\pm}0.104$ & $3.692{\pm}0.166$ \\
\bottomrule
\end{tabular}
\label{tab:l96_mechanism_ablations}
\end{table}

Table~\ref{tab:l96_mechanism_ablations}
shows that models trained with constant scalar forcing can achieve competitive predictive scores and dynamical alignment at selected $\alpha$ values. $\alpha=0.25$ matches KAFFEE on NLL, and $\alpha=0.5$ matches KAFFEE on $G_{\mathrm{tan}}$. Nevertheless, none of the $\alpha$ values we tested beats its (paired-seed) KAFFEE model on $D_{\mathrm{stsp}}^{\mathrm{core}}$ or $|\Delta h_\lambda^{\mathrm{core}}|$ across the 20 seeds. Update-SG also delivers substantially worse core reconstruction under the matched schedule. $\alpha=0$ matches the open-loop-tied (core collapse) failure mode we observed in the main Lorenz-96 experiment (Table~\ref{tab:main_l96_ablation}), because $\bm K_\alpha=0$ means no assimilation takes place.

These controls thus support an important contribution from differentiation through the Kalman update: if gradients through the Kalman update are stopped, autonomous dynamics reconstruction worsens or collapses (while NLL remains only moderately worse). Constant forcing can achieve competitive predictive scores for some $\alpha$ values, but it remains dependent on hyperparameter selection, whereas KAFFEE's use of probabilistically principled assimilation via the Kalman gain provides a data-dependent and anisotropic alternative, that does not require choosing an $\alpha$. 

\subsection{Non-chaotic controls: isolating noisy-state attenuation}
\label{app:nonchaotic_controls}

This work examines DPC mechanisms in the context of chaotic DSR. A related but simpler mechanism is classical noisy-state attenuation: if noisy observations are treated directly as states, then one-step Gaussian prediction estimates an observation-to-observation map rather than the latent transition (i.e., it becomes errors-in-variables regression). Appx.~\ref{app:attenuation_diagnostic} shows this
analytically in the scalar linear-Gaussian case.

The following controls isolate this simpler mechanism. In both experiments, the competing methods use the same transition family and initialization. The difference is only the objective. The \emph{noisy-state predictive objective} (``naive maximum likelihood estimation (MLE)'') scores Gaussian predictions directly on noisy observations, i.e. it treats $\bm x_t$ as if it were the latent state. KAFFEE instead fits the corresponding state-space model by filtering the latent state and scoring innovations. In the linear-Gaussian case, the EKF used by KAFFEE is the standard Kalman filter, so KAFFEE's training objective reduces to the exact Kalman innovation maximum likelihood. The deterministic core of KAFFEE and the noisy-state predictive objectives are not pretrained models here, but initialized randomly by drawing from Gaussian distributions.

\textbf{Stable linear spiral.} We first consider a two-dimensional stable linear spiral observed with additive Gaussian noise. The spectral radius of a transition matrix $\bm A$ is $\rho(\bm A)=\max_i|\operatorname{eig}(\bm A)|$, the largest absolute eigenvalue. The true spiral here has $\rho(\bm A)=0.965<1$, so a smaller learned radius is a more strongly damped autonomous map. As shown in Appx.~Fig.~\ref{fig:nochaos_exp1_main}, the noisy-state predictive objective collapses this radius, while KAFFEE recovers the correct damping. Panel~(a) shows the same effect across multiple observation-noise levels: the noisy-state objective becomes progressively more contractive as noise increases, while KAFFEE remains close to the true radius until the largest noise levels.

\begin{figure}[!htpb]
    \centering
    \includegraphics[width=\linewidth]{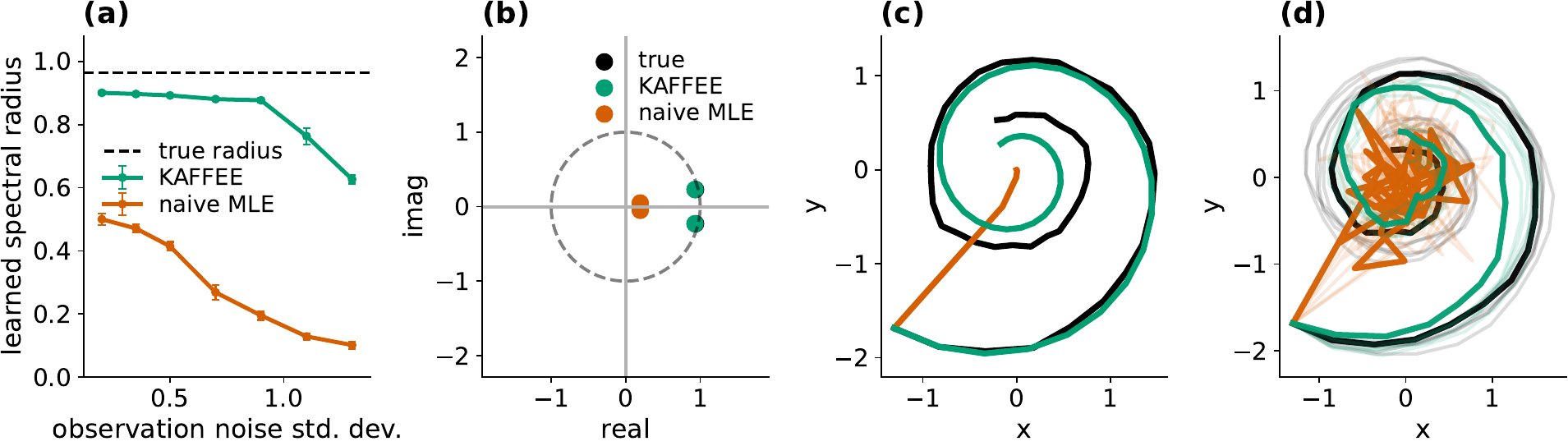}
    \caption{\textbf{Attenuation on a noisy linear spiral.}
    \textbf{(a)} Learned spectral radius $\rho(\hat{\bm A})$ against observation-noise level. Error bars are $95\%$ confidence intervals over $10$ seeds. \textbf{(b)} Eigenvalues of learned Jacobians, \textbf{(c)} deterministic autonomous rollouts from the common initialization, and \textbf{(d)} stochastic state rollouts from that same initialization. Both models share the same linear Gaussian state-space parametrization and initialization. The noisy state predictive objective over-damps the learned autonomous map, while KAFFEE recovers the correct spectral radius and rollout geometry.}
    \label{fig:nochaos_exp1_main}
\end{figure}

\textbf{Van der Pol limit cycle.}
We next use a nonlinear but non-chaotic system: a noisy Van der Pol limit cycle.
The data are generated from
\[
\dot{\bm z}
=
\begin{bmatrix}
v \\
\mu(1-x^2)v-x
\end{bmatrix},
\qquad
\mu=1.6,
\]
with RK4 step size $\Delta t=0.05$, process noise
$\sigma_{\mathrm{proc}}=0.035$, and observation noise
$\sigma_{\mathrm{obs}}=0.25$. Both methods fit the same nonlinear Gaussian state-space model with diagonal process and observation covariances and a cubic polynomial transition function
\[
f_\theta(\bm z) = \bm A \bm z + \bm W \bm \phi(\bm z) + \bm h,
\qquad
\bm z = [x, v]^\top,
\]
where
\[
\bm \phi([x,v]^\top)
=
\bigl[x^2, xv, v^2, x^3, x^2 v, x v^2, v^3\bigr]^\top.
\]

Appx.~Fig.~\ref{fig:nochaos_exp2_rollout} shows the same qualitative failure in a
nonlinear setting. The noisy-state predictive NLL learns an overcontractive
autonomous map: its shell-off rollout does not return to the limit cycle and
instead collapses to a spurious fixed point. The stochastic shell-on rollout can
partly ``mask'' this deterministic failure, but it does not repair the learned
core. KAFFEE preserves the transient and returns to the correct orbit, with local Jacobian spectra closer to those of the true flow map.

\begin{figure}[!htbp]
    \centering
    \includegraphics[width=\linewidth]{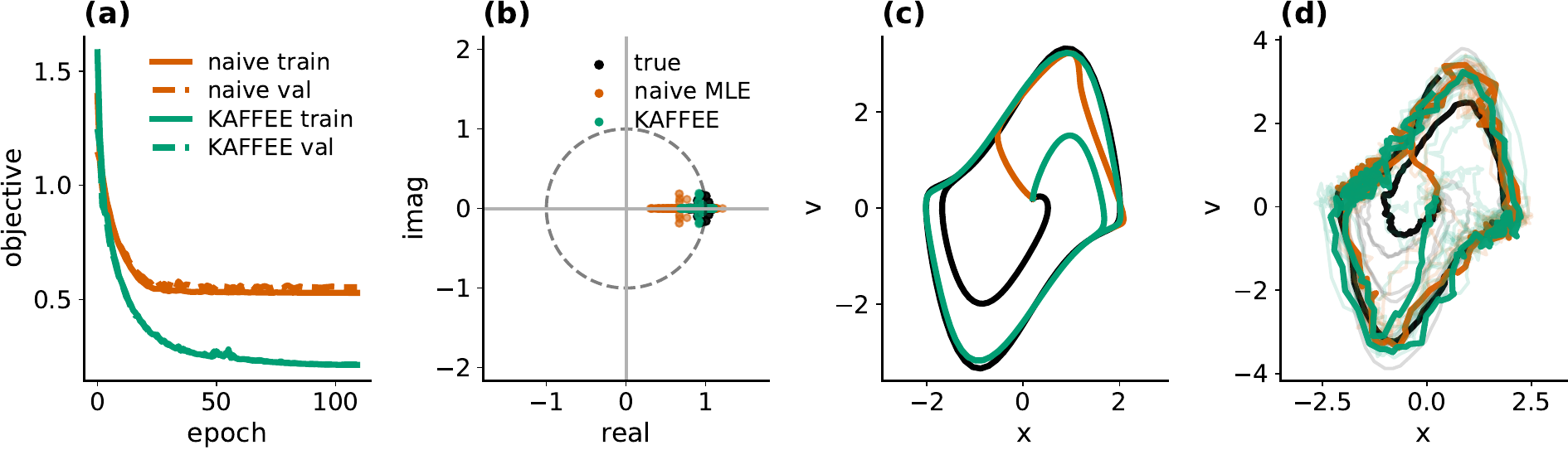}
        \caption{\textbf{Noisy Van der Pol.} \textbf{(a)} Optimization traces, \textbf{(b)} Jacobian eigenvalues along a deterministic reference rollout, \textbf{(c)} deterministic autonomous rollouts from the common initialization, and \textbf{(d)} stochastic state rollouts from that same initialization. Both methods use the same cubic polynomial transition family. The noisy-state predictive objective learns an overcontractive autonomous map that collapses to a spurious fixed point, while the stochastic rollout masks this failure. KAFFEE preserves the autonomous core by fitting the latent state-space model through scoring assimilated innovations.}
    \label{fig:nochaos_exp2_rollout}
\end{figure}

Appx.~Fig.~\ref{fig:nochaos_exp2_var20} shows a covariance propagation diagnostic for this example. At each state we initialize the same isotropic reference covariance and propagate it for one or twenty steps under either the true flow Jacobian or the learned Jacobian products. Additive process noise is omitted in this diagnostic, so darker colors correspond to stronger contraction of the propagated covariance. The noisy-state objective is systematically more contractive than the reference, whereas KAFFEE better matches the reference covariance-propagation pattern.

\begin{figure}[!htbp]
	\centering
	\includegraphics[width=\linewidth]{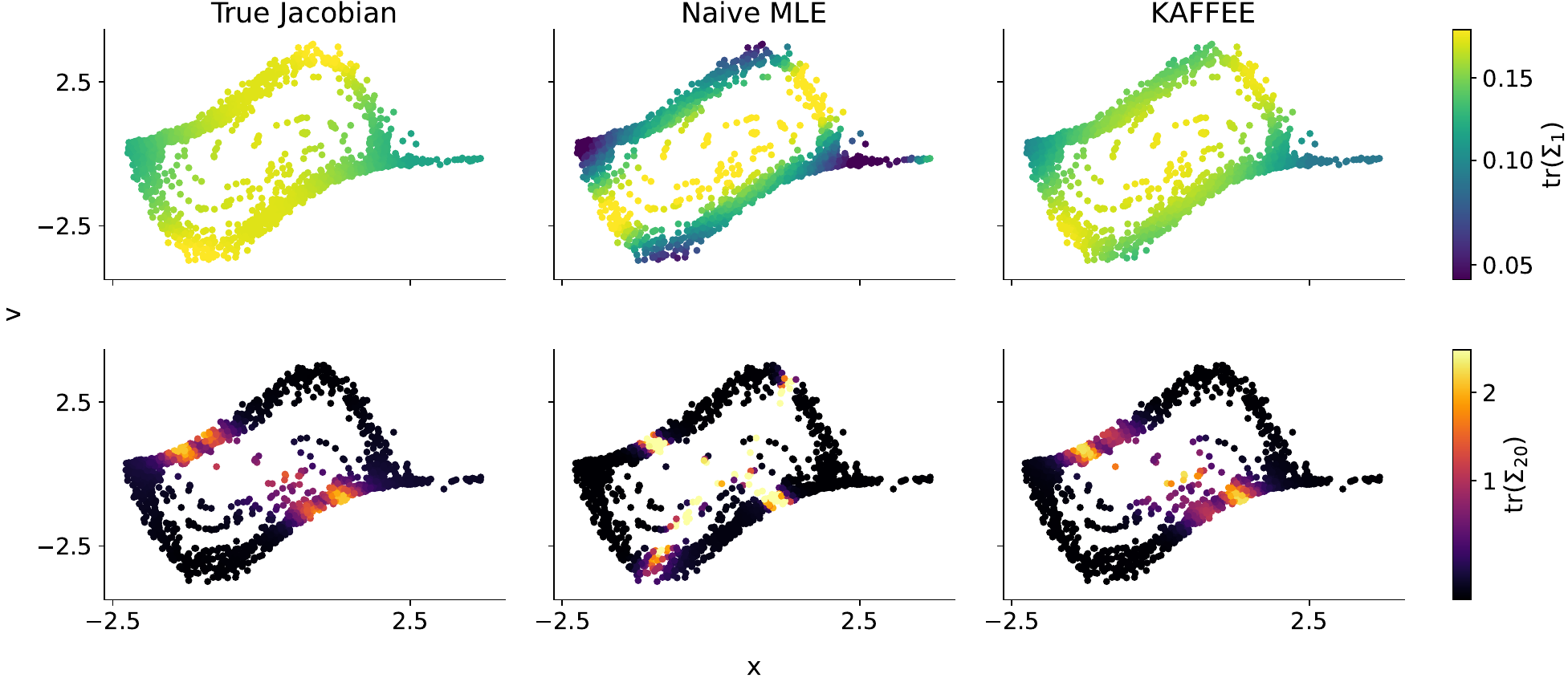}
	\caption{\textbf{Van der Pol covariance propagation map.} \textbf{Top row:} One-step propagated reference-covariance trace over the same sampled state cloud: left for the true flow Jacobian, center for the noisy-state predictive objective, right for KAFFEE.
    \textbf{Bottom row:} Analogous twenty-step covariance propagation map with the same ordering of panels. Overall, the noisy-state predictive objective yields damped/contracted uncertainty propagation, while KAFFEE recovers a dynamically grounded covariance propagation.}
	\label{fig:nochaos_exp2_var20}
\end{figure}

Together, these controls isolate the simple failure mode known as attenuation bias \citep{fuller1987measurement,staudenmayer2005measurement}, and relate it to the DPC perspective: a finite-horizon predictive objective can improve local fit by contracting the learned transition when noisy observations are used as states. Scoring innovations removes this shortcut in the correctly specified linear-Gaussian case and remains a useful inductive bias for the nonlinear example.

\subsection{Rank-reduced KAFFEE}
\label{app:rrkaffee}

In Appx.~\ref{app:kaffee_scaling_appendix}, we identify dense covariance propagation and its cubic scaling in latents as a computational bottleneck of the EKF training of KAFFEE. This subsection instantiates and evaluates a structured low-rank approximation on the Lorenz-96 experiment.

In Table~\ref{tab:main_l96_ablation}, we use a dense $M\times M$ latent covariance with $M=100$. rr-KAFFEE replaces that covariance by the rank-reduced Kalman filter (RRKF) of \citet{schmidt2023rankreduced}: a rank-$r$ square-root factor obtained by truncated SVD after Jacobian propagation. Concretely, rr-KAFFEE represents the filtering covariance as $\bm \Sigma_{t\mid t} = \bm V_t\bm V_t^\top$, with $\bm V_t \in \mathbb{R}^{M\times r}$. It additionally learns a low-rank process-noise factor $\bm S\in \mathbb{R}^{M \times r}$, such that $\bm \Gamma = \bm S\bm S^\top$, while the observation noise $\bm \Lambda$ remains diagonal. Prediction then truncates to keep only the leading $r$ singular components of $[\bm J_t \bm V_{t-1}, \bm S]$. In this sense, rr-KAFFEE is a rank-reduced variant of KAFFEE, but it changes both the filtering representation and the process-noise parametrization (low-rank vs.\ ``only'' diagonal). For training, we differentiate through the RRKF's square-root operations and truncated SVDs; training and evaluation otherwise follow the same Lorenz-96 shell protocol (see Appx.~\ref{app:alrnn_probabilistic_baselines}).

Rank $r=20$ matches the observation dimension $N$ and the true process-noise rank of the studied Lorenz-96 system. Appx.~Fig.~\ref{fig:rrkf_rank_metrics} shows that this rank already recovers dense-like quality; larger ranks remain competitive or slightly improve several local scores.

Next we study the scaling behavior of rr-KAFFEE with latent dimension. Appx.~Fig.~\ref{fig:rrkf_scaling} reports training time (1 iteration) and peak allocated CUDA memory vs.\ latent dimension $M$ (to isolate the asymptotic scaling past the computational-overhead regime, we used batch size $1$, window $10$, $r=20$, and float32 precision). Across the four largest latent dimensions tested successfully before GPU memory was exhausted, wall-clock time scales as $M^{2.85}$ for dense KAFFEE vs.\ $M^{1.89}$ for rr-KAFFEE, so cubic vs.\ quadratic in $M$. Peak allocated memory remains approximately quadratic for both, consistent with the dense core weights and their gradients and optimizer state. Dense KAFFEE last completes at $M=1.6\cdot10^4$ and runs out of memory at $M=2.0\cdot10^4$; rr-KAFFEE last completes at $M=5.6\cdot10^4$ and runs out of memory at $M=6.4\cdot10^4$. 

These results show that rr-KAFFEE achieves approximately quadratic scaling of training time with latent dimension, and substantially larger feasible latent dimensions before running out of memory. Additionally, for the Lorenz-96 experiment we studied, rr-KAFFEE retains competitive reconstruction and predictive performance when the rank is chosen appropriately. A principled choice of the rank (e.g., the rank of the process noise covariance) could be a direction for future work.

\begin{figure}[!htbp]
  \centering
  \includegraphics[width=\linewidth]{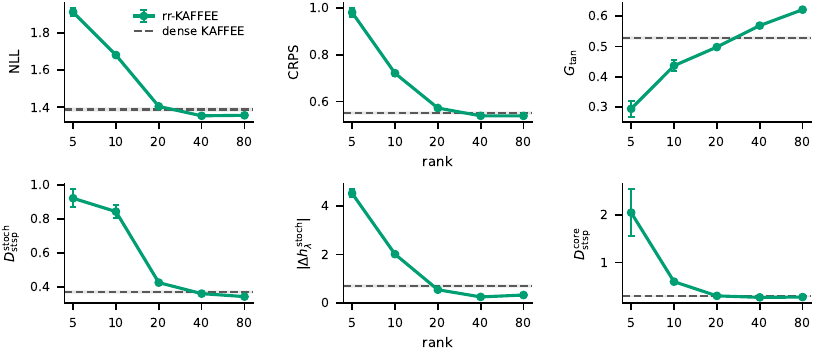}
  \caption{\textbf{rr-KAFFEE quality vs.\ filter rank on Lorenz-96.}
  Markers are seed medians; dashed horizontal lines are dense Table~\ref{tab:main_l96_ablation} KAFFEE. Rank $r=20$ already recovers dense-like scores; $r=40$ and $r=80$ remain competitive.}
  \label{fig:rrkf_rank_metrics}
\end{figure}

\begin{figure}[!htbp]
  \centering
  \includegraphics[width=\linewidth]{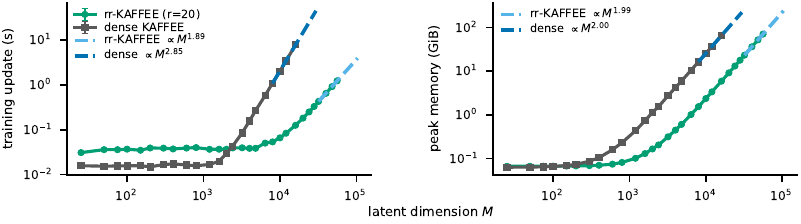}
  \caption{\textbf{Dense KAFFEE vs.\ rr-KAFFEE compute scaling.}
  Training time (1 iteration) (left) and peak allocated CUDA memory (right) vs.\ latent dimension $M$ on log-log scale. Dashed lines are power-law fits on the last four completed latent sizes, respectively. Dense KAFFEE gives $ M^{2.85}$ (time) and $M^{2}$ (memory); rr-KAFFEE ($r=20$) gives $M^{1.89}$ (time) and $M^{1.99}$ (memory).}
  \label{fig:rrkf_scaling}
\end{figure}

\subsection{Probabilistic adaptation on EEG data}
\label{app:eeg}

To evaluate probabilistic adaptation beyond synthetic dynamical systems, we apply the models/objectives of Table~\ref{tab:main_l96_ablation} to an electroencephalogram (EEG) recording. We use an eyes-open, task-free (resting-state) recording from the PhysioNet EEG Motor Movement/Imagery dataset (subject/recording S001R01) \citep{schalk2009eeg,schalk2004bci,pollard2026physionet}, distributed under the Open Data Commons Attribution License v1.0, which was studied in multiple prior DSR works \citep{mikhaeil2022difficulty,brenner2022tractable,hess2023generalized,pals2024inferring}. Starting from the same pretrained deterministic cores for each objective, we compare joint adaptations of the core and stochastic shell with frozen-core shell calibration. This tests whether adapting the dynamics improves probabilistic fit and long-run reconstruction of the observed recording beyond post-hoc noise calibration and beyond deterministic core training alone. We then repeat the comparison on recordings of four additional subjects of the same dataset.

\textbf{Experimental details.}
Here, the deterministic prior used is a clipped shallow piecewise-linear RNN (shPLRNN) \citep{hess2023generalized}. It parametrizes the latent transition of $\bm z_t \in \mathbb{R}^M$ as
\begin{equation}
\label{eq:eeg_shplrnn}
F_{\theta_c}(\bm z_t)
=\bm A\bm z_t
+\bm W_1\big[\phi(\bm W_2\bm z_t+\bm h_2)-\phi(\bm W_2\bm z_t)\big]
+\bm h_1,
\end{equation}
where $\phi$ is the element-wise applied ReLU function (yielding piecewise-linear dynamics), $\bm A\in\mathbb R^{M\times M}$ is diagonal, $\bm W_1\in\mathbb R^{M\times L}$, $\bm W_2\in\mathbb R^{L\times M}$, $\bm h_1\in\mathbb R^M$, and $\bm h_2\in\mathbb R^L$. Thus, unlike the AL-RNN (Appx.~\ref{app:alrnn_implementation_details}), the shPLRNN contains a hidden layer of width $L$, separate from the recurrent/latent state dimension $M$. This allows low-dimensional latent dynamics while keeping an expressive (wide) nonlinear transition map, hence the transition Jacobian and the propagated latent covariance are still $M\times M$, where $L \gg M$. Following the EEG settings of \citet{hess2023generalized}, we use $M=16$ and $L=512$ in our experiments. Further, we use a linear observation model $\bm{x}_t = \bm B\bm z_t$, $\bm B\in\mathbb R^{N\times M}$ that is learned during the respective core's pretraining and held fixed during probabilistic adaptation.

\begin{figure}[t]
  \centering
  \includegraphics[width=\linewidth]{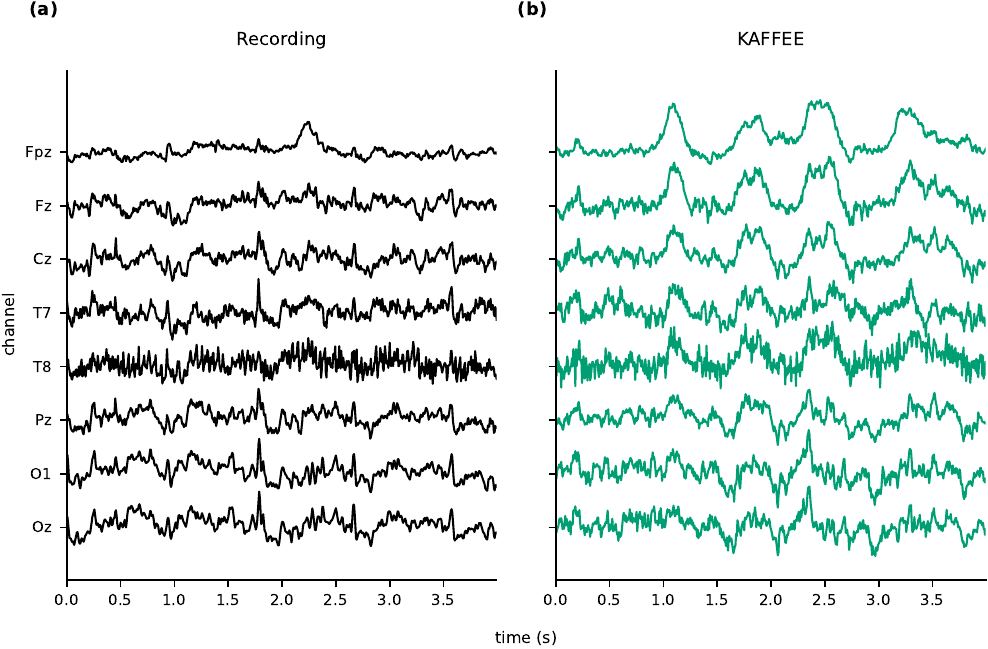}
  \caption{\textbf{Example EEG reconstruction.}
  \textbf{(a)} Eight channels from the (standardized) PhysioNet EEG recording S001R01 \citep{schalk2009eeg}.
  \textbf{(b)} The corresponding channels from a freely generated stochastic (shell-on) KAFFEE rollout after burn-in. The two segments (recording/model-generated) are not aligned in time.}
  \label{fig:eeg_example}
\end{figure}

Each EEG recording contains $N=64$ channels sampled at $160\,\mathrm{Hz}$. To match the length of the EEG time series studied in \citet{hess2023generalized}, we use the first $T =9{,}640$ samples of each recording. Each channel is standardized using its mean and standard deviation over this entire time series. The same series is used for training and evaluation. We train $20$ independent cores per recording using GTF with $\alpha=0.1$ for $5\cdot 10^3$ epochs, and select the model at the last epoch. Core training uses RAdam with an exponentially decaying learning rate from $10^{-3}$ to $10^{-6}$, $50$ batches per epoch, batch size $16$, and $50$-step windows. These settings are the same that were used in \citet{hess2023generalized} on recording S001R01. However, in that work, the authors trained and evaluated the models on the standardized and smoothed recording (using a 15-bin Hann window). Our preprocessing and evaluation differs from that procedure: since our goal is probabilistic DSR, we retain those fluctuations in the data and train the deterministic cores on unsmoothed (but standardized) data, and the trained models are scored against the unsmoothed recording.

Each core initializes all five probabilistic objectives of Appx.~\ref{app:alrnn_probabilistic_baselines}. Probabilistic adaptation then uses Adam for $900$ updates with batch size $32$ and $40$-step windows ($8$ burn-in and $32$ scored steps), with gradient clipping at $2$, and learning rate $10^{-3}$, which was cosine-annealed. Models are selected by the lowest state-space divergence of shell-on rollouts against the same recording, evaluated after update $1$ and every $25$th update.

\begin{table}[!b]
\caption{\textbf{Probabilistic adaptation on EEG.} Medians $\pm$ SEs across $20$ paired core seeds for recording S001R01. \textbf{Bold:} median whose $\pm$SE interval overlaps the column's best.}
\centering
\setlength{\tabcolsep}{3pt}
\renewcommand{\arraystretch}{1.0}
\resizebox{\linewidth}{!}{%
\begin{tabular}{@{}lcccc@{}}
\toprule
\textbf{Model / Objective} & NLL $\downarrow$ & CRPS $\downarrow$ & $D_{\mathrm{stsp}}^{\mathrm{stoch}}$ $\downarrow$ & $D_{\mathrm{H}}^{\mathrm{stoch}}$ $\downarrow$ \\
\midrule
\multicolumn{5}{@{}l}{\textbf{\emph{Deterministic prior}}}\\
\hspace{1em} shPLRNN & -- & -- & $4.73{\pm}7.69$ & $0.217{\pm}0.083$ \\
\midrule
\hspace{1em}Frozen-Core Shell Calibration & $-0.001{\pm}0.007$ & $0.1911{\pm}0.0010$ & $1.58{\pm}0.06$ & $0.218{\pm}0.009$ \\
\multicolumn{5}{@{}l}{\textbf{\emph{Open-Loop objective}}}\\
\hspace{1em}No propagation & $0.097{\pm}0.006$ & $0.2256{\pm}0.0032$ & $3.41{\pm}0.27$ & $0.220{\pm}0.009$ \\
\hspace{1em}Tied propagation & $0.042{\pm}0.003$ & $0.2327{\pm}0.0029$ & $1.91{\pm}0.06$ & $0.190{\pm}0.009$ \\
\multicolumn{5}{@{}l}{\textbf{\emph{Filtered objective}}}\\
\hspace{1em}No propagation & $-0.027{\pm}0.004$ & {\bfseries\boldmath $0.1726{\pm}0.0008$} & $1.61{\pm}0.03$ & {\bfseries\boldmath $0.108{\pm}0.004$} \\
\hspace{1em}Tied propagation (KAFFEE) & {\bfseries\boldmath $-0.042{\pm}0.004$} & $0.1781{\pm}0.0007$ & {\bfseries\boldmath $1.39{\pm}0.03$} & {\bfseries\boldmath $0.114{\pm}0.005$} \\
\bottomrule
\end{tabular}
}
\label{tab:eeg_s001}
\end{table}

For local probabilistic behavior, we evaluate all probabilistic models on one-step predictions ($H=1$) from the filtered history, using each model's covariance-propagation rule. After a burn-in phase of $14$ conditioning observations, each new observation is scored before assimilation, yielding $9{,}626$ scored predictions per recording. NLL then scores the joint Gaussian prediction with its full covariance $\bm\Omega_t$ (Eq.~\ref{eq:innovation}), averaged over time and normalized by the number of channels. CRPS is computed in closed form from the corresponding Gaussian marginals and averaged over time and
channels. Note that this differs from the CRPS used e.g. for the Lorenz-96 experiments in Appx.~\ref{app:metric_computation}, where we instead compute CRPS directly from forecast ensembles.

For stochastic dynamical fidelity, we compare freely generated shell-on trajectories with the unsmoothed recording. Again following \citet{hess2023generalized}, we score $T = 9{,}640$ generated observations after letting the model run for $0.25T$ from the first data point. State-space divergence $D_{\mathrm{stsp}}^{\mathrm{stoch}}$ and the mean Hellinger distance between the per-channel power spectra, $D_{\mathrm{H}}^{\mathrm{stoch}}$, are computed as described in Appx.~\ref{app:metric_computation}. Fig.~\ref{fig:eeg_example} illustrates eight channels from the recording alongside the corresponding channels of a freely generated stochastic KAFFEE rollout.

KAFFEE improves all four reported scores relative to frozen-core shell calibration, and both dynamic fidelity metrics relative to the deterministic prior (Table~\ref{tab:eeg_s001}). Calibrating the shell alone reduces $D_{\mathrm{stsp}}^{\mathrm{stoch}}$ from $4.73$ to $1.58$, but leaves the spectral/temporal mismatch $D_{\mathrm{H}}^{\mathrm{stoch}}$ nearly unchanged. Joint adaptation under either filtered objective yields substantially better spectral agreement, together with improved NLL and CRPS. KAFFEE also improves state-space divergence relative to frozen-core calibration. KAFFEE achieves the lowest NLL and $D_{\mathrm{stsp}}^{\mathrm{stoch}}$, while filtered no-propagation has the lowest CRPS. Both filtered objectives outperform the open-loop controls on all four reported metrics.

\textbf{Additional subjects.}
To assess whether these results depend on the single recording S001R01, we repeat the experiment on resting-state recordings for four additional subjects (recordings S002R01--S005R01). The preprocessing, training/model settings, and evaluation protocol remain the same. Table~\ref{tab:eeg_additional_subjects} reports the comparison with frozen-core shell calibration and filtered no-propagation.

\begin{table}[!htbp]
\caption{\textbf{Probabilistic adaptation on four additional EEG recordings.}
Settings match Table~\ref{tab:eeg_s001}. Shown are medians $\pm$ SEs across $20$ paired core seeds per recording. \textbf{Bold:} median whose $\pm$SE interval overlaps the column's best within each recording.
}
\centering
\setlength{\tabcolsep}{3pt}
\renewcommand{\arraystretch}{1.0}
\resizebox{\linewidth}{!}{%
\begin{tabular}{@{}llcccc@{}}
\toprule
\textbf{Subject} & \textbf{Method} & NLL $\downarrow$ & CRPS $\downarrow$ & $D_{\mathrm{stsp}}^{\mathrm{stoch}}$ $\downarrow$ & $D_\mathrm{H}^{\mathrm{stoch}}$ $\downarrow$ \\
\midrule
\multirow{3}{*}{S002} & Frozen-core shell & $0.443{\pm}0.004$ & $0.3222{\pm}0.0018$ & $3.20{\pm}0.05$ & {\bfseries\boldmath $0.161{\pm}0.003$} \\
 & Filtered no-propagation & $0.400{\pm}0.003$ & {\bfseries\boldmath $0.2652{\pm}0.0006$} & $2.77{\pm}0.05$ & {\bfseries\boldmath $0.159{\pm}0.003$} \\
 & KAFFEE & {\bfseries\boldmath $0.373{\pm}0.002$} & $0.2757{\pm}0.0004$ & {\bfseries\boldmath $2.51{\pm}0.02$} & $0.185{\pm}0.005$ \\
\midrule
\multirow{3}{*}{S003} & Frozen-core shell & $0.123{\pm}0.002$ & $0.2080{\pm}0.0013$ & $1.97{\pm}0.04$ & $0.260{\pm}0.004$ \\
 & Filtered no-propagation & $0.107{\pm}0.003$ & {\bfseries\boldmath $0.1772{\pm}0.0008$} & $2.13{\pm}0.04$ & {\bfseries\boldmath $0.192{\pm}0.005$} \\
 & KAFFEE & {\bfseries\boldmath $0.087{\pm}0.002$} & $0.1880{\pm}0.0005$ & {\bfseries\boldmath $1.80{\pm}0.03$} & {\bfseries\boldmath $0.187{\pm}0.007$} \\
\midrule
\multirow{3}{*}{S004} & Frozen-core shell & $0.393{\pm}0.003$ & $0.2539{\pm}0.0012$ & {\bfseries\boldmath $4.54{\pm}0.04$} & $0.246{\pm}0.004$ \\
 & Filtered no-propagation & {\bfseries\boldmath $0.371{\pm}0.003$} & {\bfseries\boldmath $0.2274{\pm}0.0006$} & $4.81{\pm}0.05$ & {\bfseries\boldmath $0.217{\pm}0.006$} \\
 & KAFFEE & {\bfseries\boldmath $0.376{\pm}0.006$} & $0.2386{\pm}0.0005$ & {\bfseries\boldmath $4.46{\pm}0.04$} & {\bfseries\boldmath $0.212{\pm}0.006$} \\
\midrule
\multirow{3}{*}{S005} & Frozen-core shell & $0.212{\pm}0.003$ & $0.1978{\pm}0.0001$ & $2.33{\pm}0.04$ & $0.178{\pm}0.002$ \\
 & Filtered no-propagation & $0.176{\pm}0.003$ & {\bfseries\boldmath $0.1907{\pm}0.0003$} & $2.18{\pm}0.02$ & $0.171{\pm}0.003$ \\
 & KAFFEE & {\bfseries\boldmath $0.167{\pm}0.002$} & $0.1937{\pm}0.0001$ & {\bfseries\boldmath $2.05{\pm}0.02$} & {\bfseries\boldmath $0.164{\pm}0.002$} \\
\bottomrule
\end{tabular}
}
\label{tab:eeg_additional_subjects}
\end{table}

Across all five recordings, KAFFEE has the lowest median $D_{\mathrm{stsp}}^\mathrm{stoch}$ among the probabilistic objectives, whereas filtered no-propagation has the lowest median CRPS. On S004, their NLL intervals overlap. In terms of temporal/spectral reconstruction, filtered no-propagation has the lowest median $D_\mathrm{H}$ on S001 and S002, while KAFFEE has the lowest on S003--S005. On S002, frozen-core calibration also has a lower $D_\mathrm{H}$ than KAFFEE.

Together, these results support the feasibility of probabilistically adapting deterministic surrogates to real-world data. However, since the true latent dynamics are unknown, an improved fit to the observed recordings cannot establish recovery of the underlying deterministic dynamics, nor can we directly verify the DPC gap failure modes (core collapse, noise masking, or blind uncertainty). Proxies for metrics that can reliably assess these properties without oracle knowledge warrant further research. The mechanistic evidence for the DPC gap failure modes thus remains in our controlled/synthetic Lorenz-63, Lorenz-96, and DynaMix experiments.

\clearpage

\begin{enumerate}

\item {\bf Claims}
    \item[] Question: Do the main claims made in the abstract and introduction accurately reflect the paper's contributions and scope?
    \item[] Answer:
    \answerYes{}
    \item[] Justification: We state our main claims in the abstract and the introduction. We validate them through the mechanistic Lorenz-96 ablation and DynaMix adaptation experiments, and through additional supporting results in the appendix.

\item {\bf Limitations}
    \item[] Question: Does the paper discuss the limitations of the work performed by the authors?
    \item[] Answer:
    \answerYes{}
    \item[] Justification: We summarize limitations in Section~\ref{sec:discussion}, and discuss each of the mentioned limitations in more detail in the respective appendix sections.

\item {\bf Theory assumptions and proofs}
    \item[] Question: For each theoretical result, does the paper provide the full set of assumptions and a complete (and correct) proof?
    \item[] Answer: \answerYes{}
    \item[] Justification: We state the assumptions underlying the DPC-gap definition and the covariance-volume mechanism in the main text, and provide formal derivations and proofs in the appendix, including for the open-loop volume-pressure argument, EKF equations, and anti-contraction analysis of the NLL. Informal intuition is given in the main text and detailed derivations alongside illustrations are given in the appendix.

    \item {\bf Experimental result reproducibility}
    \item[] Question: Does the paper fully disclose all the information needed to reproduce the main experimental results of the paper to the extent that it affects the main claims and/or conclusions of the paper (regardless of whether the code and data are provided or not)?
    \item[] Answer: \answerYes{}
    \item[] Justification: The main text describes the experimental settings, baselines, evaluation axes, metrics, noise levels, and number of seeds. Full implementation details, optimization schedules/algorithm, metric definitions, hyperparameters, and dataset-generation procedures are provided in the appendix.

\item {\bf Open access to data and code}
    \item[] Question: Does the paper provide open access to the data and code, with sufficient instructions to faithfully reproduce the main experimental results, as described in supplemental material?
    \item[] Answer: \answerYes{} 
    \item[] Justification: We provide code that contains implementations of our experiments, the data that was used, and the trained models at \url{https://github.com/aherz/kaffee-dpc}.

\item {\bf Experimental setting/details}
    \item[] Question: Does the paper specify all the training and test details (e.g., data splits, hyperparameters, how they were chosen, type of optimizer) necessary to understand the results?
    \item[] Answer: \answerYes{}
    \item[] Justification: We specify all training and test details necessary to understand and reproduce the results. The code we provide contains configuration files for each experiment that also specifies the used settings.

\item {\bf Experiment statistical significance}
    \item[] Question: Does the paper report error bars suitably and correctly defined or other appropriate information about the statistical significance of the experiments?
    \item[] Answer: \answerYes{}
    \item[] Justification: The main figures and tables report medians with robust standard errors over 20 random seeds. We specify the number of seeds and the type of uncertainty reported. These error bars are shown for the experiments supporting the main claims. The figure captions and metric section in the appendix specify how errors were aggregated and the sample counts: for the Lorenz-63/96 and the EEG experiments, we use individual training seeds as statistical unit. For DynaMix, we aggregate individual per-system seeds over the 13 tested systems. The resulting error bars describe the corresponding dispersion, they are not formal hypothesis tests of statistical equivalence.

\item {\bf Experiments compute resources}
    \item[] Question: For each experiment, does the paper provide sufficient information on the computer resources (type of compute workers, memory, time of execution) needed to reproduce the experiments?
    \item[] Answer: \answerYes{}
    \item[] Justification:  We discuss scaling in Appx.~\ref{app:kaffee_scaling_appendix} and report measured GPU times and peak memory allocation for the main experiments in Appx.~\ref{app:compute_resources}. Those measurements were taken using NVIDIA RTX PRO 6000 Blackwell GPUs. They exclude preliminary and exploratory runs that did not enter the paper.
    
\item {\bf Code of ethics}
    \item[] Question: Does the research conducted in the paper conform, in every respect, with the NeurIPS Code of Ethics \url{https://neurips.cc/public/EthicsGuidelines}?
    \item[] Answer: \answerYes{}
    \item[] Justification: Our research is methodological and uses synthetic dynamical systems, previously published model assets, and secondary analysis of existing public EEG recordings (from the ``PhysioNet'' dataset \citep{schalk2009eeg}; see Appx.~\ref{app:eeg}). We reviewed the NeurIPS Code of Ethics and our work conforms to it.

\item {\bf Broader impacts}
    \item[] Question: Does the paper discuss both potential positive societal impacts and negative societal impacts of the work performed?
    \item[] Answer: \answerYes{}
    \item[] Justification: Our work is foundational/methodological research on probabilistic modeling of dynamical systems. We briefly discuss potential positive and negative impacts in the paper. In summary, potential positive impacts include improved uncertainty-aware modeling for scientific forecasting, control, active data collection/experimental design, and discovery. Potential negative impacts are indirect and dependent on the application: inaccurate probabilistic surrogates could lead to overconfident downstream decisions in scientific or engineering applications. The paper discusses the need for careful evaluation of dynamical systems models to mitigate such risks.
    
\item {\bf Safeguards}
    \item[] Question: Does the paper describe safeguards that have been put in place for responsible release of data or models that have a high risk for misuse (e.g., pre-trained language models, image generators, or scraped datasets)?
    \item[] Answer: \answerNA{}
    \item[] Justification: We do not release such high-risk models or other assets with a substantial risk of misuse.

\item {\bf Licenses for existing assets}
    \item[] Question: Are the creators or original owners of assets (e.g., code, data, models), used in the paper, properly credited and are the license and terms of use explicitly mentioned and properly respected?
    \item[] Answer: \answerYes{}
    \item[] Justification: We cite the original sources for existing assets that we use in the paper, including data (such as Lorenz-96/63, the benchmark dataset containing 13 chaotic systems, and the studied EEG recordings \citep{lorenz1963deterministic,lorenz1996predictability,gilpin2021chaos,schalk2009eeg}) and previously published model classes or algorithms (such as DynaMix, AL-RNNs, and shPLRNNs \citep{hemmer2025true,brenner2024almostlinear,hess2023generalized}), and specify their licenses in the corresponding experimental-details sections.

\item {\bf New assets}
    \item[] Question: Are new assets introduced in the paper well documented and is the documentation provided alongside the assets?
    \item[] Answer: \answerYes{}
    \item[] Justification: We do not introduce a new dataset or benchmark asset that would require separate documentation. The synthetic data-generation procedures and experimental configurations are described in the paper. The model assets alongside training and licenses are included in the provided code.

\item {\bf Crowdsourcing and research with human subjects}
    \item[] Question: For crowdsourcing experiments and research with human subjects, does the paper include the full text of instructions given to participants and screenshots, if applicable, as well as details about compensation (if any)? 
    \item[] Answer: \answerNA{}
    \item[] Justification: The paper does not involve participant recruitment, crowdsourcing, or compensation. Our EEG experiment reuses existing public recordings \citep{schalk2009eeg}, as described in Appx.~\ref{app:eeg}.

\item {\bf Institutional review board (IRB) approvals or equivalent for research with human subjects}
    \item[] Question: Does the paper describe potential risks incurred by study participants, whether such risks were disclosed to the subjects, and whether Institutional Review Board (IRB) approvals (or an equivalent approval/review based on the requirements of your country or institution) were obtained?
    \item[] Answer: \answerNA{}
    \item[] Justification: Experiments were performed using synthetic data except for the EEG recordings in Appx.~\ref{app:eeg}. However, this experiment was a secondary analysis of existing public recordings. We did not recruit participants or collect new human data.

\item {\bf Declaration of LLM usage}
    \item[] Question: Does the paper describe the usage of LLMs if it is an important, original, or non-standard component of the core methods in this research? Note that if the LLM is used only for writing, editing, or formatting purposes and does \emph{not} impact the core methodology, scientific rigor, or originality of the research, declaration is not required.
    \item[] Answer: \answerNA{}
    \item[] Justification: Our core method development, experiments, and scientific contributions do not involve LLMs as an important, original, or non-standard component. LLMs were used to support writing, editing and formatting, and for implementing methods and editing code. All LLM-assisted work was reviewed and verified by the authors; we take responsibility for the final content of this work. 

\end{enumerate}

\end{document}